\documentclass[letterpaper, 10pt]{article}

\usepackage[utf8]{inputenc}
\usepackage{geometry,cite}
\usepackage{smile}
\usepackage{bbm}
\usepackage[pagebackref,letterpaper=true,colorlinks=true,pdfpagemode=none,colorlinks, linkcolor=orange, anchorcolor=blue, citecolor=blue,pdfstartview=FitH,linktocpage=true]{hyperref}
\usepackage{xcolor}

\newcommand{\interior}[1]{%
  {\kern0pt#1}^{\mathrm{o}}%
}

\definecolor{ytw}{RGB}{255,69,0}

\allowdisplaybreaks

\usepackage{amssymb}

\let\emptyset\varnothing

\newcommand{\diff}{\mathrm{d}}

\title{\bf Theoretical Insights for Diffusion Guidance: \\ A Case Study for Gaussian Mixture Models}
\author{ 
	Yuchen Wu\thanks{Department of Statistics and Data Science, the Wharton School, University of Pennsylvania}
	\and
	Minshuo Chen\thanks{Department of Electrical and Computer Engineering, Princeton University}
	\and
	Zihao Li\footnotemark[2]
	\and
	Mengdi Wang\footnotemark[2]
	\and
	Yuting Wei\footnotemark[1]
}
\date{\today}

\begin{document}
\maketitle

%\yw{Summary of changes: 
%\begin{enumerate}
%    \item Although diffusion model is not shift invariant, we can still incorporate locational shift. See Assumption \ref{assumption:confidence}.   
%    \item Added a quantitative lower bound for the DDPM classification confidence, which depends on the realization of the Brownian motion. 
%   \item Removed assumptions on component centers when the number of clusters is two. 
%    \item Added results for the discretized process. 
%\end{enumerate}}

\begin{abstract}

Diffusion models benefit from instillation of task-specific information into the score function to steer the sample generation towards desired properties. 
Such information is coined as guidance. For example, in text-to-image synthesis, text input is encoded as guidance to generate semantically aligned images. 
Proper guidance inputs are closely tied to the performance of diffusion models. 
A common observation is that strong guidance promotes a tight alignment to the task-specific information, while reducing the diversity of the generated samples. 
In this paper, we provide the first theoretical study towards understanding the influence of guidance on diffusion models in the context of Gaussian mixture models. 
Under mild conditions, 
we prove that incorporating diffusion guidance not only boosts classification confidence but also diminishes distribution diversity, leading to a reduction in the differential entropy of the output distribution.
Our analysis covers the widely adopted sampling schemes including DDPM and DDIM, and leverages comparison inequalities for differential equations as well as the Fokker-Planck equation that characterizes the evolution of probability density function, which may be of independent theoretical interest.
\end{abstract}

\tableofcontents

\section{Introduction}

Understanding and designing algorithms for generative models that adapt to certain constraints play a crucial role in modern machine learning applications. 
For example, contemporary large language models --- where a large model is pretrained and various natural language processing (NLP) tasks are performed based on human prompts without re-training --- often demonstrate remarkable in-context learning abilities \cite{}; 
Text-to-image models contribute to major successes in image generators like DALL$\cdot$E 2, Stable Diffusion and Imagen \citep{ramesh2022hierarchical,rombach2022high,saharia2022photorealistic}, 
which offer remarkable platforms for users to generate vivid images by typing in a text prompt.
However, it has been observed that these models can oftentimes generate unrealistic or biased content, or not follow the users' instructions \citep{bommasani2021opportunities,luvcic2019high,weidinger2021ethical}. For this reason, various guided techniques have been developed to enhance the sampling qualities in accordance with users' intention \citep{ouyang2022training,dhariwal2021diffusion,ho2022classifier}. 
Despite the significant empirical improvements that are observed using these guidance approaches, parameters and models are trained mainly in a trial-and-error manner. 
The theoretical underpinnings of these methods are still far from being mature.

\subsection{Training with guidance for diffusion models}

To uncover the unreasonable power of these guided approaches and better assist practice, this paper takes the first step towards this goal in the context of diffusion models. 
Diffusion models, which convert noise into new data instances by learning to reverse a Markov diffusion process, have become a cornerstone in contemporary generative modeling \citep{song2020score,ho2020denoising,yang2023diffusion}.
Compared to alternative generative models, such as variational autoencoder or generative adversarial network, diffusion models are known to be more stable, and generate high-quality samples based on learning the gradient of the log-density function (also known as the score function). When data is multi-modal, namely, it potentially comes from multiple classes, a natural question is how to make use of these class labels for conditional synthesis.  
Towards this direction, \cite{dhariwal2021diffusion} put forward the idea of \emph{classifier guidance} --- an approach to enhance the sample quality with the aid of an extra trained classifier. The  classifier guidance approach combines an unconditional diffusion model’s score estimate with the gradient of the log probability of a classifier. 
Subsequently, \cite{ho2022classifier} presented the so-called \emph{classifier-free guidance}, which instead mixes the score estimates of an unconditional diffusion model with that of a conditional diffusion model jointly trained over the data and the label.
For both guidance methods, 
adjusting the mixing weights of the unconditional score estimate and the other component controls the trade-off between the Fr\'{e}chet Inception Distance (FID) and the Inception Score (IS) in the context of image synthesis. 
The resulting procedures are empirically verified to generate extremely high-fidelity samples that are at least comparable to, if not better than, other types of generative models.

One interesting feature observed for these guided procedures is an improvement in the sample quality and a decrease in the sample diversity as one increases the guidance strength (mixing weight of the other component). 
Specifically, \cite{ho2022classifier} illustrates such phenomenon numerically via a simple two-dimensional distribution comprising a mixture of three isotropic Gaussian distributions. 
In particular, with an increased guidance strength, the generated conditional distribution shifts its probability mass farther away from other classes, and most of the mass becomes concentrated in smaller regions, as can be seen in Figure~\ref{fig:three-components}. 
% \yuting{add a figure here?}\yw{check} 
In this paper, we seek to theoretically explain this observation and provide some rigorous guarantees on how the guidance strength affects the confidence of classification and the in-class sample diversity. 

\begin{figure}
    \centering
    \includegraphics[width=\textwidth]{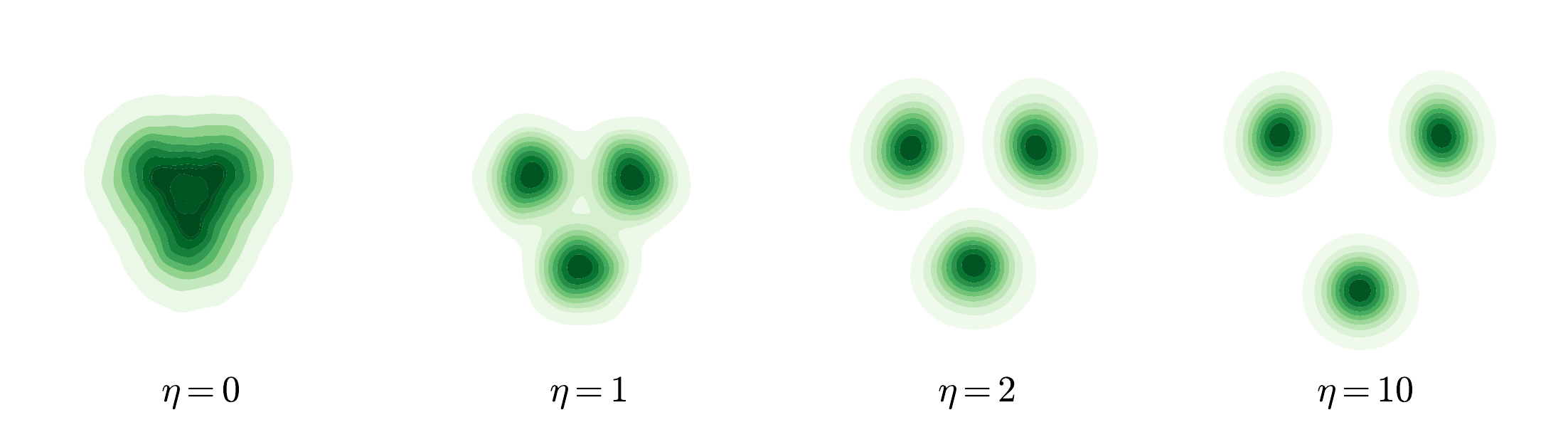}
    \caption{The effect of guidance on a three-component GMM in $\RR^2$. Each component has weight $1 / 3$ and identity covariance, and the component centers are $(\sqrt{3} / 2, 1 / 2)$, $(-\sqrt{3} / 2, 1 / 2)$ and $(0, -1)$. The leftmost panel displays the unguided density. We increase the guidance strength from left to right. 
    This plot imitates Figures 2 of \cite{ho2022classifier}. }
    \label{fig:three-components}
\end{figure}

\subsection{Sampling from Gaussian mixture models} 

To allow for precise theoretical characterizations, we shall focus on the prototypical problem of sampling from Gaussian mixture models (GMMs). Specifically, we consider the data distribution $p_{\ast}$ which takes the following form 
\begin{align}
\label{model:GMM}
    p_{\ast} \overset{d}{=} \sum_{y \in \cY} w_y \normal(\mu_y, \Sigma). 
\end{align}
Here, we use $y$ to denote the class label which takes value in a finite set $\cY := \{1, 2,\ldots, |\cY|\}$.
Given any class label $y \in \cY$, $(\mu_y,\Sigma) \in \RR^d \times \RR^{d\times d}$ gives the center and the covariance matrix for the  Gaussian component that corresponds to $y$. 
In addition, $w_y \in \RR_{\geq 0}$ stands for the component weight for class $y \in \cY$, which satisfies $\sum_{y \in \cY} w_y = 1$. 
% \yuting{$w_{y}$ is used for labels, while $w$ is used for strength of the guidance. shall we replace one of them to be $\omega$?} \yw{I set up a macro $\backslash$gs, now let's use $\eta$}

In this work, we investigate two widely adopted sampling methods for diffusion models, including a stochastic differential equation (SDE) based approach called the \emph{denoising diffusion probabilistic models} (DDPMs) \citep{ho2020denoising} and an ordinary differential equation (ODE) based approach called \emph{denoising diffusion implicit models} (DDIMs) \citep{song2020denoising}. An overview of these two methods under both classifier guidance and classifier-free guidance is provided in Section~\ref{sec:preliminary}.
As shall be clear momentarily, both methods involve a tuning parameter $\gs > 0$ which controls the strength of the classifier guidance (\emph{resp.} full-model guidance) in the classifier guidance (\emph{resp.} classifier-free guidance) approach. The overarching goal is to understand how the guidance strength affects the sample qualities, in particular, the confidence of classification and the in-class diversity.

% \begin{enumerate}
%     \item Diffusion model is a powerful generative model. Describe diffusion model. 
%     \item Some drawbacks of diffusion model comparing to GAN. Introduce guidance, classifer and classifier-free. 
%     \item The effect of guidance, increase accuracy while decrease diversity. 
%     \item However, diffusion guidance lacks theoretical underpinning. We establish a first results under GMM. 
% \end{enumerate}

\subsection{A glimpse of main contributions}

In what follows, we highlight several of our key findings.
\begin{itemize}
    \item Consider a Gaussian mixture models with general positions. For both DDPM and DDIM samplers with diffusion guidance, we demonstrate in Section~\ref{sec:pre-confidence}, that the classification confidence --- which measures the posterior probability associated with the guided class given an output sample --- only increases when diffusion guidance is applied. These quantitative results (Theorems~\ref{thm:DDIM-confidence} and \ref{thm:DDPM-confidence}) are further accompanied by qualitative results (Theorems~\ref{thm:DDIM-confidence-quan} and \ref{thm:DDPM-confidence-quan}), titrating the exact level of influence of diffusion guidance for posterior classification accuracy. 
    These findings offer theoretical validation for employing diffusion guidance to enhance conditional sampling.

    % First, we show that guidance strength can only increase the classification confidence, if we additionally assume $\Sigma = I_d$ and the component centers display ``low correlation''. We also provide a quantitative lower bound on the effect of guidance on classification confidence.  
    \item As for the in-class diversity, in Section~\ref{sec:diversity}, we analyze the impact of guidance strength on the differential entropy of the resulting distribution for DDIM samplers. 
    It turns out that increasing the diffusion guidance always results in a reduction in the differential entropy. 
    This offers the first theoretical explanation for the benefit of the diffusion guidance in generating more homogeneous samples. 

    \item Finally, we exhibit that the role of the guidance strength can be complicated by an example of a  three-component GMM when their means are aligned. In this case, we reveal both theoretically and numerically the existence of a phase transition in the behavior of the classification confidence as one increases the guidance strength. Cautions thus need to be exercised in practice in terms of selecting a proper guidance strength. 
    More details can be found in Section~\ref{sec:example-curious}.
\end{itemize}

% We illustrate our discoveries with an example of a symmetric GMM in Figure \ref{fig:prob-DP}. 

\begin{figure}%
    \centering
    \includegraphics[width=0.49\textwidth]{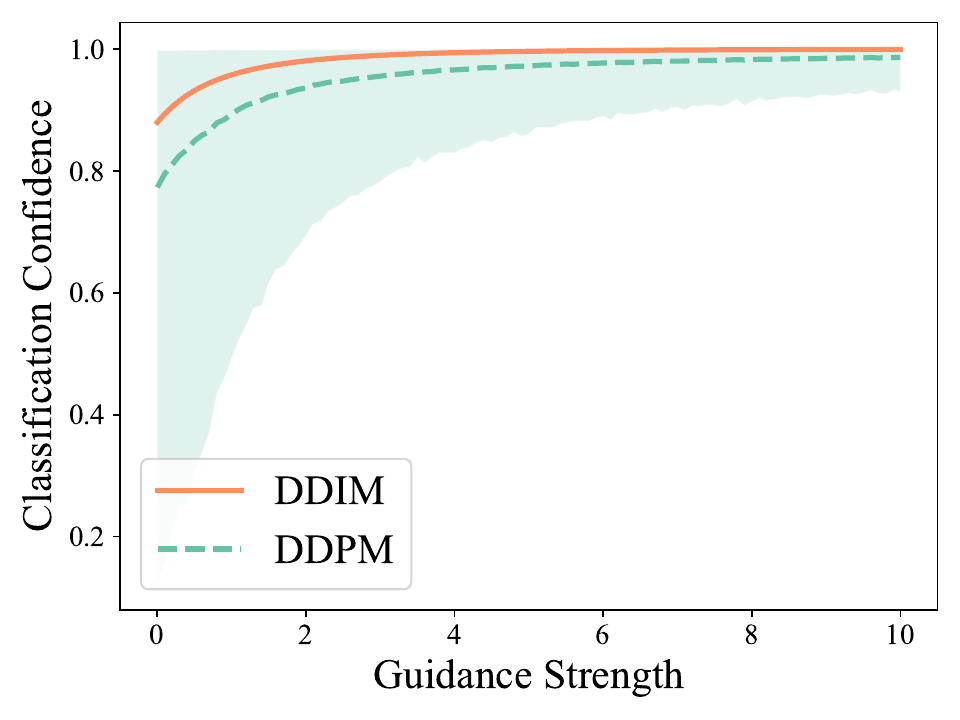} %
    \hfill
    \includegraphics[width=0.49\textwidth]{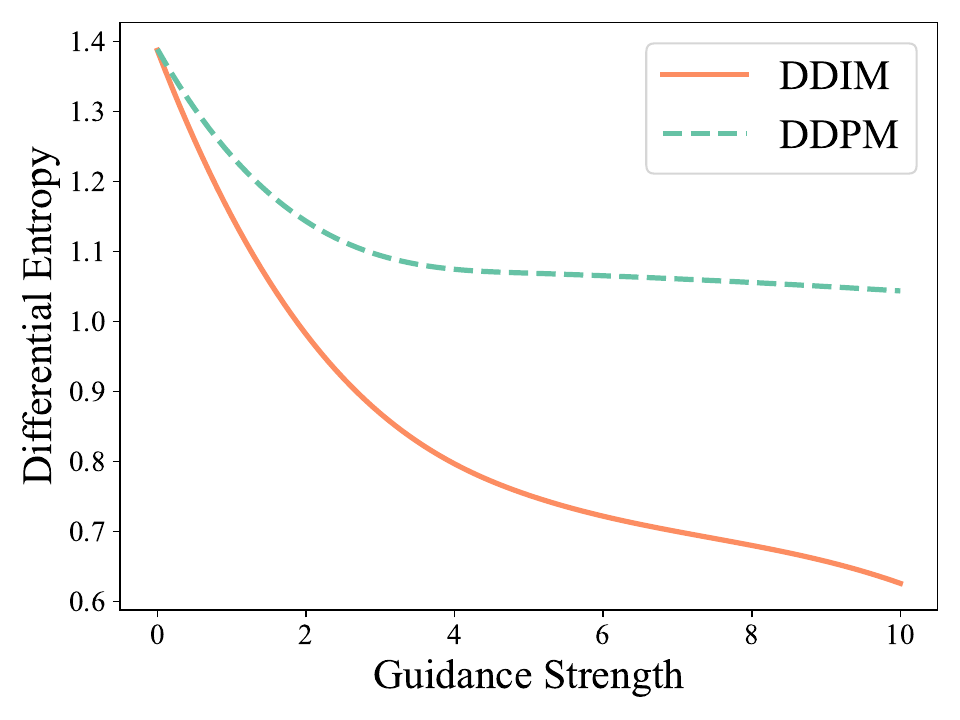} %
    \caption{The effect of guidance on a symmetric GMM: $p_{\ast} = \frac{1}{2} \normal(1, 1) + \frac{1}{2} \normal(-1, 1)$. (a) In the left panel, we initiate the reverse processes at the origin, and record the classification confidence (measured by the posterior probability of class label) under different levels of guidance. For the DDPM sampler the output sample is random. We generate $10^4$ samples for each guidance strength and plot the averaged classification confidence for both the DDPM and the DDIM samplers, as well as the $97.5\%$ and $2.5\%$ quantiles for the DDPM sampler. (b) In the right panel, we initiate the processes following a standard Gaussian distribution, and plot the differential entropy of the output distributions. For each guidance strength we also generate $10^4$ samples. We adopt the function $\mathtt{scipy.stats.differential\_entropy()}$ from the $\mathtt{scipy}$ module in Python to estimate the differential entropy based on these generated samples. }%
    \label{fig:prob-DP}%
\end{figure}

\paragraph{Notation.}
For two random objects $X$ and $Y$, we say $X \perp Y$ if and only if they are independent of each other. 
For $n \in \NN$, we define the set $[n] = \{1, 2, \cdots, n\}$, and make the convention that $[0] = \emptyset$.
We use $\sigma_{\min}(M)$ to denote the minimum eigenvalue of a matrix $M$.

% \subsection{Roadmap}

\section{Preliminaries}
\label{sec:preliminary}

In this section, we introduce the basics of diffusion models, both with and without guidance. 
Our investigation encompasses both the DDPM and the DDIM samplers. 
As aforementioned, there exist two primary forms of guidance, namely, the classifier guidance and the classifier-free guidance. 
We shall delve into a separate discussion of these two guidance forms below. 
As we will observe, these two forms of guidance coincide when precise access to the ground truth probability distributions is available.
To enhance readers' understanding, we initiate our investigation with continuous-time processes.
We later offer generalizations to discrete processes in Section \ref{sec:discretization}.
\subsection{Diffusion model without guidance}

% \yuting{emphasize that our goal is to do conditional sampling}

We begin by revisiting the concept of diffusion model without guidance. 
There has been a  surge of recent interest and theoretical advancements to understand sampling qualities of diffusion models (e.g.~\citet{block2020generative,de2021diffusion,liu2022let,de2022convergence,lee2023convergence,pidstrigach2022score,chen2022sampling,benton2023linear,chen2022improved,chen2023restoration,chen2023score,mei2023deep,tang2024contractive,li2023towards,li2024towards,li2024accelerating}). 
In this paper, we focus our attention on the task of conditional sampling. 
More specifically, let $p_{\ast}$ denote the data distribution over $(x, y)$, where $x$ is the data feature and $y$ stands for the data label.  
Our goal is to sample from the conditional distribution $p_{\ast}(\cdot \mid y)$, conditioning on a label realization $y$.
Throughout the paper, we use $y$ to represent the label we wish to condition on. 
The diffusion model consists of two processes: a forward process that converts the target distribution into noise, and a reverse process that sequentially denoises the process to reconstruct the target distribution.
Throughout this paper, we set the forward process to be an Ornstein–Uhlenbeck (OU) process: 
\begin{align}
\label{eq:OU-process} 
    \diff z_t^{\rightarrow} = - z_t^{\rightarrow} \diff t + \sqrt{2} \diff B_t, \qquad z_0^{\rightarrow} \sim p_{\ast}(\cdot \mid y), 
\end{align}
where $(B_t)_{0 \leq t \leq T}$ is a $d$-dimensional standard Brownian motion.
For $0 \leq t \leq T$, we denote by $p_t$ the distribution of $z_t^{\rightarrow}$. 
The reverse process of \eqref{eq:OU-process} can be constructed using either an ODE or SDE implementation, which we state below: 
\begin{align}
\label{eq:unguided}
\begin{split}
	& \diff z_t^{\leftarrow} = (z_t^{\leftarrow} + \nabla \log p_{T - t}(z_t^{\leftarrow} \mid y)) \diff t, \\
	& \diff \bar{z}_t^{\leftarrow} = (\bar{z}_t^{\leftarrow} + 2 \nabla \log p_{T - t}(\bar{z}_t^{\leftarrow} \mid  y)) \diff t + \sqrt{2} \diff B_t.
\end{split}
\end{align}
In the above display,  %$s_t(x, y) = \nabla_x \log p_t(x \mid y)$ is the conditional score function, 
%$p_{\mathrm{init}}$ is the initial distribution with
$z_0^{\leftarrow}, \bar z_0^{\leftarrow} \sim p_{\mathrm{init}}$ for some initial distribution $p_{\mathrm{init}}$, and $(B_t)_{0 \leq t \leq T}$ once again is the standard Brownian motion in $\RR^d$.
Hereafter, unless stated otherwise, we always take the gradient with respect to the first argument. 
Classical findings in probability theory \citep{anderson1982reverse} implies that when $p_{\mathrm{init}} = p_T(\cdot \mid y)$, it holds that
\begin{align*}
z_t^{\leftarrow} \overset{d}{=} \bar z_t^{\leftarrow} \overset{d}{=} z_{T - t}^{\rightarrow}.
\end{align*}
As a consequence, if we can implement process \eqref{eq:unguided}, then in principle we shall be able to generate new samples from our target distribution. 
To design an implementable algorithm, practioners not only apply discretization to processes in  Eq.~\eqref{eq:unguided}, but also substitute the score functions and the theoretically ideal initial distribution $p_T(\cdot \mid y)$ with their respective estimates. 
A standard approach for approximating $p_T(\cdot \mid y)$ is by setting $p_{\mathrm{init}} = \normal(0, I_d)$.  
%We represent the target label towards which the diffusion guidance leads by $y$.

For the sake of simplicity, in the sequel we write $(z_t)_{0 \leq t \leq T} = (z_t^{\leftarrow})_{0 \leq t \leq T}$ and $(\bar z_t)_{0 \leq t \leq T} = (\bar 
 z_t^{\leftarrow})_{0 \leq t \leq T}$ without introducing any confusion.

%For comparison, we also consider the DDIM and the DDPM samplers without guidance: 
%
%\begin{align*}
%	& \diff z_t = (z_t + s_{T - t}(z_t, y)) \diff t,  \\
%	& \diff \bar{z}_t = (\bar{z}_t + 2 s_{T - t}(\bar{z}_t, y)) \diff t + \sqrt{2} \diff B_t. 
%\end{align*}

\subsection{Classifier diffusion guidance}

%We then introduce two types of guidances for the processes stated in Eq.~\eqref{eq:unguided}, and start with the classifier guidance. 
Classifier guidance was first proposed by \cite{dhariwal2021diffusion} to improve the quality of images produced by diffusion models, with the aid of an extra trained classifier. 
To achieve this, they modify the score function to include the gradient of the logarithmic prediction probability of an auxiliary classifier. 
To be definite, the DDIM and the DDPM samplers under classifier guidance are as follows: 
\begin{align}
	& \diff x_t^c = \left( x_t^c + s_{T - t}(x_t^c, y) + \gs \nabla \log c_{T - t}(x_t^c, y) \right) \diff t, &  x_0^c \sim p_{\mathrm{init}},  %\tag{DDIM-$w$} 
 \label{eq:DDIM-classifier} \\
	& \diff \bar x_t^c = \left( \bar x_t^c + 2s_{T - t}(\bar x_t^c, y) + 2\gs \nabla \log c_{T - t}(\bar x_t^c, y) \right) \diff t + \sqrt{2} \diff B_t, & \bar{x}_0^c \sim p_{\mathrm{init}}. %\tag{DDPM-$w$}
 \label{eq:DDPM-classifier}
\end{align}
In the above display, $\gs \geq 0$ is a parameter that controls the strength of the classifier guidance,
$s_{T - t}(x, y)$ is an estimate to $\nabla \log p_{T - t}(x \mid y)$,
and $c_{T - t}(x, y)$ is a probabilistic classifier that is designed to estimate the conditional probability $p_{T - t}(y \mid x)$. 
When $\gs = 0$ and $s_{T - t}(x, y) = \nabla \log p_{T - t}(x \mid y)$, processes \eqref{eq:DDIM-classifier} and \eqref{eq:DDPM-classifier} reduce to their unguided counterparts. 
%Processes \eqref{eq:DDIM-classifier} and \eqref{eq:DDPM-classifier} share the same initial distribution $p_{\mathrm{init}}$, which is usually taken to be $p_T$ or $\normal(0, I_d)$. 

\subsection{Classifier-free diffusion guidance}

Classifier guidance effectively boosts the sample quality of diffusion models. 
However, it requires an extra classifier, potentially introducing complexity to the model training pipeline.
Classifier-free guidance is an alternative method of modifying the score functions to have the same effect as classifier guidance, but without a classifier \citep{ho2022classifier}. 
To be concrete, classifier-free guidance involves the following processes: 
\begin{align}
	& \diff x_t^f = \left( x_t^f + (1 + \gs)s_{T - t}(x_t^f, y) - \gs s_{T - t}(x_t^f) \right) \diff t, & x_0^f \sim p_{\mathrm{init}} %\tag{DDIM-$w$} 
 \label{eq:DDIM-free} \\
	& \diff \bar x_t^f = \left( \bar x_t^f + 2(1 + \gs)s_{T - t}(\bar x_t^f, y) - 2\gs  s_{T - t}(\bar x_t^f) \right) \diff t + \sqrt{2} \diff B_t, &\, \,\,\,\bar{x}_0^f \sim p_{\mathrm{init}}.  %\tag{DDPM-$w$}
 \label{eq:DDPM-free}
\end{align}
In the displayed content above, with a slight abuse of notation, we use $s_{t} (x)$ without the second argument to represent an estimate to the unconditional score function $\nabla_x \log p_t(x)$.
Note that in situations where we have exact access to the ground truth functionals (i.e., $s_{t}(x, y) = \nabla_x \log p_t(x \mid y)$, $s_t(x) = \nabla_x \log p_t(x)$, and $c_t(x, y) = p_t(y \mid x)$), one can verify that 
$$(x_t^f)_{0 \leq t \leq T} \overset{d}{=} (x_t^c)_{0 \leq t \leq T}, \qquad (\bar{x}_t^f)_{0 \leq t \leq T} \overset{d}{=} (\bar{x}_t^c)_{0 \leq t \leq T}.$$
This result is independent of the choice of the initial distribution $p_{\mathrm{init}}$. 
Similarly, by setting $\gs = 0$ and using the ground truth functionals, processes \eqref{eq:DDIM-free} and \eqref{eq:DDPM-free} reduce to the unguided ones.

%
%This is accomplished by setting $w = 0$ in processes \eqref{eq:DDIM-classifier} and \eqref{eq:DDPM-classifier} (or equivalently, processes \eqref{eq:DDIM-free} and \eqref{eq:DDPM-free}). 

It was observed that guidance for diffusion model, either classifier-based or classifier-free, has the effect of increasing classification confidence and decreasing sample diversity \citep{ho2022classifier}. 
This paper seeks to offer a theoretical explanation of this phenomenon within the framework of GMM.

\subsection{Guided diffusion for Gaussian mixture models}

Under the GMM as stated in Eq~\eqref{model:GMM}, both the score functions and the logarithmic class probabilities admit closed-form expressions, and we shall adopt these ground truth functionals to construct our samplers. Namely, throughout this paper, we set 
\begin{align}
\label{eq:true-cs}
    & s_t(x, y) = \nabla_x \log p_t(x \mid y) =   - \Sigma_t^{-1} x + e^{-t} \Sigma_t^{-1} \mu_y, \\
    & \nabla_x \log c_t(x, y) = \nabla_x \log p_t(y \mid x) =  e^{-t} \Sigma_t^{-1} \mu_y - \sum_{y' \in \cY} e^{-t} q_t(x, y') \Sigma_t^{-1} \mu_{y'}, \nonumber
\end{align} 
where $\Sigma_t := e^{-2t} \Sigma + (1 - e^{-2t}) I_d$, and 
\begin{align}
\label{eq:def-qt}
     q_t(x, y) :=  \frac{w_y \exp \left( e^{-t} \langle \Sigma_t^{-1} \mu_{y},   x \rangle - e^{-2t} \langle \mu_{y},  \Sigma_t^{-1} \mu_{y} \rangle / 2 \right) }{\sum_{y' \in \cY} w_{y'} \exp \left( e^{-t} \langle \Sigma_t^{-1} \mu_{y'},   x \rangle - e^{-2t} \langle \mu_{y'},  \Sigma_t^{-1} \mu_{y'} \rangle / 2\right)  }.
\end{align}
%
% To be precise,  we denote by $p_{\ast}$ the data distribution and assume it takes the following form: 
% %
% \begin{align}
% \label{model:GMM}
% 	p_{\ast} \overset{d}{=} \sum_{y \in \cY} w_y \normal(\mu_y, \Sigma), 
% \end{align}
% %
% %In the above display the vector $\mu_0$ is indicative of a locational shift. We comment that non-identifiability issue exists in the current formulation, while we keep the additional parameter $\mu_0$ in orde
% where $w_y \in \RR_{\geq 0}$ is the component weight that corresponds to label $y$, satisfying $\sum_{y \in \cY} w_y = 1$, $\{\mu_y: y \in \cY\} \subseteq \RR^d$ are the component centers, and $\Sigma \in \RR^{d \times d}$ is the shared component covariance. 
Note that $q_t(x, y)$ is the posterior probability of having label $y$, upon observing $x = e^{-t} x_{\ast} + \sqrt{1 - e^{-2t}} g$, where $x_{\ast} \sim p_{\ast}$, $g \sim \normal(0, I_d)$ and $x_{\ast} \perp g$. 
When the functionals listed in Eq.~\eqref{eq:true-cs} are adopted to construct the diffusion model samplers as listed in Eq.~\eqref{eq:DDIM-classifier}-\eqref{eq:DDPM-free}, obviously we have 
\begin{align*}
    (x_t^f)_{0 \leq t \leq T} \overset{d}{=} (x_t^c)_{0 \leq t \leq T}, \qquad (\bar{x}_t^f)_{0 \leq t \leq T} \overset{d}{=} (\bar{x}_t^c)_{0 \leq t \leq T}. 
\end{align*}
In fact, in this case the classifier-based and the classifier-free diffusion models share the same diffusion and drift terms. 
Due to this observation, in the remainder of the paper we unify the notations by setting 
\begin{align*}
    (\bar{x}_t)_{0 \leq t \leq T} = (\bar{x}_t^c)_{0 \leq t \leq T} = (\bar{x}_t^f)_{0 \leq t \leq T}, \qquad (x_t)_{0 \leq t \leq T} = (x_t^c)_{0 \leq t \leq T} = (x_t^f)_{0 \leq t \leq T},
\end{align*}
%$(\bar{x}_t)_{0 \leq t \leq T} = (\bar{x}_t^c)_{0 \leq t \leq T} = (\bar{x}_t^f)_{0 \leq t \leq T}$ and $(x_t)_{0 \leq t \leq T} = (x_t^c)_{0 \leq t \leq T} = (x_t^f)_{0 \leq t \leq T}$, 
treating classifier-based and classifier-free guidance as the same algorithm. 
% \yuting{try to use fewer inline equations.}\yw{check}

Plugging Eq.~\eqref{eq:true-cs} into Eq.~\eqref{eq:DDIM-classifier}-\eqref{eq:DDPM-free}, we obtain 
\begin{align}
    & \diff x_t = \Big( x_t - \Sigma_{T - t}^{-1} x_t + e^{-T + t} \Sigma_{T - t}^{-1} \mu_y + \gs e^{-T + t} \Sigma_{T - t}^{-1} \mu_y -\gs e^{-T + t}  \sum_{y' \in \cY}  q_{T - t}(x_t, y') \Sigma_{T - t}^{-1} \mu_{y'} \Big) \diff t, \label{eq:dxt-long} \\
    & \diff \bar x_t = \Big( \bar x_t - 2\Sigma_{T - t}^{-1} \bar x_t + 2e^{-T + t} \Sigma_{T - t}^{-1} \mu_y + 2\gs e^{-T + t} \Sigma_{T - t}^{-1} \mu_y -2\gs e^{-T + t}  \sum_{y' \in \cY}  q_{T - t}(\bar x_t, y') \Sigma_{T - t}^{-1} \mu_{y'} \Big) \diff t  + \sqrt{2} \diff B_t, \label{eq:dxt-bar-long}
\end{align}
for any guidance level $\gs.$

\section{Effect of guidance on classification confidence}
\label{sec:pre-confidence}

As our first contribution, we offer a theoretical explanation for the phenomenon where a diffusion model with guidance directs generated samples toward a region with higher confidence, in contrast to samples generated by the unguided counterpart.  
To measure such confidence, we propose to examine the posterior probability
\begin{align}
    \cP(x, y) := q_0(x, y) = \frac{w_y \exp \left( \langle \Sigma^{-1} \mu_{y},   x \rangle - \langle \mu_{y},  \Sigma^{-1} \mu_{y} \rangle / 2 \right) }{\sum_{y' \in \cY} w_{y'} \exp \left(  \langle \Sigma^{-1} \mu_{y'},   x \rangle -  \langle \mu_{y'},  \Sigma^{-1} \mu_{y'} \rangle / 2\right)  }
\end{align}
along the trajectory of the diffusion process as defined in Eq.~\eqref{eq:def-qt}. 
We show that diffusion guidance with a non-negative guidance strength can only increase the posterior probability, given that the component centers exhibit limited correlation.  
Our formal assumptions are provided below.
\begin{assumption}
\label{assumption:confidence}
    % Recall that the distribution is guided towards the conditional distribution $p_{\ast}(\cdot \mid y)$. 
    We impose the following conditions on model \eqref{model:GMM}: 
    \begin{enumerate}
        \item There exists $\mu_0 \in \RR^d$, such that for all $y' \in \cY$, it holds that $|\langle \mu_y - \mu_0, \mu_{y'} - \mu_0  \rangle| \leq \varepsilon$, for some small positive constant $\varepsilon$. We further assume that $\varepsilon \leq \|\mu_y - \mu_0\|_2^2 / 3$. 
        \item The prior probability $w_{y'}$ is strictly positive for all $y' \in \cY$. 
        \item The GMM has an isotropic common covariance: $\Sigma = I_d$.  
    \end{enumerate}
\end{assumption}

\begin{remark}
    If $d$ is large, then the first point of Assumption \ref{assumption:confidence} is typically satisfied when the component centers are independently generated from certain prior distribution. 
    For instance, one can verify that the assumption is satisfied with high probability if $(\mu_{y'})_{y' \in \cY} \sim_{i.i.d.} \mbox{Unif}(\mathbb{S}^{d - 1})$ with a sufficiently large $d$, where $\mathbb{S}^{d - 1}$ is the unit sphere in $\RR^d$. 
\end{remark}

%For this section, we additionally assume $\langle \mu_y, \mu_{y'} \rangle = 0$ for all $y, y' \in \cY$ with $y \neq y'$. 
%We define $\cP(x, y) = q_0(x, y)$ that represents the likelihood function corresponding to label $y$. 
%Under Assumption \ref{assumption:confidence}, 
%observe that $\nabla_x \log \cP(x, y) =  \mu_y - \sum_{y' \in \cY} \cP(x, y') \mu_{y'}$.

The dynamics of $\cP(x_t, y)$ can be represented through either an ODE or an SDE, depending on whether we utilize the DDIM or the DDPM framework.
We explore further details in the remainder of this section.

\subsection{Effect on the DDIM sampler}
\label{sec:effect-DDIM}

In this section, we analyze the impact of guidance on the DDIM sampler, as defined in Eq.~\eqref{eq:dxt-long}. 
Our main result for this part delineates the impact of guidance on the DDIM sampler in terms of classification confidence, which we present as Theorem \ref{thm:DDIM-confidence} below.
%To maintain conciseness, the proof for Theorem \ref{thm:DDIM-confidence} is deferred to Appendix \ref{sec:proof-thm:DDIM-confidence}.
%
% \yuting{remind readers about $z_0$ and $x_0$ again.}\yw{check} 
%
\begin{theorem}
\label{thm:DDIM-confidence}
    We assume model \eqref{model:GMM} and Assumption \ref{assumption:confidence}. Recall that $x_0$ and $z_0$ are the initializations of the DDIM samplers as defined in Eq.~\eqref{eq:unguided} and \eqref{eq:dxt-long}, respectively. In addition, we assume $\langle x_0, \mu_y - \mu_{y'} \rangle \geq \langle z_0, \mu_y - \mu_{y'} \rangle$ for all $y' \in \cY$\footnote{Note that this conditions is fulfilled when $x_0 = z_0$. }. Then for any $\gs \geq 0$ and all $t \in [0, T]$, it holds that $$\cP(x_t, y) \geq \cP(z_t, y).$$
\end{theorem}
Theorem \ref{thm:DDIM-confidence} implies that when the processes have the same initialization, 
the classification confidence associated with the guided process remains no smaller than that associated with the unguided process along the entire diffusion trajectory. It therefore validates the empirical observation regarding diffusion guidance.

In order to offer some theoretical insights while at the same time maintaining brevity, we present a proof sketch of Theorem \ref{thm:DDIM-confidence} here, and delay  the majority of technical details to Appendix \ref{sec:proof-thm:DDIM-confidence}.
First, taking the inner product of the derivative given in Eq.~\eqref{eq:dxt-long} and the mean vector difference $\mu_y - \mu_{y'}$ for some $y' \in \cY$, we obtain
\begin{align}
\label{eq:xt-muy-muyp}
\begin{split}
    & \frac{\diff}{\diff t}  \langle x_t, \mu_y - \mu_{y'} \rangle \\
    = & e^{-(T  - t)}\|\mu_y\|_2^2 - e^{-(T - t)} \langle \mu_y, \mu_{y'} \rangle + \gs e^{-(T - t)} (1 - q_{T - t}(x_t, y)) \|\mu_y - \mu_0\|_2^2 \\
    & + \gs e^{-(T - t)} q_{T - t}(x_t, y') \|\mu_{y'} - \mu_0\|_2^2 + \cE_t,
\end{split}
\end{align}
where $\cE_t$ is a function of $(x_t, t)$, satisfying $|\cE_t| \leq 3\gs e^{-(T - t)} (1 - q_{T - t}(x_t, y)) \varepsilon$.  
A detailed derivation of Eq.~\eqref{eq:xt-muy-muyp} is given in Appendix \ref{sec:increase-confidence}. 
Using the assumption that $\varepsilon \leq \|\mu_y\|_2^2 / 3$, one can obtain 
%$\langle x_t, \mu_y \rangle$ with respect to $t$ then we get
%
%\begin{align*}
%    \frac{\diff}{\diff t} \langle x_t, \mu_y \rangle = & (1 + w - wq_{T - t}(x_t, y) )e^{-(T - t)} \|\mu_y\|_2^2 - w e^{-(T - t)} \sum_{y' \neq y, y' \in \cY}  q_{T - t}(x_t, y') \langle \mu_y, \mu_{y'} \rangle, \\
%    \frac{\diff}{\diff t} \langle x_t, \mu_{y'} \rangle = & (1 + w - w q_{T - t}(x_t, y)) e^{-(T - t)} \langle \mu_y, \mu_{y'}\rangle - we^{-(T - t)} q_{T - t}(x_t, y') \|\mu_{y'}\|_2^2 \\
%    & - we^{-(T - t)} \sum_{y'' \neq y, y'} q_{T - t}(x_t, y'') \langle \mu_{y''}, \mu_{y'} \rangle, 
%\end{align*}
%
%where $y' \neq y$. Using the first point in Assumption \ref{assumption:confidence}, we conclude that 
%
%\begin{align*}
%    & \frac{\diff}{\diff t} \langle x_t, \mu_y \rangle = (1 + w - w q_{T - t}(x_t, y)) e^{-(T - t)} \|\mu_y\|_2^2 + \cE_t, \\
%    & \frac{\diff}{\diff t} \langle x_t, \mu_{y'} \rangle = e^{-(T - t)} \langle \mu_y, \mu_{y'} \rangle - w e^{-(T - t)} q_{T - t}(x_t, y') \|\mu_{y'}\|_2^2 + \cE_t',
%\end{align*}
%
% where $(\cE_t, \cE_t')$ are functions of $(x_t, t)$, satisfying $|\cE_t| \leq we^{-(T - t)} (1 - q_{T - t}(x_t, y)) \varepsilon$ and $|\cE'_t| \leq 2we^{-(T - t)}(1 -  q_{T - t}(x_t, y)) \varepsilon$. 
% Note that if $\varepsilon \leq \|\mu_y\|_2^2 / 3$, then 
%
\begin{align}
\label{eq:diff-x-mu-mu}
\begin{split}
    & \frac{\diff}{\diff t} \langle x_t, \mu_y - \mu_{y'} \rangle \\
    \geq & e^{-(T - t)} \|\mu_y\|_2^2 - e^{-(T - t)} \langle \mu_y, \mu_{y'} \rangle + \gs e^{-(T - t)} (1 - q_{T - t}(x_t, y)) (\|\mu_y - \mu_0\|_2^2 - 3 \varepsilon). 
    % = & (1 + w - w q_{T - t}(x_t, y)) e^{-(T - t)}\|\mu_y\|_2^2 - e^{-(T - t)} \langle \mu_y, \mu_{y'} \rangle + w e^{-(T - t)} q_{T - t}(x_t, y') \|\mu_{y'}\|_2^2 + \cE_t - \cE_t' \\
    % \geq & e^{-(T - t)} \|\mu_y\|_2^2 - e^{-(T - t)} \langle \mu_{y}, \mu_{y'} \rangle + we^{-(T - t)}(1 - q_{T - t}(x_t, y))  (\|\mu_y\|_2^2 - 3 \varepsilon) \\
    % & + w e^{-(T - t)} q_{T - t}(x_t, y') \|\mu_{y'}\|_2^2 \\
    % \geq & e^{-(T - t)} \|\mu_y\|_2^2 - e^{-(T - t)} \langle \mu_{y}, \mu_{y'} \rangle.
\end{split}
\end{align}
As for the unguided process $(z_t)_{0 \leq t \leq T}$, similarly, we derive 
\begin{align}
\label{eq:diff-z-mu-mu}
    \frac{\diff}{\diff t} \langle z_t, \mu_{y} - \mu_{y'} \rangle = e^{-(T - t)} \|\mu_y\|_2^2 - e^{-(T - t)} \langle \mu_y, \mu_{y'} \rangle. 
\end{align}
Putting Eqs.~\eqref{eq:diff-x-mu-mu} and \eqref{eq:diff-z-mu-mu} together motivates us to employ the ODE comparison theorem \citep{mcnabb1986comparison} to study these two dynamics. 
For readers' convenience, we include the comparison theorem below. 
%A proof of the theorem can be found in many classical ODE textbooks. 
% \yuting{cite:}\yw{check}
%
\begin{lemma}[ODE comparison theorem]
	\label{lemma:comparison-theorem}
	Suppose $f(t, u)$ is continuous in $(t, u)$ and Lipschitz continuous in $u$. Suppose $u(t)$, $v(t)$ are $C^1$ for $t \in [0, T]$, and satisfy 
	\begin{align*}
		u'(t) \leq f(t, u(t)), \qquad v'(t) = f(t, v(t)). 
	\end{align*}
	In addition, we assume $u(0) \leq v(0)$. Then $u(t) \leq v(t)$ for all $t \in [0, T]$. 
\end{lemma} 
As a direct consequence of Lemma \ref{lemma:comparison-theorem}, Eq.~\eqref{eq:diff-x-mu-mu} and \eqref{eq:diff-z-mu-mu}, we derive the following lemma: 
\begin{lemma}
\label{lemma:compare-inner-product}
   Under model \eqref{model:GMM} and Assumption \ref{assumption:confidence}, suppose in addition that $\langle z_0, \mu_y - \mu_{y'}  \rangle \leq \langle x_0, \mu_y - \mu_{y'} \rangle$. Then, it holds that
   \begin{align}
    	\langle x_t, \mu_y - \mu_{y'} \rangle \geq \langle z_t, \mu_y - \mu_{y'} \rangle,
    \end{align} 
   for all $t \in [0, T]$.   
\end{lemma}
Our proof of Theorem \ref{thm:DDIM-confidence} makes key use of Lemma \ref{lemma:compare-inner-product}, and 
offers a qualitative comparison between diffusion model with guidance and the original diffusion model. 
We refer the readers to Appendix \ref{sec:proof-thm:DDIM-confidence} for a complete proof of Theorem \ref{thm:DDIM-confidence}. 
We also prove a result below which quantitatively measures the role of guidance. 
We provide the proof of Theorem~\ref{thm:DDIM-confidence-quan} in Appendix \ref{sec:proof-thm:DDIM-confidence-quan}.

\begin{theorem}
\label{thm:DDIM-confidence-quan}
Under the assumptions of Theorem \ref{thm:DDIM-confidence}, for any $\gs \geq 0$, it holds that 
\begin{align*}
    \cP(x_T, y) \geq \frac{\cP(z_T, y)}{\cP(z_T, y) + (1 - \cP(z_T, y) )\cdot \exp(-\cU)} \geq \cP(z_T, y). %\cP(x_t, y) \geq \frac{\cP(z_t, y)}{\cP(z_t, y) + (1 - \cP(z_t, y)) \cdot \exp \big( -(e^{-(T - t)} - e^{-T})M - M_0 \big)} \geq \cP(z_t, y).
\end{align*}
In the above display, $\cU \in \RR_{\geq 0}$ is any real number that satisfies
\begin{align*}
	\cU < \langle x_0 - z_0, \mu_y - \mu_{y'} \rangle + (1 - e^{-T}) \cdot \gs e^{-\Delta / 8} (\|\mu_y - \mu_0\|_2^2 - 3 \varepsilon) \cdot \min \Big\{\cF\big( \max_{0 \leq t \leq T} \cP(z_t, y), \cU \big),\, \xi_w \Big\}. 
\end{align*} 
%
%\begin{align*}
%	& M = \gs e^{-\Delta / 8} (\|\mu_y - \mu_0\|_2^2 - 3 \varepsilon) \cdot \min \Big\{\cF\big( \max_{0 \leq t \leq T} \cP(z_t, y), \cU \big),\, \xi_w \Big\}, \\
%	& M_0 = \langle x_0 - z_0, \mu_y - \mu_{y'} \rangle,   
%\end{align*}
%
%and $\cU$ is any positive real number that satisfies 
% 
%\begin{align*}
%	& \cU \geq \langle x_0 - z_0, \mu_y - \mu_{y'} \rangle + (1 - e^{-T})  M. 
%\end{align*}
%
Here, 
\begin{align}
\label{eq:cF-xi-Delta}
	\cF(p, u ) = \frac{(1 - p)e^{-u}}{p + (1 - p) e^{-u}}, \qquad \xi_w = 1 - w_y / (w_y + \min_{y' \neq y} w_{y'} ), \qquad \Delta = \max_{y' \in \cY} |\|\mu_y\|_2^2 - \|\mu_{y'}\|_2^2|.
\end{align}
%
%$$ \cF(p, u ) = \frac{(1 - p)e^{-u}}{p + (1 - p) e^{-u}}, \qquad \xi_w = 1 - w_y / (w_y + \min_{y' \neq y} w_{y'} ), \qquad \Delta = \max_{y' \in \cY} |\|\mu_y\|_2^2 - \|\mu_{y'}\|_2^2|.$$
%
%\begin{align*}
%    & M := \gs (1 - \max \{ G(g(\cP(x_0, y))), G(w_y / (w_y + \min_{y' \neq y} w_{y'}))\})  (\|\mu_y - \mu_0\|_2^2 - 3 \varepsilon), \\
%    & g(p) := \frac{p}{p + \exp (-(1 - e^{-T}) U)(1 - p)}, \\
%    & U :=  \max_{y' \in \cY \backslash \{y\}}\|\mu_y\|_2 \cdot \|\mu_y - \mu_{y'}\| + \gs \big(\max_{y' \in \cY \backslash \{y\}}\|\mu_{y'} - \mu_0\|_2^2 + \|\mu_y - \mu_0\|_2^2 \big) + 3\gs \varepsilon, \\
%    & G(x) := x / (x + (1 - x) \cdot \exp(-\Delta / 8)), \qquad \Delta := \max_{y' \in \cY} |\|\mu_y\|_2^2 - \|\mu_{y'}\|_2^2|. 
%\end{align*}
%
Note that the lower bound above (with an optimal choice of $\cU$) converges to 1 as $\eta \to \infty$.
In addition, for a sufficiently large $\eta$ it holds that
\begin{align*}
	\cP(x_T, y) \geq 1 - \frac{-C_0 - \mbox{logit} (\cP(x_0, y)) + \log \eta}{\eta C_1}, 
\end{align*}  
where $C_0 = \min_{y' \in \cY} (1 - e^{-T}) \langle \mu_y, \mu_y - \mu_{y'} \rangle$, $C_1 = e^{-\Delta / 8} (\|\mu_y - \mu_0\|_2^2 - 3 \varepsilon)$, and $\mbox{logit}(p) = \log (p / (1 - p))$.  
\end{theorem}

% \yuting{add some explanations of this result?}
The first part of Theorem \ref{thm:DDIM-confidence-quan} quantifies the effect of guidance strength on $\cP(x_t, y)$ and provides lower bounds with respect to the non-guided process $\cP(z_t, y)$ everywhere along the diffusion path. 
We note that this lower bound serves as an initial attempt and might be still far from tight. We leave the improvement to future works. 
The second part of Theorem \ref{thm:DDIM-confidence-quan} implies that $\cP(x_T, y) \to 1$ as $\eta \to \infty$, and the convergence rate is at least $1 - O(\eta^{-1} \log \eta)$.  
In another word, if the guidance strength is chosen to be very large, then the classification confidence will be close to one.

\subsection{Effect on the DDPM sampler}
\label{sec:effect-DDPM}

We then switch to consider the DDPM sampler, and we  compare in this section $\cP(\bar{x}_t, y)$ and $\cP(\bar{z}_t, y)$, where we recall that $\{\bar x_t\}_{0 \leq t \leq T}$ and $\{\bar z_t\}_{0 \leq t \leq T}$ are defined respectively in Eq.~\eqref{eq:dxt-bar-long} and \eqref{eq:unguided}. A notable distinction with the DDIM result arises in the need for an SDE comparison theorem, which we state as Lemma \ref{lemma:SDE-comparison-theorem} in the appendix. 
Lemma \ref{lemma:SDE-comparison-theorem} enables us to establish the following theorem, the proof of which can be found in Appendix \ref{sec:proof-thm:DDPM-confidence}. 
\begin{theorem}
\label{thm:DDPM-confidence}
    %We adopt both model \eqref{model:GMM} and Assumption \ref{assumption:confidence}. We also assume  $\langle z_0, \mu_y - \mu_{y'}  \rangle \leq \langle x_0, \mu_y - \mu_{y'} \rangle$ for all $y' \in \cY$.
   We assume the assumptions of Theorem \ref{thm:DDIM-confidence}, 
   then for any $\gs \geq 0$, almost surely we have $$\cP(\bar x_t, y) \geq \cP(\bar z_t, y)$$ for all $t \in [0, T]$.  
\end{theorem}

We also develop a quantitative comparison, presented as Theorem \ref{thm:DDPM-confidence-quan} below, the proof of which is deferred to Appendix \ref{sec:proof-thm:DDPM-confidence-quan}.  

\begin{theorem}
    \label{thm:DDPM-confidence-quan}
    We assume the conditions of Theorem \ref{thm:DDIM-confidence}. Then, for any $\gs \geq 0$, almost surely we have 
    \begin{align*}
       \cP(\bar x_T, y) \geq \frac{\cP(\bar z_T, y)}{\cP(\bar z_T, y) + (1 - \cP(\bar z_T, y)) \cdot \exp(-\bar \cU)} \geq \cP(\bar z_T, y),  %\cP(\bar x_t, y) \geq  \frac{\cP(\bar z_t, y)}{\cP(\bar z_t, y) + (1 - \cP(\bar z_t, y)) \cdot \exp\left( -\bar M (e^{t - T} - e^{-t - T}) \right)}
    \end{align*}
    where $\bar \cU$ is any non-negative number that satisfies 
    \begin{align*}
    	\bar\cU < e^{-T}\langle \bar x_0 - \bar z_0, \mu_y - \mu_{y'} \rangle + \gs (1 - e^{-2T}) e^{-\Delta / 8} \min \big\{\cF\big( \max_{0 \leq t \leq T} \cP(\bar z_t, y),\, e^{T} \bar \cU \big),\, \xi_w\big\} (\|\mu_y - \mu_0\|_2^2 - 3 \varepsilon),
    \end{align*}
    where we recall that $(\cF, \xi_w, \Delta)$ are defined in Eq.~\eqref{eq:cF-xi-Delta}. One can verify that the above lower bound (with an optimal choice of $\bar \cU$) approaches 1 as $\gs$ tends to infinity. If we fix the path initialization and the Brownian motion realization and only set $\gs \to \infty$, then the convergence rate is at least $1 - O(\gs^{-e^{-T}} (\log \gs)^{2e^{-T}})$. 
    %
    %for all $t \in [0, T]$. Here, the parameters are defined as  
    %
    %\begin{align*}
    %    & \bar M := \gs (1 - \bar{G}(q_T(\bar x_0, y))) (\|\mu_y - \mu_0\|_2^2 - 3 \varepsilon), \\
  %      & \bar G(x) := x / (x + (1 - x) \cdot \exp(-e^{-T} \bar U - \Delta / 2)), \\
 %       & \bar U := (e^T - e^{-T}) (\|\mu_y\|_2 \cdot \|\mu_y - \mu_{y'}\|_2 + \gs \|\mu_y - \mu_0\|_2^2 + \gs \sup_{y' \neq y} \|\mu_{y'} - \mu_0\|_2^2 + 3\gs \varepsilon) \\
 %   & \qquad + \sup_{y' \neq y, 0 \leq t \leq T} \Big| \int_0^t \sqrt{2} e^{s} \langle \diff B_s, \mu_y - \mu_{y'} \rangle \Big|, \\
%    & \Delta := \max_{y' \in \cY} |\|\mu_y\|_2^2 - \|\mu_{y'}\|_2^2|. 
%    \end{align*}
\end{theorem}

Theorems~\ref{thm:DDPM-confidence} and \ref{thm:DDPM-confidence-quan} are counterparts of the results for the DDIM sampler that we have established in Section \ref{sec:effect-DDIM}, indicating that adding guidance only increases the classification confidence for the DDPM sampler. Due to the stochastic nature of the DDPM sampler, the results in this section only hold almost surely.

%
%\begin{align*}
%	& \diff \langle \bar{x}_t, \mu_y \rangle = -\langle \bar{x}_t, \mu_y  \rangle + 2e^{-(T - t)} \|\mu_y\|_2^2 + 2we^{-(T - t)}\|\mu_y\|_2^2 \cdot (1 - q_{T - t}(\bar x_t, y)) + \sqrt{2} \langle \diff B_t, \mu_y \rangle, \\
%	& \diff \langle \bar{z}_t, \mu_y \rangle = -\langle \bar{z}_t, \mu_y \rangle + 2e^{-(T - t)} \|\mu_y\|_2^2 + \sqrt{2} \langle \diff B_t, \mu_y \rangle. 
%\end{align*}
%

%\begin{theorem}
%	Let $\bar{x}_0 = \bar{z}_0$. Then for any $w \geq 0$, almost surely the following is true for all $t \in [0, T]$: 
	%
%	\begin{align*}
%		\langle \bar{x}_t, \mu_y \rangle \geq \langle \bar{z}_t, \mu_y \rangle, \qquad \langle \bar{x}_t, \mu_{y'} \rangle \leq \langle \bar{z}_t, \mu_{y'} \rangle, \qquad y' \in \cY, y' \neq y. 
%	\end{align*}
%\end{theorem}

\subsection{Special case: GMM with two clusters}

The results presented in Sections \ref{sec:effect-DDIM} and \ref{sec:effect-DDPM} are derived based on Assumption \ref{assumption:confidence}. 
It turns out that we can further relax our assumptions when the number of Gaussian components is two (i.e., $|\cY| = 2$), which we report in this section. 

Without loss, we let $\cY = \{1, 2\}$, and assume guidance is towards the cluster that has label $1$. Correspondingly, the GMM considered here admits the following representation: 
\begin{align*}
    w_1 \normal(\mu_1, I_d) + w_2 \normal(\mu_2, I_d), 
\end{align*}
where $w_1, w_2 \in \RR_{\geq 0}$ satisfies $w_1 + w_2 = 1$. 

To summarize, in order to establish a similar set of results for the two-component GMM, we only require the second and the third points of Assumption \ref{assumption:confidence}. 
We collect results for the DDIM and the DDPM samplers separately below as Theorems \ref{thm:DDIM-two-clusters} and \ref{thm:DDPM-two-clusters}. 
We prove them in Appendices \ref{sec:proof-thm:DDIM-two-clusters} and \ref{sec:proof-thm:DDPM-two-clusters}, respectively. 

\begin{theorem}
\label{thm:DDIM-two-clusters}
    We assume $|\cY| = 2$, as well as the second and the third points of Assumption \ref{assumption:confidence}. Then the following statements regarding the DDIM sampler are true: 
    \begin{enumerate}
        \item If $\langle x_0, \mu_1 - \mu_2 \rangle \geq \langle z_0, \mu_1 - \mu_2 \rangle$, then $\cP(x_t, 1) \geq \cP(z_t, 1)$ for all $t \in [0, T]$. 
        \item If $\langle x_0, \mu_1 - \mu_2 \rangle \geq \langle z_0, \mu_1 - \mu_2 \rangle$, then 
        \begin{align*}
            \cP(x_T, 1) \geq \frac{\cP(z_T, 1)}{\cP(z_T, 1) + (1 - \cP(z_T, 1)) \cdot \exp(-\cU)} \geq \cP(z_T, 1), %\cP(x_t, 1) \geq \frac{\cP(z_t, 1)}{\cP(z_t, 1) + (1 - \cP(z_t, 1)) \cdot \left( - (e^{-(T - t)} - e^{-T}) M_1 \right)}, 
        \end{align*}
        where $\cU$ is any non-negative number that satisfies 
        \begin{align*}
        	\cU < 2\langle x_0 - z_0, \mu \rangle + 4 \gs e^{-\Delta_1 / 8} \|\mu\|_2^2 (1 - e^{-T}) \min \big\{\cF\big( \max_{0 \leq t \leq T} \cP(z_t, 1),\, \cU \big), \, w_2 \big\}. 
        \end{align*}
        In the above display, $\cF(\cdot)$ is defined in Eq.~\eqref{eq:cF-xi-Delta}, and $\Delta_1 = |\|\mu_1\|_2^2 - \|\mu_2\|_2^2|$. The lower bound above approaches one as $\gs \to \infty$. Furthermore, the convergence rate is at least $1 - O(\gs^{-1}(\log \gs)^2)$.
        %
        %\begin{align*}
        %    & M_1 := 4\gs (1 - \max \{G_1(g_1(\cP(x_0, y))), G_1(w_1)\}) \|\mu\|_2^2, \\
       %     & G_1(x) := x / (x + (1 - x) \cdot \exp(-\Delta_1 / 8)), \\
        %    & g_1(x) := x / (x + (1 - x) \cdot \exp( -(1 - e^{-T}) U_1 )), \\
       %     & U_1 := \|\mu_1\|_2^2 - \langle \mu_1, \mu_2 \rangle + \gs \|\mu_1 - \mu_2\|_2^2, \qquad \Delta_1 := \big| \|\mu_1\|_2^2 - \|\mu_2\|_2^2 \big|. 
       % \end{align*}
    \end{enumerate}
\end{theorem}

\begin{theorem}
\label{thm:DDPM-two-clusters}
    We assume the conditions of Theorem \ref{thm:DDIM-two-clusters}, and consider the DDPM sampler. Then the following statements hold almost surely: 
    \begin{enumerate}
        \item If $\langle \bar x_0, \mu_1 - \mu_2 \rangle \geq \langle \bar z_0, \mu_1 - \mu_2\rangle$, then $\cP(\bar x_t, 1) \geq \cP(\bar z_t, 1)$ for all $t \in [0, T]$. 
        \item If $\langle \bar x_0, \mu_1 - \mu_2 \rangle \geq \langle \bar z_0, \mu_1 - \mu_2\rangle$, then  for all $t \in [0, T]$
        \begin{align*}
            \cP(\bar x_T, 1) \geq \frac{\cP(\bar z_T, 1)}{\cP(\bar z_T, 1) + (1 - \cP(\bar z_T, 1)) \cdot \exp(-\bar \cU)} \geq \cP(\bar z_T, 1),  %\cP(\bar x_t, 1) \geq \frac{\cP(\bar z_t, 1)}{\cP(\bar z_t, 1) + (1 - \cP(\bar z_t, 1)) \cdot \exp\left( -\bar M_1 (e^{-T + t} - e^{-T - t}) \right)}, 
        \end{align*}
        where $\bar \cU$ is any non-negative number such that 
        \begin{align*}
        	\bar\cU < 2e^{-T}\langle \bar x_0 - \bar z_0, \mu \rangle + 4\gs e^{-\Delta_1 / 8} \|\mu\|_2^2  (1 - e^{-2T}) \min \Big\{\cF \big( \max_{0 \leq t \leq T} \cP(\bar z_t, 1), e^T\bar \cU \big), w_2\Big\},
        \end{align*}
        where we recall that $\cF(\cdot)$ is defined in Eq.~\eqref{eq:cF-xi-Delta}, and $\Delta_1 = |\|\mu_1\|_2^2 - \|\mu_2\|_2^2|$. The lower bound in the theorem converges to 1 as $\gs \to \infty$. Furthermore, the convergence rate is at least $1 - O(\gs^{-e^{-T}}(\log \gs)^{2 e^{-T}})$. 
        %
        %\begin{align*}
        %    & \bar M_1 :=  8\gs (1 - \bar{G}_1(q_T(\bar x_0, 1)))\|\mu\|_2^2, \\
       %     & \bar G_1(x) := x / (x + (1 - x) \cdot \exp \left( -e^{-T} \bar U_1 - \Delta_1 / 2 \right)), \qquad \Delta_1 := \big| \|\mu_1\|_2^2 - \|\mu_2\|_2^2 \big|, \\
      %      & \bar U_1 := (e^T - e^{-T}) (\|\mu_1\|_2^2 - \langle \mu_1, \mu_2 \rangle + 4\gs \|\mu\|_2^2) + \sup_{0 \leq t \leq T} \left| 
    %2e^{-t} \int_0^t \sqrt{2} \langle \diff B_s, \mu\rangle\right|.
    %    \end{align*}
    \end{enumerate}
\end{theorem}
The results above confirm that diffusion model with guidance always promotes classification confidence in two-component GMMs. It is interesting to note that augmenting a center component to the two-component GMM leads to complicated consequence in terms of guidance; see details in Section~\ref{sec:example-curious}.

%\subsubsection{Special case II: GMM with aligned centers}

\section{Effect of guidance on distribution diversity}
\label{sec:diversity}

In this section, we investigate the impact of guidance on distribution diversity. 
%We consider two metrics for quantifying the diversity: the entropy of probability distributions and the Euclidean norm to the component center. 
%\subsection{Effect on entropy}
%
We propose to employ the \emph{differential entropy} of probability distributions to measure diversity \citep{shannon1948mathematical}.
This section exclusively concentrates on the DDIM sampler.
To define differential entropy, we denote by $Q(t, x)$ the probability density function of $x_t$, where we recall that $(x_t)_{0 \leq t \leq T}$ is defined in Eq.~\eqref{eq:dxt-long}.
For comparison, we also denote by $Q_0(t, x)$ the probability density function of the unguided process $(z_t)_{0 \leq t \leq T}$ defined in Eq.~\eqref{eq:unguided}.
We shall prove in appendix that the probability density functions exist for all $t \in [0, T]$ if we assume it exists at $t = 0$. 
Our objective is to delineate the influence of diffusion guidance on the entropy functionals, as defined below:
\begin{align}
\label{eq:differentiable-entropy}
\begin{split}
	 & H(t) := - \int Q (t, x) \log Q (t, x) \diff x, \qquad 0 \leq t \leq T, \\
      & H_0(t) := - \int Q_0(t, x) \log Q_0(t, x) \diff x, \qquad 0 \leq t \leq T.
\end{split}
\end{align}
Intuitively, a high entropy indicates that the distribution is spread in the space, while on the contrary, a low entropy is oftentimes associated with relatively concentrated distributions.

We propose to analyze the evolution of the entropy using the Fokker-Planck equation \citep{fokker1914mittlere}, which characterizes the distributional evolution of the DDIM sampler. Readers may refer to Lemma~\ref{lemma:Fokker-Planck} for a detailed exposure. 
%For this part we assume
%
%\begin{align*}
%	\mu_{\ast} \overset{d}{=} \sum_{y \in \cY} w_y \normal(\mu_y, \Sigma). 
%\end{align*}

\begin{lemma}[Fokker–Planck equation]
\label{lemma:Fokker-Planck}
	Consider the $d$-dimensional SDE
	\begin{align*}
		\diff X_t = \mu(t, X_t) \diff t + \sigma(t, X_t) \diff B_t, 
	\end{align*}
 where $\mu, \sigma: \RR_{\geq 0} \times \RR^d \mapsto \RR^d$ satisfies $\|\mu(t, x) - \mu(t, y)\|_2 + \|\sigma(t, x) - \sigma(t, y)\|_2 \leq C \|x - y\|_2$ for some constant $C$ and all $x, y \in \RR^d$.
 Assume that the probability density function (w.r.t. the Lebesgue measure) of $X_t$ exists for all $t \in [0, T]$, and denote by $p(t, x)$ the probability density function for $X_t$. We also assume all the relevant functions are continuously differentiable, then 
	\begin{align*}
		\frac{\partial}{\partial t} p(t, x) = - \sum_{i = 1}^d \frac{\partial}{\partial x_i} \left[ \mu_i(t, x) p(t, x) \right] + \sum_{i = 1}^d \sum_{j = 1}^d \frac{\partial^2}{\partial x_i \partial x_j} \left[ D_{ij}(t, x) p(t, x) \right], 
	\end{align*}
	where $D(t, x) = \sigma(t, x) \sigma(t, x)^{\top} / 2$. 
	
\end{lemma}

%

%\subsubsection{DDIM sampler}

Our theorem is stated below. A heuristic derivation based on the Fokker-Planck equation is in Appendix \ref{sec:FK-derivation}, and a formal proof of the theorem is postponed to Appendix \ref{sec:proof-thm:DDIM-diversity}.   
\begin{theorem}
\label{thm:DDIM-diversity}
	We assume that both $x_0$ and $z_0$ have %continuously differentiable 
	probability density functions with respect to the Lebesgue measure,  and the corresponding differential entropies exist and are finite, satisfying $H_0(0) \geq H(0)$. We also assume model \eqref{model:GMM}, $\Sigma$ is non-degenerate, as well as the second point of Assumption \ref{assumption:confidence}. Then for all $0 \leq t \leq T$, it holds that $H_0(t) \geq H(t)$. 
\end{theorem}

In contrast to the results presented in Section \ref{sec:pre-confidence}, Theorem \ref{thm:DDIM-diversity} does not require an isotropic covariance matrix, and places no assumptions on the component centers. 
Mild regularity condition is imposed on the process initialization to ensure the existence of the differential entropy. 

Setting $t = T$, Theorem~\ref{thm:DDIM-diversity} says that the generated distribution under diffusion guidance has lower entropy compared to that without guidance. This corroborates the common observation displayed in Figure~\ref{fig:three-components}: diffusion guidance reduces diversity of the generated samples.

%\subsection{Attraction towards component center}

%With regard to our second criterion, we assess the Euclidean distance to the component center $\mu_y$. 

%\subsubsection{DDIM sampler}

%We begin with an examination of the DDIM sampler. 

%Observe that
%
%\begin{align*}
%    & \frac{\diff \langle x_t, \mu_y \rangle}{\diff t} = e^{-(T - t)} (1 + w - w q_{T - t}(x_t, y)) \|\mu_y\|_2^2 - w e^{-(T - t)} \sum_{y' \neq y} q_{T - t}(x_t, y') \langle \mu_y, \mu_{y'} \rangle, \\
%    & \frac{\langle \diff x_t, x_t \rangle}{\diff t} = e^{-(T - t)} (1 + w - w q_{T - t}(x_t, y)) \langle \mu_y, x_t \rangle - w e^{-(T - t)} \sum_{y' \neq y} q_{T - t}(x_t, y') \langle \mu_{y'}, x_t \rangle. 
%\end{align*}

\section{Effect of guidance on discretized process}
\label{sec:discretization}

In practice, it is essential to employ discretization to approximate the continuous-time processes. To be specific, the algorithmic implementations of processes \eqref{eq:dxt-long} and \eqref{eq:dxt-bar-long} are as follows: 
\begin{align}
%\label{eq:discretized-process}
    & X_{k + 1} = X_k + \delta_k \left( X_k + \nabla_x \log p_{T - t_k} (X_k, y) + \gs \nabla_x \log p_{T - t_k} (y \mid X_k) \right), \label{eq:X-k+1-X-k}\\
    & \bar X_{k + 1} = \bar X_k + \delta_k \left( \bar X_k + 2\nabla_x \log p_{T - t_k}(\bar X_k, y) + 2\gs \nabla_x \log p_{T - t_k}(y \mid \bar X_k) \right) + \sqrt{2 \delta_k} W_k. \label{eq:bar-X-k+1-X-k}
\end{align}
In the above display, $W_k \sim \normal(0, I_d)$ and is independent of the previous iterates, $0 = t_0 < t_1 < \cdots < t_K \leq T$, $\delta_k 
> 0$ and $t_{k + 1} = \sum_{i = 0}^{k} \delta_i$ for all $k = 0, 1, \cdots, K - 1$.  

Analogously, to set up comparison, we also consider the discretized processes without guidance: 
\begin{align}
    & Z_{k + 1} = Z_k + \delta_k \left(Z_k + \nabla_x \log p_{T - t_k} (Z_k, y) \right), \\
    & \bar Z_{k + 1} = \bar Z_k + \delta_k \left( \bar Z_k + 2 \nabla_x \log p_{T - t_k} (\bar Z_k, y) \right) + \sqrt{2 \delta_k} W_k. 
\end{align}
We unify the discretization schemes for both the guided and the unguided processes to facilitate meaningful comparisons.
In the current regime, we are able to establish results related to classification confidence and distribution diversity, which we collect below. 
We utilize the widely recognized Euler discretization scheme to present our results. 
However, we note that with minimal adjustments, our findings can extend to accommodate other discretization schemes, for instance the ones based on the exponential integrator.  

\subsection{Results for the DDIM sampler}

We first investigate the classification confidence, and establish the following theorem. 
We postpone the proof of the theorem to Appendix \ref{sec:proof-thm:DDIM-confidence-dis}.  
\begin{theorem}
    \label{thm:DDIM-confidence-dis}
    We assume model \eqref{model:GMM} and Assumption \ref{assumption:confidence}.
    We also assume $\langle X_0, \mu_y - \mu_{y'} \rangle \geq \langle Z_0, \mu_y - \mu_{y'}\rangle$ for all $y' \in \cY$. 
    Then the following statements are true:
    \begin{enumerate}
        \item For all $k \in \{0\} \cup [K]$, it holds that $\cP(X_k, y) \geq \cP(Z_k, y)$.
        \item We let $\Delta_{\max} = \max_{j \in \{0\} \cup [K - 1]} \delta_j$, then for any $\gs \geq 0$, it holds that 
        \begin{align*}
        	\cP(X_K, y) \geq \frac{\cP(Z_K, y)}{\cP(Z_K, y) + (1 - \cP(Z_K, y)) \cdot \exp(-\cU)} \geq \cP(Z_K, y), 
        \end{align*}
        where $\cU > 0$ is any number that satisfies
        \begin{align*}
        	& \cU - \langle X_0 - Z_0, \mu_y - \mu_{y'} \rangle  \\
	< \,&  e^{-\Delta_{\max}} (1 - e^{-T})\Big(  \gs e^{-\Delta / 8}(\|\mu_y - \mu_0\|_2^2 - 3 \varepsilon) \cdot \min \{\cF(\max_{0 \leq k \leq K} \cP(Z_k, y), \cU), \xi_w\} \Big),
        \end{align*}
        where we recall that $(\cF, \xi_w, \Delta)$ are defined in Eq.~\eqref{eq:cF-xi-Delta}. Furthermore, as $\gs \to \infty$, we have $ \cP(X_K, y) \geq 1 - O(\gs^{-1} (\log \gs)^2)$. 
        %
        %\begin{align*}
        %    \cP(X_k, y) \geq \frac{\cP(Z_k, y)}{\cP(Z_k, y) + (1 - \cP(Z_k, y)) \cdot \exp \big( -e^{-\max_{j \in [K]} \delta_{j - 1}}(e^{-T + t_k} - e^{-T}) M \big)},
        %\end{align*}
        %
        %where we recall that $M$ is defined in Theorem \ref{thm:DDIM-confidence-quan}. 
    \end{enumerate}
\end{theorem}

From Theorem \ref{thm:DDIM-confidence-dis}, we see that the application of discretization preserves the boosting effect on classification confidence induced by diffusion guidance. 
%It is easily seen that discretization preserves the boost of the prediction confidence under guidance.
Yet we note that the discretization step sizes $\{\delta_k\}_{0 \leq k \leq K - 1}$ interact with the increment of the classification confidence: Large step size leads to a marginal increase, as demonstrated by the second point of Theorem~\ref{thm:DDIM-confidence-dis}. 
% We can establish analogous results for the discretized DDPM sampler, which is provided in Appendix~\ref{sec:ddpm_dis}.

In terms of the effect of guidance on distribution diversity,  under mild additional assumptions on the discretization scheme, we are able to establish results on differential entropy for the discretized DDIM sampler that is similar to Theorem \ref{thm:DDIM-diversity}. To set up the stage, for $k \in \{0, 1, \cdots, K\}$, we denote by $\cH(k)$ the differential entropy of $X_k$\footnote{Namely, $\cH(k) = -\int p_k(x)\log p_k(x) \diff x$, where $p_k(\cdot)$ is the density function of $X_k$. We shall prove in Appendix \ref{sec:proof-thm:DDIM-entropy-dis} that under mild assumptions such differential entropy exists.} and denote by $\cH_0(k)$ that of $Z_k$. 
Our main theorem for this part shows that under mild regularity conditions, it holds that $\cH(k) \leq \cH_0(k)$ for all $k \in \{0, 1, \cdots, K\}$. 
Theorem~\ref{thm:DDIM-entropy-dis} resembles the conclusion of Theorem~\ref{thm:DDIM-diversity} by requiring relatively small step sizes.
The proof of Theorem \ref{thm:DDIM-entropy-dis} is postponed to Appendix \ref{sec:proof-thm:DDIM-entropy-dis}.

\begin{theorem}
\label{thm:DDIM-entropy-dis}
    We assume both $X_0$ and $Z_0$ have  %continuously differentiable 
    probability density functions with respect to the Lebesgue measure, and the corresponding differential entropies exist and are finite, satisfying $\cH(0) \leq \cH_0(0)$.
    We also assume model \eqref{model:GMM}, the second point of Assumption \ref{assumption:confidence}, and that $\Sigma$ is non-degenerate. 
    In addition, for all $k \in \{0\} \cup [K - 1]$, we require the step sizes are small enough such that
    \begin{align*}
        1 + \delta_k > \frac{\delta_k}{\sigma_{\min}(\Sigma) \wedge 1 } + \frac{\delta_k \gs \sup_{y' \in \cY} \|\mu_{y'}\|_2^2}{\sigma_{\min}(\Sigma)^2 \wedge 1}, \qquad \delta_k + \frac{\delta_k}{\sigma_{\min}(\Sigma) \wedge 1 } + \frac{\delta_k \gs \sup_{y' \in \cY} \|\mu_{y'}\|_2^2}{\sigma_{\min}(\Sigma)^2 \wedge 1} < 1 / 2, 
    \end{align*}
    where $\sigma_{\min}(\Sigma)$ is the minimum eigenvalue of $\Sigma$. 
    %We also assume model \eqref{model:GMM}, as well as the second and the third points of Assumption \ref{assumption:confidence}. 
    Then, for all $k \in \{0\} \cup [K]$ we have $\cH(k) \leq \cH_0(k)$. 
\end{theorem}

%While the range of the required step size is rather technical, we anticipate the discretization error interacts with the differential entropy of the generated distribution. Indeed, in the discretized DDIM samplers, there is also complicated interplay between the strength of the guidance with the discretization step size.

\subsection{Results for the DDPM sampler}

As for the DDPM sampler, we can only establish results for classification confidence. The proof is deferred to Appendix \ref{sec:proof-thm:DDPM-confidence-dis}. 

\begin{theorem}
    \label{thm:DDPM-confidence-dis}
    We assume the conditions of Theorem \ref{thm:DDIM-confidence-dis}. Then we have the following results: 
    \begin{enumerate}
        \item For all $k \in \{0\} \cup [K]$, it holds that $\cP(\bar X_k, y) \geq \cP(\bar Z_k, y)$. 
        \item If additionally we assume $\Delta_{\max} = \max_{k \in \{0\} \cup[K - 1] } \delta_k \leq 1 / 2$, then 
        \begin{align*}
        	\cP(\bar X_K, y) \geq \frac{\cP(\bar Z_K, y)}{\cP(\bar Z_K, y) + (1 - \cP(\bar Z_K, y)) \cdot \exp(- e^{-2T} \bar \cU)} \geq \cP(\bar Z_K, y),
        \end{align*}
        where $\bar \cU \in \RR_{\geq 0}$ is any number that satisfies 
        \begin{align*}
	& \bar\cU - e^{-T - \Delta_{\max}} \langle \bar X_0 -  \bar Z_{0}, \mu_y - \mu_{y'} \rangle \\
	 < &\, \gs e^{-\Delta / 8} (e^{-T} - e^{-3T}) \min \{\cF\big( \max_{0 \leq k \leq K} \cP(\bar Z_k, y), \bar \cU \big), \xi_w\}  (\|\mu_y - \mu_0\|_2^2 - 3 \varepsilon). \nonumber 
\end{align*}
We recall that $(\cF, \Delta, \xi_w)$ are defined in Eq.~\eqref{eq:cF-xi-Delta}. In addition, as $\gs \to \infty$, 
the convergence rate is at least $1 - O(\gs^{-e^{-T}}(\log \gs)^{2 e^{-T}})$. 
        %
        %\begin{align*}
        %    \cP(\bar X_k, y) \geq \frac{\cP(\bar Z_k, y)}{\cP(\bar Z_k, y) + (1 - \cP(\bar Z_k, y)) \cdot \exp \left( - \bar M_d (e^{-T } - e^{-T - 2 t_k})\right)}
        %\end{align*}
        %
        %for all $k \in \{0\} \cup [K]$. In the above display, 
        %
        %\begin{align*}
        %    & \bar M_d = \gs(1 - \bar G_d(q_T(\bar X_0, y)))(\|\mu_y - \mu_0\|_2^2 - 3 \varepsilon),  \\
        %    & \bar G_d(x) = x / (x + (1 - x) \cdot \exp (-e^{-T} \bar U_d - \Delta / 2)), \\
        %   & \bar U_d =   (e^T - e^{-T}) \left(\|\mu_y\|_2 \|\mu_y - \mu_{y'}\|_2 + 3\gs \varepsilon + \gs \|\mu_y - \mu_0\|_2^2 + \gs \sup_{y' \neq y} \|\mu_{y'} - \mu_0\|_2^2 \right) \\
    %& \,\,\,\,\,\,\,\,\,\,\,\, + e^{T}\sup_{y' \neq y, k \in \{0\} \cup [K - 1]} \Big| \sum_{i = 0}^k \sqrt{2 \delta_i} \prod_{j = i + 1}^{k} (1 - \delta_j) \langle W_i, \mu_y - \mu_{y'} \rangle \Big|,
  %      \end{align*}
        %
 %       where we recall that $\Delta$ is defined in Theorem \ref{thm:DDIM-confidence-quan}. 
    \end{enumerate}
\end{theorem}

\begin{remark}
    We can eliminate the assumptions on the component centers for Theorem \ref{thm:DDIM-confidence-dis} and \ref{thm:DDPM-confidence-dis} when $|\cY| = 2$. The proof is similar to that of Theorem \ref{thm:DDPM-confidence-quan} and \ref{thm:DDPM-two-clusters}, and we skip it for the sake of simplicity. 
\end{remark}

\section{A curious example of strong guidance}
\label{sec:example-curious}

In this section, we illustrate a possible negative impact of strong guidance under discretized backward sampling in a three-component GMM. 
This discovery complements the theoretical study in the preceding sections, and reveals that strong guidance can lead to heavy unexpected distribution distortions, such as splitting one Gaussian component into two. We utilize the following GMM with mean vectors symmetric about zero and aligned, i.e.,
\begin{align}
\label{eq:gmm_negative}
\mu_{\rm neg} \overset{d}{=} \frac{1}{3} {\sf N}(-\mu, I_d) + \frac{1}{3} {\sf N}(0, I_d) + \frac{1}{3} {\sf N}(\mu, I_d).
\end{align}
Here, $\mu \neq \vec{0}_d$ is one of the cluster mean vectors in $\RR^d$, and its magnitude determines the separation between the three clusters. 
Note that the first item of Assumption~\ref{assumption:confidence} does not hold for $\mu_{\rm neg}$, when the process is guided towards the central component. Intuitively, the first item of Assumption~\ref{assumption:confidence} implies that the cluster mean vectors in the GMM should be approximately orthogonal to each other. 
However, it is clear that the mean vectors are on the same line in $\mu_{\rm neg}$. 
A formal verification of this claim can be found in Appendix~\ref{pf:negative_result}.

We study the generation of samples corresponding to the center component ${\sf N}(0, I_d)$. For simplicity, we shall focus on the discretized DDIM backward process \eqref{eq:X-k+1-X-k}. The following result demonstrates a phase transition in the behavior of $\inner{X_k}{\mu}$ as the strength of guidance gradually increases.
\begin{proposition}
\label{prop:phase_change}
Consider the Gaussian mixture model in \eqref{eq:gmm_negative}. There exist constants $\gs_0 \leq \frac{1}{\norm{\mu}_2^2 \max \delta_k}$ and $\gs_0'$ that depend on the discretization step sizes $\{\delta_k\}_{k=0}^{K-1}$, such that for any $k$ verifying $e^{-T + t_k} \geq 1/2$,

\begin{itemize}
	\item {\it (Convergent phase)} when $\gs \leq \gs_0$, $|\inner{X_{k+1}}{\mu}| < |\inner{X_k}{\mu}|$ for $\inner{X_k}{\mu} \neq 0$;

	\item {\it (Splitting phase)} when $\gs \geq \gs_0'$,
\begin{align*}
|\inner{X_{k+1}}{\mu}| > |\inner{X_k}{\mu}| & \quad \text{if}\quad |\inner{X_k}{\mu}| \in (0, a]; \\
|\inner{X_{k+1}}{\mu}| < |\inner{X_k}{\mu}| & \quad \text{if} \quad |\inner{X_k}{\mu}| > b,
\end{align*}
where $a$ and $b$ are positive and increase as the strength $\gs$ increases. 
\end{itemize}
\end{proposition}
The proof is deferred to Appendix~\ref{pf:negative_result}. We discuss interpretations of Proposition~\ref{prop:phase_change} below.
\paragraph{A phase shift due to strong guidance} A large value of $|\inner{X_k}{\mu}|$ indicates a strong likelihood that $X_k$ will be classified into one of the side components ${\sf N}(\pm \mu, I_d)$ rather than the center component. Therefore, Proposition~\ref{prop:phase_change} implies that there exists a phase shift for the placement of the probability mass corresponding to the center component. With weak guidance, the center component becomes condensed. However, under too strong guidance, the center component tends to vanish as the generated samples are pushed towards side centers, as illustrated in Figure~\ref{fig:negative_effect}.

\paragraph{Influence of the discretization step size and strong guidance} We also observe in Figure~\ref{fig:negative_effect} that the phase shift phenomenon entangles with the discretization step size: Coarse discretization is prone to enter the {\it (Splitting phase)} with strong guidance. This supports our theory on the ranges of the thresholds $\gs_0$ and $\gs_0'$. In the extreme case with $\delta_k = 0$ for all $k$, that is, we are using the exact continuous-time backward process for generating samples, only the {\it (Convergent phase)} stays.

On the other hand, with strong guidance, the center component is split into two symmetric components. The separation between the two symmetric components increases as guidance $\gs$ increases. This also corroborates our theory on the values of $a$ and $b$.

We remark that the convergence in the {\it (Convergent phase)} can be geometrically fast as shown in an improved result in Lemma~\ref{lemma:phase_change}. Under a discretized DDPM sampling scheme, we can also observe the phase shift subject to strong guidance as shown in Figure~\ref{fig:3d-align-ddpm}.

\begin{figure}[h]
\centering
\includegraphics[width = 0.65\textwidth]{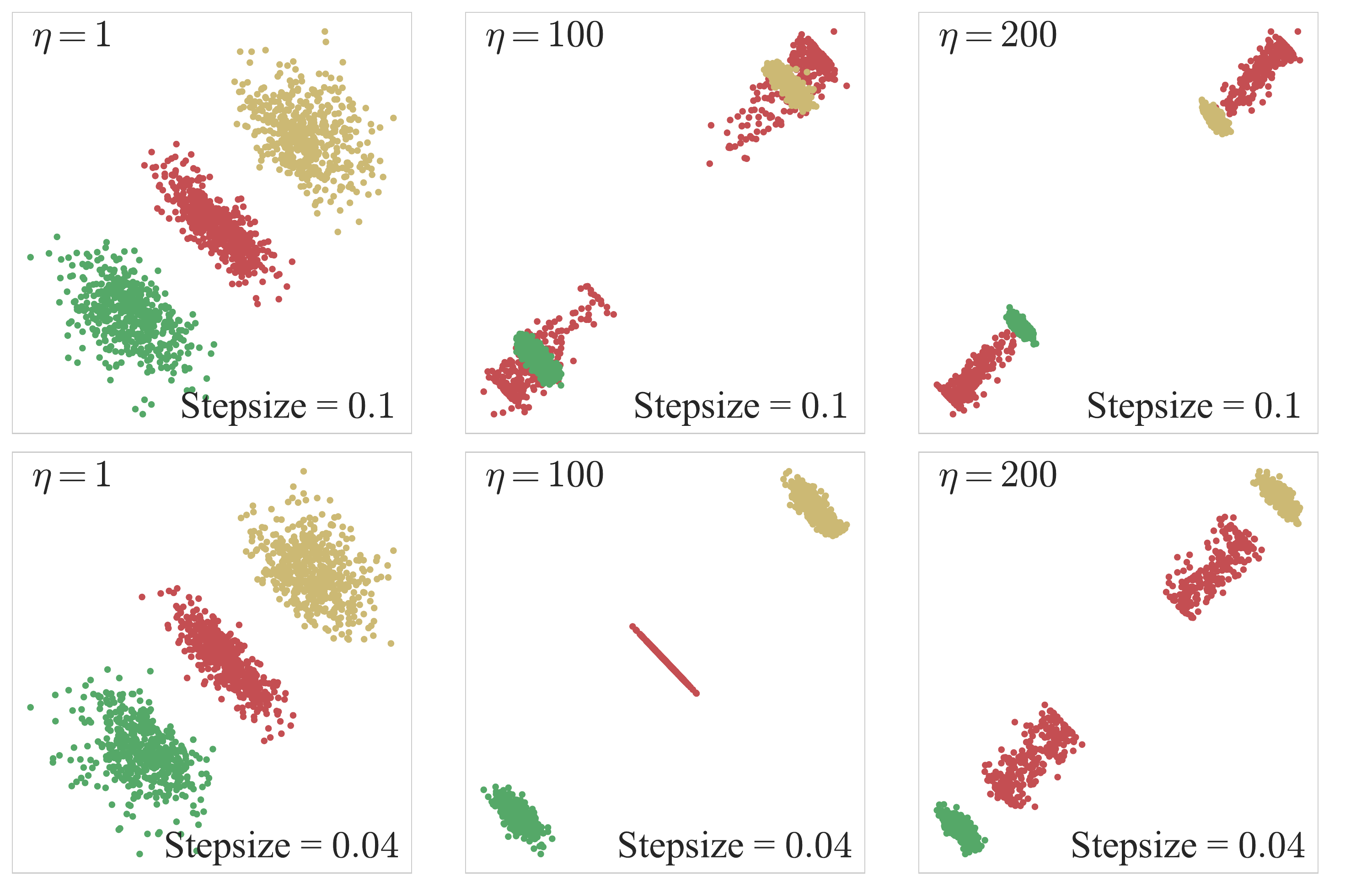}
\caption{Illustration of the negative effect of large guidance. In this plot, we set $\mu = [2, 2]^\top$ in $\mu_{\rm neg}$, and increase the guidance strength $\gs$ from left to right. The upper row uses a relatively large discretization step size ($\delta_k = 0.1$ for all $k \in \{0\} \cup [K - 1]$). Under strong guidance, the center component splits into two clusters at an earlier stage. The bottom row uses a much smaller discretization step size ($\delta_k = 0.04$ for all $k \in \{0\} \cup [K - 1]$); the center component then splits only with a much larger guidance strength.}
\label{fig:negative_effect}
\end{figure}

\section{Conclusion and discussion}

In this paper, we establish the theoretical foundation 
for diffusion guidance in the context of sampling from Gaussian mixture models with shared covariance matrices. 
Under a set of mild regularity conditions, we show that guidance increases the prediction confidence along every realized path, while decreasing the overall distribution diversity.
Our analysis is based on ODE and SDE comparison theorems, along with the Fokker-Planck equation that depicts the evolution of probability density functions. 

We list here several interesting future directions that deserve further investigation. First, the quantitative lower bounds we present in the paper might not be tight, and a more careful examination of the guidance effect is worthy of future studies. Secondly, due to technical reasons, we currently lack a characterization of the reduction in diversity that arises from guidance for the DDPM sampler. It is of great interest to derive similar guarantees for the DDPM sampler. 
Finally, we expect our framework to go beyond sampling from GMMs and we leave this extension to future work.

\subsection*{Acknowlegements}

Y.~Wei is supported in part by the NSF grants DMS-2147546/2015447, CAREER award DMS-2143215, CCF-2106778, and the Google Research Scholar Award.

%\begin{enumerate}
%    \item The quantitative lower bounds are far from tight.
%    \item Derive DP result for DDPM
%    \item Generalize to distributions beyond GMM
%\end{enumerate}

%\section{General distribution}

%Can we move beyond GMM? Some random thoughts here...

%Consider the following discretized processes: 
%
%\begin{align*}
%	& X_{k + 1} = X_k + \delta_k \cdot \left( X_k + \log p_{t_k}(X_k, y) + w \nabla_x \log p_{t_k} (X_k, y) \right), \\
%	& \bar{X}_{k + 1} = \bar{X}_k + \delta_k \cdot \left( \bar X_k + 2\log p_{t_k}(\bar X_k, y) + 2w \nabla_x \log p_{t_k}(\bar X_k, y) \right) + \sqrt{2\delta_k} W_k, \\
%	& Z_{k + 1} = Z_k + \delta_k \cdot (Z_k + \log p_{t_k} (Z_k, y)), \\
%	& \bar{Z}_{k + 1} = \bar{Z}_k + \delta_k \cdot (\bar{Z}_k + 2 \log p_{t_k} (\bar{Z}_k, y) ) + \sqrt{2\delta_k} W_k. 
%\end{align*}
%
%Consider the entropy for random vectors, which admits the following decomposition:
%
%\begin{align*}
%	H(X_1, X_2, \cdots, X_K) = H(X_1) + H(X_2 \mid X_1) + \cdots H(X_K \mid X_1, X_2, \cdots, X_{K - 1}). 
%\end{align*}
%

\bibliographystyle{plainnat}
\bibliography{ref.bib,reference-consistency}

\appendix

\section{Proofs related to confidence enhancement}
\label{sec:increase-confidence}

This section contains proofs pertinent to results on guidance improving prediction confidence.
We first prove Eq.~\eqref{eq:xt-muy-muyp}. Note that
\begin{align}
\label{eq:diff-x-mu-mu-equality}
\begin{split}
    & \frac{\diff}{\diff t}  \langle x_t, \mu_y - \mu_{y'} \rangle \\
    = & e^{-(T - t)} \|\mu_y\|_2^2 + \gs e^{-(T - t)} \|\mu_y - \mu_0\|_2^2  - \gs e^{-(T - t)} \sum_{y'' \in \cY} q_{T - t}(x_t, y'') \langle \mu_y - \mu_0, \mu_{y''} - \mu_0 \rangle \\
    & - e^{-(T - t)} \langle \mu_y, \mu_{y'} \rangle - \gs e^{-(T - t)} \langle \mu_y - \mu_0, \mu_{y'} - \mu_0 \rangle + \gs e^{-(T - t)} \sum_{y'' \in \cY} q_{T - t}(x_t, y'') \langle \mu_{y'} - \mu_0, \mu_{y''} - \mu_0 \rangle \\
    = & e^{-(T - t)} \|\mu_y\|_2^2 - e^{-(T - t)} \langle \mu_y, \mu_{y'} \rangle + \gs e^{-(T - t)} (1 - q_{T - t}(x_t, y)) \|\mu_y - \mu_0\|_2^2 \\
    & + \gs e^{-(T - t)} q_{T - t}(x_t, y') \|\mu_{y'} - \mu_0\|_2^2 + \cE_t,  %\\
    %& + we^{-(T - t)} q_{T - t}(x_t, y') \|\mu_{y'} - \mu_0\|_2^2 \\
    %= & (1 + w - wq_{T - t}(x_t, y) )e^{-(T - t)} \|\mu_y\|_2^2 - w e^{-(T - t)} \sum_{y' \neq y, y' \in \cY}  q_{T - t}(x_t, y') \langle \mu_y, \mu_{y'} \rangle \\
   %&  - (1 + w - w q_{T - t}(x_t, y)) e^{-(T - t)} \langle \mu_y, \mu_{y'}\rangle + we^{-(T - t)} q_{T - t}(x_t, y') \|\mu_{y'}\|_2^2  \\
    %& + we^{-(T - t)} \sum_{y'' \neq y, y'} q_{T - t}(x_t, y'') \langle \mu_{y''}, \mu_{y'} \rangle 
\end{split}
\end{align}
where 
\begin{align*}
    \cE_t = & - \gs e^{-(T - t)} \sum_{y'' \in \cY \backslash \{y\}} q_{T - t}(x_t, y'') \langle \mu_y - \mu_0, \mu_{y''} - \mu_0 \rangle - \gs  e^{-(T - t)}(1 - q_{T - t}(x_t, y)) \langle \mu_y - \mu_0, \mu_{y'} - \mu_0 \rangle
 \\
 & + \gs e^{-(T - t)} \sum_{y'' \in \cY \backslash \{y, y'\}} q_{T - t}(x_t, y'') \langle \mu_{y'} - \mu_0, \mu_{y''} - \mu_0 \rangle. 
\end{align*}
By triangle inequality and Assumption \ref{assumption:confidence} it holds that $|\cE_t| \leq 3\gs e^{-(T - t)} (1 - q_{T - t}(x_t, y)) \varepsilon$. This completes the proof of Eq.~\eqref{eq:xt-muy-muyp}. 

%$\cE_t = - we^{-(T - t)} \sum_{y'' \in \cY \backslash \{y\}} q_{T - t}(x_t, y'') \langle \mu_y - \mu_0, \mu_{y''} - \mu_0 \rangle + w e^{-(T - t)} \sum_{y'' \in \cY \backslash \{y, y'\}} q_{T - t}(x_t, y'') \langle \mu_{y'} - \mu_0, \mu_{y''} - \mu_0 \rangle$

\subsection{Proof of Theorem \ref{thm:DDIM-confidence}}
\label{sec:proof-thm:DDIM-confidence}

Observe that
\begin{align*}
    & \cP(x_t, y) = \frac{w_y  }{w_y + \sum_{y' \in \cY, y' \neq y} w_{y'} \exp \left( \langle x_t, \mu_{y'} - \mu_y \rangle - \|\mu_{y'}\|_2^2 / 2 + \|\mu_y\|_2^2 / 2 \right)}, \\
    & \cP(z_t, y) = \frac{w_y  }{w_y + \sum_{y' \in \cY, y' \neq y} w_{y'} \exp \left( \langle z_t, \mu_{y'} - \mu_y \rangle - \|\mu_{y'}\|_2^2 / 2 + \|\mu_y\|_2^2 / 2 \right)}. 
\end{align*}
According to Lemma \ref{lemma:compare-inner-product}, we have $\langle x_t, \mu_y - \mu_{y'} \rangle \geq \langle z_t, \mu_y - \mu_{y'} \rangle$ for all $y' \in \cY \backslash \{y\}$, hence $\cP(x_t, y) \geq \cP(z_t, y)$ for all $t \in [0, T]$. This completes the proof of Theorem \ref{thm:DDIM-confidence}.

\subsection{Proof of Theorem \ref{thm:DDIM-confidence-quan}}
\label{sec:proof-thm:DDIM-confidence-quan}

\subsubsection*{Proof of the first result}

The idea is to first establish an upper bound for $q_{T - t}(x_t, y)$, then in turn use it to lower bound the effect of guidance. 
Notice that
\begin{align}
\label{eq:upper-bound-tilde-q}
\begin{split}
    q_{T - t}(x_t, y) = & \frac{w_y}{w_y + \sum_{y' \neq y} w_{y'} \exp \left( e^{-(T - t)} \langle x_t, \mu_{y'} - \mu_y\rangle - e^{-2(T - t)} (\|\mu_{y'}\|_2^2 - \|\mu_y\|_2^2) / 2 \right) } \\
    = & \frac{\tilde{q}_{T - t}(x_t, y)}{\tilde{q}_{T - t}(x_t, y) + (1 - \tilde{q}_{T - t}(x_t, y)) \cdot \exp \left( - (e^{-2(T - t)} - e^{-(T - t)})(\|\mu_{y'}\|_2^2 - \|\mu_y\|_2^2) / 2 \right)} \\
    \leq & \frac{\tilde{q}_{T - t}(x_t, y)}{\tilde{q}_{T - t}(x_t, y) + (1 - \tilde{q}_{T - t}(x_t, y)) \cdot \exp\left( -\Delta / 8 \right)}, 
\end{split}
\end{align}
where
\begin{align}
\label{eq:def-tilde-q}
   \tilde q_{T - t}(x_t, y) = \frac{w_y}{w_y + \sum_{y' \neq y} w_{y'} \exp \left( e^{-(T - t)} \langle x_t, \mu_{y'} - \mu_y\rangle - e^{-(T - t)} (\|\mu_{y'}\|_2^2 - \|\mu_y\|_2^2) / 2 \right)}.  
\end{align}
If $\exp(\langle x_t, \mu_y \rangle - \|\mu_y\|_2^2 / 2) = \max_{y' \in \cY} \exp(\langle x_t, \mu_{y'} \rangle - \|\mu_{y'}\|_2^2 / 2)$, then one can verify that
\begin{align*}
	\tilde q_{T - t}(x_t, y) = \frac{w_y \exp \left( e^{-(T - t)} \langle x_t, \mu_y \rangle - e^{-2(T - t)} \|\mu_y\|_2^2 / 2 \right)}{\sum_{y' \in \cY} w_{y'} \exp \left( e^{-(T - t)} \langle x_t, \mu_{y'} \rangle - e^{-2(T - t)} \|\mu_{y'}\|_2^2 / 2 \right)} \leq \cP(x_t, y)
\end{align*}
Plugging this upper bound back into Eq.~\eqref{eq:upper-bound-tilde-q}, we get
%$\cP(x_t, y) = \max_{y' \in \cY} \cP(x_t, y')$, then $\tilde{q}_{T - t}(x_t, y) \leq \cP(x_t, y)$, hence by Eq.~\eqref{eq:P-xt-y}
%
\begin{align}
\label{eq:qT-t-upper1}
\begin{split}
    q_{T - t}(x_t, y) \leq & \frac{\cP(x_t, y)}{\cP(x_t, y) + (1 - \cP(x_t, y)) \cdot \exp(-\Delta / 8)}. % \\
    %\leq & \frac{g(\cP(x_0, y))}{g(\cP(x_0, y)) + (1 - g(\cP(x_0, y))) \cdot \exp(-\Delta / 8)}. 
\end{split}
\end{align}
On the other hand, if $\exp(\langle x_t, \mu_y\rangle - \|\mu_y\|_2^2 / 2) \neq \max_{y' \in \cY} \exp(\langle x_t, \mu_{y'}\rangle - \|\mu_{y'}\|_2^2 / 2)$, then one can verify that $\tilde{q}_{T - t}(x_t, y) \leq w_y / (w_y + \min_{y' \neq y} w_{y'})$, which together with Eq.~\eqref{eq:upper-bound-tilde-q} further implies that 
%if $\cP(x_t, y) \neq \max_{y' \in \cY} \cP(x_t, y')$, then $\tilde{q}_{T - t}(x_t, y) \leq 1 / 2$, and 
%
\begin{align}
\label{eq:qT-t-upper2}
    q_{T - t}(x_t, y) \leq \frac{w_y}{w_y + \min_{y' \neq y} w_{y'} \exp (-\Delta / 8)}. 
\end{align}
Putting together Eq.~\eqref{eq:qT-t-upper1} and \eqref{eq:qT-t-upper2}, we conclude that 
\begin{align}
\label{eq:q-T-t-max-G-G}
	q_{T - t}(x_t, y) \leq \max\left\{ G(\cP(x_t, y)),\, G(w_y/(w_y + \min_{y' \neq y} w_{y'})) \right\}, 
\end{align}
%
%obtain $q_{T - t}(x_t, y) \leq \max \{ G(g(\cP(x_0, y))), G(w_y/(w_y + \min_{y' \neq y} w_{y'}))\}$, 
where $G(x) := x / (x + (1 - x) \cdot \exp(-\Delta / 8))$ is a function that maps $[0, 1]$ to $[0, 1]$. 
Taking the derivative of $G$, we see that for all $x \in [0, 1]$, 
\begin{align}
\label{eq:Gp}
	G'(x) = \frac{\exp(-\Delta / 8)}{ [x + (1 - x) \cdot \exp(-\Delta / 8)]^2} \in \big[\exp(-\Delta / 8), \exp(\Delta / 8) \big]. 
\end{align}
Let $\xi_w := 1 - w_y / (w_y + \min_{y' \neq y} w_{y'} ) > 0$.
Note that $G(1) = 1$, hence by Eq.~\eqref{eq:Gp} we obtain that $1 - G(\cP(x_t, y)) \geq \exp(-\Delta / 8) \cdot (1 - \cP(x_t, y))$ and $1 - G(1 - \xi_w) \geq \exp(-\Delta / 8) \cdot \xi_w$.
Substituting these bounds as well as the upper bound of Eq.~\eqref{eq:q-T-t-max-G-G} into Eq.~\eqref{eq:diff-x-mu-mu}, we are able to derive a lower bound for $\diff \langle x_t, \mu_y - \mu_{y'} \rangle / \diff t$: 
\begin{align}
\label{eq:last-step-thm-2.5}
\begin{split}
    & \frac{\diff}{\diff t} \langle x_t, \mu_y - \mu_{y'} \rangle \\
    \geq & e^{-(T - t)} \Big( \|\mu_y\|_2^2 - \langle \mu_{y}, \mu_{y'} \rangle + \gs e^{-\Delta / 8}  (\|\mu_y  - \mu_0\|_2^2 - 3 \varepsilon) \cdot \min \{ 1 - \cP(x_t, y),\, \xi_w \} \Big).
    %\geq & e^{-(T - t)} \cdot \left( \|\mu_y\|_2^2 - \langle \mu_{y}, \mu_{y'} \rangle + \gs (1 - \max \{ G(g(\cP(x_0, y))), G(w_y / (w_y + \min_{y' \neq y} w_{y'}))\})  (\|\mu_y  - \mu_0\|_2^2 - 3 \varepsilon) \right).
\end{split}
\end{align}
Equivalently, we can write Eq.~\eqref{eq:last-step-thm-2.5} as 
\begin{align*}
	\frac{\diff}{\diff t} \langle x_t - z_t, \mu_y - \mu_{y'} \rangle \geq e^{-(T - t)} \gs e^{-\Delta / 8}  (\|\mu_y  - \mu_0\|_2^2 - 3 \varepsilon) \cdot \min \{ 1 - \cP(x_t, y),\, \xi_w \}. 
\end{align*}
Now suppose $\langle x_t, \mu_y - \mu_{y'}  \rangle - \langle z_t, \mu_y - \mu_{y'} \rangle \in [0, \cU]$ for all $t \in [0, T]$ and $y' \in \cY$. Using this bound, we get 
\begin{align}
\label{eq:p-F-max}
\begin{split}
	\cP(x_t, y) \leq & \frac{\cP(z_t, y)}{\cP(z_t, y) + (1 - \cP(z_t, y)) \cdot \exp(-\cU)} \\
	 \leq & \frac{\max_{0 \leq t \leq T} \cP(z_t, y)}{\max_{0 \leq t \leq T} \cP(z_t, y) + (1 - \max_{0 \leq t \leq T} \cP(z_t, y)) \cdot \exp(-\cU)} \\
	  = & 1 - \cF\big( \max_{0 \leq t \leq T} \cP(z_t, y), \cU \big),
\end{split}  
\end{align}
where we let $\cF(p, u) = (1 - p) e^{-u} / (p + (1 - p)e^{-u})$. 
Therefore, in order for such a $\cU \in \RR_{\geq 0}$ to serve as a valid upper bound, it is necessary to have
\begin{align}
\label{eq:cU-condition}
	\cU \geq \langle x_0 - z_0, \mu_y - \mu_{y'} \rangle + (1 - e^{-T}) \cdot \gs e^{-\Delta / 8} (\|\mu_y - \mu_0\|_2^2 - 3 \varepsilon) \cdot \min \Big\{\cF\big( \max_{0 \leq t \leq T} \cP(z_t, y), \cU \big),\, \xi_w \Big\}.
\end{align}
For any $\cU \in \RR_{\geq 0}$ that does not satisfy Eq.~\eqref{eq:cU-condition}, we know that $\langle x_T - z_T, \mu_y - \mu_{y'} \rangle \geq \cU$, hence 
\begin{align*}
	\cP(x_T, y) \geq \frac{\cP(z_T, y)}{\cP(z_T, y) + (1 - \cP(z_T, y) )\cdot \exp(-\cU)}. 
\end{align*}
This completes the proof of the first result of the theorem. 

\subsubsection*{Proof of the second result}

We separately discuss two cases, depending on whether $\langle \mu_y, \mu_y - \mu_{y'}\rangle$ is non-negative for all $y' \in \cY$. 

If $\langle \mu_y, \mu_y - \mu_{y'}\rangle \geq 0$ for all $y' \in \cY$, then by Eq.~\eqref{eq:diff-x-mu-mu} we know that $t \mapsto \langle x_t, \mu_y - \mu_{y'} \rangle$ as a function of $t$ is non-decreasing, which further implies that $\cP(x_{t_1}, y) \geq \cP(x_{t_2}, y)$ for all $T \geq t_1 \geq t_2 \geq 0$. We denote by $p(\eta)$ an upper bound for $\cP(x_T, y)$ that implicitly depends on $x_0$. Following the analysis we have established to prove the first point (in particular, Eq.~\eqref{eq:lower-bound-Gpe}), we have 
\begin{align*}
	q_{T - t}(x_t, y) \leq \max \left\{ G(p(\eta)), G(1 - \xi_w) \right\}. 
\end{align*}
Putting together the above upper bound and Eq.~\eqref{eq:diff-x-mu-mu}, we obtain that 
\begin{align}
\label{eq:lower-bound-Gpe}
\begin{split}
	& \langle x_T, \mu_y - \mu_{y'} \rangle - \langle x_0, \mu_y - \mu_{y'} \rangle \\
	\geq & (1 - e^{-T}) \cdot \left( \langle \mu_y, \mu_y - \mu_{y'} \rangle + \eta \min \{1 - G(p(\eta)), 1 - G(1 - \xi_w)\}(\|\mu_y - \mu_0\|_2^2 - 3 \varepsilon) \right),
\end{split} 
\end{align}
where we recall that $\xi_w = 1 - w_y / (w_y + \min_{y' \neq y} w_{y'} )$ and  $G(x) = x / (x + (1 - x) \cdot \exp(-\Delta / 8))$. 
%To get the second lower bound in Eq.~\eqref{eq:lower-bound-Gpe}, we use the following result that we have established while proving the first point of the theorem: 
%
%\begin{align*}
%	q_{T - t}(x_t, y) \leq \max \left\{ G(\cP(x_t, y)), 1 - \xi_w \right\}. 
%\end{align*}
%
%Taking the derivative of $G$, we see that
%
%\begin{align}
%\label{eq:Gp}
%	G'(x) = \frac{\exp(-\Delta / 8)}{ [x + (1 - x) \cdot \exp(-\Delta / 8)]^2} \in \big[\exp(-\Delta / 8), \exp(\Delta / 8) \big]
%\end{align}
%
%for all $x \in [0, 1]$. 
Note that $G(1) = 1$, hence by Eq.~\eqref{eq:Gp} we obtain that $1 - G(p(\eta)) \geq \exp(-\Delta / 8) \cdot (1 - p(\eta))$ and $1 - G(1 - \xi_w) \geq \exp(-\Delta / 8) \xi_w$. Plugging this lower bound into Eq.~\eqref{eq:lower-bound-Gpe}, we get
\begin{align*}
	\cP(x_T, y) \geq \frac{\cP(x_0, y)}{\cP(x_0, y) + (1 - \cP(x_0, y)) \cdot \exp\left( - C_0 - \gs C_1 \min\{ 1 - p(\eta), \xi_w\} \right)}, 
\end{align*}
where $C_0 = \min_{y' \in \cY} (1 - e^{-T}) \langle \mu_y, \mu_y - \mu_{y'} \rangle$ and $C_1 = e^{-\Delta / 8} (\|\mu_y - \mu_0\|_2^2 - 3 \varepsilon)$. 
%and $C_2 = \xi_w(\|\mu_y - \mu_0\|_2^2 - 3 \varepsilon)$. 
By definition we know $p(\eta) \geq \cP(x_T, y)$, hence 
\begin{align}
\label{eq:cPx0ypeta}
	\cP(x_0, y) (1 - p(\eta)) \leq (1 - \cP(x_0, y)) \cdot \exp \left( -C_0 - \eta C_1 \min \{1 - p(\eta), \xi_w\} \right). 
\end{align}
When $\eta > 1$, we can write 
\begin{align}
\label{eq:1-p-eta}
	1 - p(\eta) = \frac{-C_0 - \mbox{logit}(\cP(x_0, y)) + \delta_{\eta} \log \eta}{\eta C_1}, 
\end{align}
where $\mbox{logit}(p) = \log (p / (1 - p))$, and $\delta_{\eta} > 0$.
Plugging Eq.~\eqref{eq:1-p-eta} into Eq.~\eqref{eq:cPx0ypeta}, we see that at least one of the following two inequalities hold:
\begin{align*}
	& \frac{-C_0 - \mbox{logit}(\cP(x_0, y)) + \delta_{\eta} \log \eta}{\eta C_1} \leq \frac{1}{\eta^{\delta_{\eta}}}, \\
	& \frac{-C_0 - \mbox{logit}(\cP(x_0, y)) + \delta_{\eta} \log \eta}{\eta C_1} \leq \frac{1 - \cP(x_0, y)}{\cP(x_0, y)} \cdot \exp(-C_0 - \eta C_1 \xi_w). 
\end{align*}
Inspecting the above formulas, we see that 
for a sufficiently large $\eta$, it holds that $\delta_{\eta} < 1$, which implies that 
\begin{align}
\label{eq:first-lower-bound-p-eta}
	p(\eta) \geq 1 - \frac{-C_0 - \mbox{logit} (\cP(x_0, y)) + \log \eta}{\eta C_1}. 
\end{align}
On the other hand, if not all $\langle \mu_y, \mu_y - \mu_{y'} \rangle$ are non-negative, then we denote the smallest one by $-V_s = \langle \mu_y, \mu_y - \mu_{y_s} \rangle < 0$ for some $y_s \in \cY$. 
We shall choose $\eta$ that is large enough such that $V_s < \gs e^{-\Delta / 8} \xi_w (\|\mu_y - \mu_0\|_2^2 - 3 \varepsilon)$. In this case, for all $y' \in \cY$, it holds that
\begin{align*}
	\frac{\diff}{\diff t} \langle x_t, \mu_y - \mu_{y'} \rangle \geq &  e^{-T + t}(-V_s + \eta \min \{1 - G(\cP(x_t, y)), 1 - G(1 - \xi_w)\} (\|\mu_y - \mu_0\|_2^2 - 3 \varepsilon)) \\
	\geq & e^{-T + t}(-V_s + \eta e^{-\Delta / 8} \min \{ 1 - \cP(x_t, y), \, \xi_w \}) (\|\mu_y - \mu_0\|_2^2 - 3 \varepsilon).
\end{align*}
Therefore, if $1 - \cP(x_t, y) \geq V_s e^{\Delta / 8} \eta^{-1} (\|\mu_y - \mu_0\|_2^2 - 3 \varepsilon)^{-1}$ for all $t \in [0, T]$, then $\langle x_t, \mu_y - \mu_{y'} \rangle$ as a function of $t$ is non-decreasing on $[0, T]$, hence $\cP(x_t, y)$ as a function of $t$ is also non-decreasing on $[0, T]$. 
Following exactly the same route before, we are able to derive the lower bound as stated in Eq.~\eqref{eq:first-lower-bound-p-eta}.
On the other hand, if $1 - \cP(x_t, y) < V_s e^{\Delta / 8} \eta^{-1} (\|\mu_y - \mu_0\|_2^2 - 3 \varepsilon)^{-1}$ at some time point $t = t_{\ast}$, then for all $t \in [t_{\ast}, T]$, it is not hard to see that 
\begin{align*}
	1 - \cP(x_t, y) \leq \frac{V_s e^{\Delta / 8}}{\eta (\|\mu_y - \mu_0\|_2^2 - 3 \varepsilon)}. 
\end{align*} 
Putting together the above results, we conclude that for a sufficiently large $\eta$ we always have Eq.~\eqref{eq:first-lower-bound-p-eta}. The proof is complete. 

\subsection{Proof of Theorem \ref{thm:DDPM-confidence}}
\label{sec:proof-thm:DDPM-confidence}

We initiate the proof by presenting an SDE comparison theorem.  
Lemma \ref{lemma:SDE-comparison-theorem} is adapted from Theorem 3.1 of \cite{zhu2010comparison}. 

\begin{lemma}[SDE comparison theorem]
\label{lemma:SDE-comparison-theorem}
	Consider the following two $m$-dimensional SDEs defined on $[0, T]$:
	\begin{align*}
		& X_t^1 = x^1 + \int_0^t b_1 (s, X_s^1) \diff s + \int_0^t \sigma_1(s, X_s^1) \diff W_s, \\
		& X_t^2 = x^2 + \int_0^t b_2(s, X_s^2) \diff s + \int_0^t \sigma_2(s, X_s^2) \diff W_s. 
	\end{align*}
	We assume the following conditions:
	\begin{enumerate}
		\item $b(t, x), \sigma(t, x)$ are continuous in $(t, x)$, 
		\item There exists a sufficiently large constant $\mu > 0$, such that for all $x, x' \in \RR^m$ and $t \in [0, T]$, it holds that 
		\begin{align*}
			& \|b(t, x) - b(t, x')\|_2 + \|\sigma(t, x) - \sigma(t, x')\|_2 \leq \mu \|x - x'\|_2, \\
			& \|b(t, x)\|_2 + \|\sigma(t, x)\|_2 \leq \mu(1 + \|x\|_2). 
		\end{align*}
	\end{enumerate}
	Then the following are equivalent:
	\begin{enumerate}
		\item[(i)] For any $t \in [0, T]$ and $x^1, x^2 \in \RR^m$ such that $x^1 \geq x^2$, almost surely we have $X_t^1 \geq X_t^2$ for all $t \in [0, T]$. 
		\item[(ii)] $\sigma^1 \equiv \sigma^2$, and for any $t \in [0, T]$, $k = 1, 2, \cdots, m$, 
		\begin{align*}
			\left\{ \begin{array}{l}
				(a) \,\,\,\,\,\sigma_k^1 \mbox{ depends only on } x_k, \\
				(b) \,\,\,\,\,\mbox{for all }x', \delta^k x \in \RR^m, \mbox{ such that }\delta^k x \geq 0, (\delta^k x)_k = 0, \\
				\,\,\,\,\,\,\,\,\,\,\,\,\,\,\,\,\,\,\,\,b_k^1(t, \delta^k x + x') \geq b_k^2(t, x'). 
			\end{array} \right.
		\end{align*}
	\end{enumerate}
\end{lemma}

We then prove the theorem. To this end, we establish the subsequent lemma. Note that Theorem \ref{thm:DDPM-confidence} follows straightforwardly from Lemma \ref{lemma:SDE-comparison-confidence}. 
\begin{lemma}
\label{lemma:SDE-comparison-confidence}
    We assume the conditions of Theorem \ref{thm:DDPM-confidence}. Then for all $y' \in \cY$, almost surely we have $\langle x_t, \mu_y - \mu_{y'} \rangle \geq \langle z_t, \mu_y - \mu_{y'}\rangle$ for all $t \in [0, T]$. 
\end{lemma}
\begin{proof}[Proof of Lemma \ref{lemma:SDE-comparison-confidence}]

Note that 
\begin{align}
\label{eq:SDE-long}
\begin{split}
   &  \diff \langle \bar{x}_t, \mu_y - \mu_{y'} \rangle \\
   = & \left[ - \langle \bar x_t, \mu_y - \mu_{y'} \rangle + 2 e^{-(T - t)} (1 + \gs - \gs q_{T - t}(\bar x_t, y))\|\mu_y\|_2^2 - 2\gs e^{-(T - t)} \sum_{y'' \neq y} q_{T - t}(\bar x_t, y'') \langle \mu_y, \mu_{y''}\rangle \right. \\
   & \left. - 2 e^{-(T - t)} (1 + \gs - \gs q_{T - t}(\bar x_t, {y})) \langle \mu_y, \mu_{y'}  \rangle + 2 \gs e^{-(T - t)} \sum_{y'' \neq y} q_{T - t}(\bar x_t, y'') \langle \mu_{y'}, \mu_{y''} \rangle \right] \diff t \\
   & + \sqrt{2} \langle \diff B_t, \mu_y - \mu_{y'} \rangle \\
   = & \left[ - \langle \bar x_t, \mu_y - \mu_{y'} \rangle + 2e^{-(T - t)} \|\mu_y\|_2^2 - 2e^{-(T - t)} \langle \mu_y, \mu_{y'} \rangle + 2\gs e^{-(T - t)}(1 - q_{T - t}(\bar x_t, y))\|\mu_y - \mu_0\|_2^2 \right. \\
   & \left. + 2\gs e^{-(T - t)} q_{T - t}(\bar x_t, y') \|\mu_{y'} - \mu_0\|_2^2  + \bar\cE_t \right] + \sqrt{2} \langle \diff B_t, \mu_y - \mu_{y'} \rangle \\
   \geq & \left[ - \langle \bar {x}_t, \mu_y - \mu_{y'}\rangle + 2e^{-(T - t)} \|\mu_y\|_2^2 - 2 e^{-(T - t)} \langle \mu_y, \mu_{y'} \rangle + 2\gs e^{-(T - t)} (1 - q_{T - t}(\bar x_t, y)) (\|\mu_y - \mu_0\|_2^2 - 3 \varepsilon)  \right] \diff t \\
   & + \sqrt{2} \langle \diff B_t, \mu_y - \mu_{y'} \rangle, 
\end{split}
\end{align}
where $\bar\cE_t$ is a function of $(\bar x_t, t)$, and $|\bar \cE_t| \leq 6\gs e^{-(T - t)} (1 - q_{T - t}(\bar x_t, y)  ) \varepsilon$.
Note that the unguided process $(\bar z_t)_{0 \leq t \leq T}$ satisfies the following SDE
\begin{align}
\label{eq:zt-mu-SDE}
    \diff \langle \bar z_t, \mu_y - \mu_{y'} \rangle = \left[ - \langle \bar z_t, \mu_y - \mu_{y'} \rangle + 2e^{-(T - t)} \|\mu_y\|_2^2 - 2 e^{-(T - t)} \langle \mu_y, \mu_{y'} \rangle \right] \diff t + \sqrt{2} \langle \diff B_t, \mu_y - \mu_{y'} \rangle. 
\end{align}
Lemma \ref{lemma:SDE-comparison-confidence} then follows as a straightforward consequence of Lemma \ref{lemma:SDE-comparison-theorem}. 

\end{proof}

\subsection{Proof of Theorem \ref{thm:DDPM-confidence-quan}}
\label{sec:proof-thm:DDPM-confidence-quan}

Plugging Eq.~\eqref{eq:q-T-t-max-G-G} and \eqref{eq:Gp} into the last line of Eq.~\eqref{eq:SDE-long}, we obtain 
\begin{align}
\label{eq:C-A17}
	& \diff \langle \bar x_t, \mu_y - \mu_{y'} \rangle \nonumber \\
	\geq & \left[ - \langle \bar {x}_t, \mu_y - \mu_{y'}\rangle + 2e^{-(T - t)} \Big( \|\mu_y\|_2^2 -  \langle \mu_y, \mu_{y'} \rangle + \gs e^{ - \Delta / 8} \min \{1 - \cP(x_t, y), \xi_w\} (\|\mu_y - \mu_0\|_2^2 - 3 \varepsilon) \Big)  \right] \diff t  \nonumber \\
   & + \sqrt{2} \langle \diff B_t, \mu_y - \mu_{y'} \rangle. 
\end{align}
Invoking the method of integrating factors, we see that
\begin{align*}
	& \diff \left[ e^t \langle \bar x_t, \mu_y - \mu_{y'} \rangle \right] \\
	\geq & 2e^{-T + 2t} \Big( \|\mu_y\|_2^2 -  \langle \mu_y, \mu_{y'} \rangle + \gs e^{ - \Delta / 8} \min \{1 - \cP(x_t, y), \xi_w\} (\|\mu_y - \mu_0\|_2^2 - 3 \varepsilon) \Big)\diff t + \sqrt{2} e^t\langle \diff B_t, \mu_y - \mu_{y'} \rangle. 
\end{align*}
Note that 
\begin{align*}
	\diff \left[ e^t \langle \bar z_t, \mu_y - \mu_{y'} \rangle \right] = 2e^{-T + 2t} \Big( \|\mu_y\|_2^2 -  \langle \mu_y, \mu_{y'} \rangle \Big) \diff t + \sqrt{2} e^t\langle \diff B_t, \mu_y - \mu_{y'} \rangle. 
\end{align*}
Combining the above two equations, we obtain that 
\begin{align}
\label{eq:A16-et}
	\diff \left[ e^t \langle \bar x_t - \bar z_t, \mu_y - \mu_{y'} \rangle \right] \geq 2 \gs e^{-T + 2t - \Delta / 8} \min \{1 - \cP(x_t, y), \xi_w\} (\|\mu_y - \mu_0\|_2^2 - 3 \varepsilon) \diff t.  
\end{align}
By assumption $\langle \bar x_0 - \bar z_0, \mu_y - \mu_{y'} \rangle \geq 0$, hence $\langle \bar x_t - \bar z_t, \mu_y - \mu_{y'} \rangle \geq 0$ for all $t \in [0, T]$. 
Suppose we have $e^t \langle \bar x_t, \mu_y - \mu_{y'}  \rangle - e^t  \langle \bar z_t, \mu_y - \mu_{y'} \rangle \in [0, \cU]$ for all $t \in [0, T]$ and $y' \in \cY$. Then from Eq.~\eqref{eq:p-F-max} we know that for all $t \in [0, T]$, 
\begin{align}
\label{eq:A18-cP}
	\cP(\bar x_t, y) \leq 1 - \cF\big( \max_{0 \leq t \leq T} \cP(\bar z_t, y),\, \cU \big). 
\end{align}
Using Eq.~\eqref{eq:A16-et} and \eqref{eq:A18-cP}, we see that in order for $\cU$ to be a valid upper bound, we must have 
\begin{align}
\label{eq:A19-U}
	\cU \geq \langle \bar x_0 - \bar z_0, \mu_y - \mu_{y'} \rangle + \gs (e^T - e^{-T}) e^{-\Delta / 8} \min \big\{\cF\big( \max_{0 \leq t \leq T} \cP(\bar z_t, y),\, \cU \big),\, \xi_w\big\} (\|\mu_y - \mu_0\|_2^2 - 3 \varepsilon). 
\end{align}
For any $\cU$ that does not satisfy Eq.~\eqref{eq:A19-U}, we know that there exists $t \in [0, T]$, such that $e^t \langle \bar x_t - \bar z_t, \mu_y - \mu_{y'} \rangle \geq \cU$, hence $\langle \bar x_T - \bar z_T, \mu_y - \mu_{y'} \rangle \geq e^{-T} \cU$. As a consequence, we have
\begin{align}
\label{eq:A20-cP}
	\cP(\bar x_T, y) \geq \frac{\cP(\bar z_T, y)}{\cP(\bar z_T, y) + (1 - \cP(\bar z_T, y)) \cdot \exp(-e^{-T} \cU)}. 
\end{align}
Setting $\bar \cU = e^{-T} \cU$ completes the proof of the first result. 

As for the proof of the convergence rate, note that if we set $e^{-\cU} = \gs^{-1} (\log \gs)^2$, then as $\gs \to \infty$,  the left hand side of Eq.~\eqref{eq:A19-U} is of order $O(\log \gs)$, while the right hand side of Eq.~\eqref{eq:A19-U} is of order $O(\gs \wedge (\log \gs)^2)$. Hence, for a large enough $\gs$ Eq.~\eqref{eq:A19-U} is not satisfied. 
Plugging such $\cU$ into Eq.~\eqref{eq:A20-cP}, we conclude that $\cP(\bar x_T, y) \geq 1 - O(\gs^{-e^{-T}} (\log \gs)^{2e^{-T}})$.

\subsection{Proof of Theorem \ref{thm:DDIM-two-clusters}}
\label{sec:proof-thm:DDIM-two-clusters}

\subsubsection*{Proof of the first claim}

Inspecting Eq.~\eqref{eq:diff-x-mu-mu-equality},  \eqref{eq:diff-z-mu-mu} and setting $\mu_0 = (\mu_1 + \mu_2) / 2$, $\mu = \mu_1 - \mu_0$ therein, we have
\begin{align}
   & 2 \frac{\diff}{\diff t} \langle x_t, \mu \rangle = e^{-(T - t)} \|\mu_1\|_2^2 - e^{-(T - t)} \langle \mu_1, \mu_2 \rangle + 4\gs e^{-(T - t)} (1 - q_{T - t}(x_t, 1)) \|\mu\|_2^2, \label{eq:two-component-xt-mu} \\
   & 2 \frac{\diff}{\diff t} \langle z_t, \mu \rangle = e^{-(T - t)} \|\mu_1\|_2^2 - e^{-(T - t)} \langle \mu_1, \mu_2 \rangle. \label{eq:two-component-zt-mu}
\end{align}
Applying the ODE comparison theorem (Lemma \ref{lemma:comparison-theorem}), we conclude that $\langle x_t, \mu\rangle \geq \langle z_t, \mu \rangle$ for all $t \in [0, T]$. 
The first claim of the lemma then immediately follows, as in this case
\begin{align*}
    & \cP(x_t, 1) = \frac{w_1 \exp (\langle x_t, \mu \rangle - \|\mu_1\|_2^2 / 2)}{w_1 \exp (\langle x_t, \mu \rangle - \|\mu_1\|_2^2 / 2) + w_2 \exp (-\langle x_t, \mu \rangle - \|\mu_2\|_2^2 / 2)}, \\
    & \cP(z_t, 1) = \frac{w_1 \exp (\langle z_t, \mu \rangle - \|\mu_1\|_2^2 / 2)}{w_1 \exp (\langle z_t, \mu \rangle - \|\mu_1\|_2^2 / 2) + w_2 \exp (-\langle z_t, \mu \rangle - \|\mu_2\|_2^2 / 2)}.
\end{align*}

\subsubsection*{Proof of the second claim}

%We then prove the second claim. 
%By Eq.~\eqref{eq:two-component-xt-mu}, it holds that
%
%\begin{align*}
%    2\frac{\diff}{\diff t} \langle x_t, \mu \rangle \leq e^{-(T - t)} \left( \|\mu_1\|_2^2 - \langle \mu_1, \mu_2 \rangle + 4\gs\|\mu\|_2^2 \right), 
%\end{align*}
%
%hence $2 \langle x_t, \mu \rangle \leq 2 \langle x_0, \mu \rangle + (e^{-(T - t)} - e^{-T}) ( \|\mu_1\|_2^2 - \langle \mu_1, \mu_2 \rangle + 4\gs\|\mu\|_2^2)$. As a consequence, 
%
%\begin{align*}
%    \cP(x_t, 1) = & \frac{w_1}{w_1 + w_2 \exp \left( -2 \langle x_t, \mu \rangle - (\|\mu_2\|_2^2 - \|\mu_1\|_2^2) / 2 \right)} \\
%    \leq & \frac{w_1}{w_1 + w_2 \exp \left( -2 \langle x_0, \mu \rangle - (\|\mu_2\|_2^2 - \|\mu_1\|_2^2) / 2 - (e^{-(T - t)} - e^{-T})U_1 \right)} \\
%    = & g_1(\cP(x_0, 1)) :=  \frac{\cP(x_0, 1)}{\cP(x_0, 1) + (1 - \cP(x_0, 1)) \cdot \exp \left( -(1 - e^{-T}) U_1\right)}
%\end{align*}
%
%where $U_1 = \|\mu_1\|_2^2 - \langle \mu_1, \mu_2\rangle + 4\gs \|\mu\|_2^2$, and $g_1(x) = x / (x + (1 - x) \cdot \exp(-(1 - e^{-T}) U_1))$. 
Similar to the derivation of Eq.~\eqref{eq:upper-bound-tilde-q}, we conclude that 
\begin{align}
\label{eq:two-cluster-q-bound1}
    q_{T - t}(x_t, 1) \leq \frac{\tilde q_{T - t}(x_t, 1)}{\tilde q_{T - t}(x_t, 1) + (1 - \tilde q_{T - t}(x_t, 1)) \cdot \exp (-\Delta_1 / 8)}, 
\end{align}
where we recall that $\tilde q_{T - t}$ is defined in Eq.~\eqref{eq:def-tilde-q}, and $\Delta_1 = |\|\mu_1\|_2^2 - \|\mu_2\|_2^2|$. 
If $\exp (\langle x_t, \mu_1 \rangle - \|\mu_1\|_2^2 / 2) \geq \exp(\langle x_t, \mu_2\rangle - \|\mu_2\|_2^2 / 2)$, then 
\begin{align*}
	\tilde q_{T - t}(x_t, 1) \leq \cP(x_t, 1). 
\end{align*}
Plugging the above inequality into Eq.~\eqref{eq:two-cluster-q-bound1}, we obtain that  
%$\tilde q_{T - t}(x_t, 1) \leq \cP(x_t, 1)$, and by Eq.~\eqref{eq:two-cluster-q-bound1}
%
\begin{align}
\label{eq:two-cluster-q-bound2}
\begin{split}
    q_{T - t}(x_t, 1) \leq & \frac{\cP(x_t, 1)}{\cP(x_t, 1) + (1 - \cP(x_t, 1)) \cdot \exp(-\Delta_1 / 8)}.  %\\
    %\leq & \frac{g_1(\cP(x_0, 1))}{g_1(\cP(x_0, 1)) + (1 - g_1(\cP(x_0, 1))) \cdot \exp (-\Delta_1 / 8)}. 
\end{split}
\end{align}
On the other hand, if $\exp(\langle x_t, \mu_1\rangle - \|\mu_1\|_2^2 / 2) \leq \exp(\langle x_t, \mu_2\rangle - \|\mu_2\|_2^2 / 2)$, then $\tilde q_{T - t}(x_t, 1) \leq w_1$. 
Putting together this upper bound, Eq.~\eqref{eq:two-cluster-q-bound1} and \eqref{eq:two-cluster-q-bound2}, we conclude that 
$$q_{T - t}(x_t, 1) \leq \max \{G_1(\cP(x_t, 1)), G_1(w_1)\},$$ 
%$$q_{T - t}(x_t, 1) \leq \max \{G_1(g_1(\cP(x_0, 1))), G_1(w_1)\}$$, 
%
where $G_1(x) := x / (x + (1 - x) \cdot \exp(-\Delta_1 / 8))$ maps $[0, 1]$ to $[0, 1]$. Taking the derivative of $G_1$, we see that for all $x \in [0, 1]$,
\begin{align}
\label{eq:G1-diff}
	G_1'(x) = \frac{\exp(-\Delta_1 / 8)}{ [x + (1 - x) \cdot \exp(-\Delta_1 / 8)]^2} \in \big[\exp(-\Delta_1 / 8), \exp(\Delta_1 / 8) \big]. 
\end{align}
Observe that $G_1(1) = 1$, then by Eq.~\eqref{eq:G1-diff} we have 
\begin{align*}
	1 - G_1(\cP(x_t, 1)) \geq e^{-\Delta_1 / 8} (1 - \cP(x_t, 1)), \qquad 1 - G_1(w_1) \geq e^{-\Delta_1 / 8} (1 - w_1), 
\end{align*}
which further implies that 
\begin{align}
\label{eq:q-T-t-lower-min}
	1 - q_{T - t}(x_t, 1) \geq e^{-\Delta_1 / 8} \min \{1 - \cP(x_t, 1), 1 - w_1\}.  
\end{align}
Plugging the lower bound in Eq.~\eqref{eq:q-T-t-lower-min} into Eq.~\eqref{eq:two-component-xt-mu} and \eqref{eq:two-component-zt-mu}, we obtain 
\begin{align}
\label{eq:A27-new}
	\frac{\diff}{\diff t} \langle x_t - z_t, \mu\rangle \geq 2 \gs e^{-T + t - \Delta_1 / 8} \min \{1 - \cP(x_t, 1), 1 - w_1\} \|\mu\|_2^2. 
\end{align}
The above equation together with the assumption $\langle x_0 - z_0, \mu\rangle \geq 0$ implies that $\langle x_t - z_t, \mu \rangle \geq 0$ for all $t \in [0, T]$. Now suppose $2\langle x_t - z_t, \mu \rangle \in [0, \cU]$ for all $t \in [0, T]$. Similar to the derivation of Eq.~\eqref{eq:p-F-max}, we conclude that 
\begin{align}
\label{eq:A27-P}
	\cP(x_t, 1) \leq 1 - \cF\big( \max_{0 \leq t \leq T} \cP(z_t, 1),\, \cU \big). 
\end{align}
Plugging Eq.~\eqref{eq:A27-P} into Eq.~\eqref{eq:A27-new}, we see that in order for $\cU$ to be a valid upper bound, we must have 
\begin{align}
\label{eq:A29}
	\cU \geq 2\langle x_0 - z_0, \mu \rangle + 4 \gs e^{-\Delta_1 / 8} \|\mu\|_2^2 (1 - e^{-T}) \min \big\{\cF\big( \max_{0 \leq t \leq T} \cP(z_t, 1),\, \cU \big), \, 1 - w_1 \big\}.  
\end{align}
If $\cU$ does not satisfy Eq.~\eqref{eq:A29}, then $2\langle x_T - z_T, \mu \rangle \geq \cU$, hence 
\begin{align}
\label{eq:A30}
	\cP(x_T, 1) \geq \frac{\cP(z_T, 1)}{\cP(z_T, 1) + (1 - \cP(z_T, 1)) \cdot \exp(-\cU)}. 
\end{align}
The proof of the first result is complete. 
We then prove the result regarding the convergence rate as $\gs \to \infty$. To this end, we set $e^{-\cU} = \gs^{-1} (\log \gs)^2$. For such $\cU$, the left hand side of Eq.~\eqref{eq:A29} is of order $O(\log \gs)$, while the right hand side of Eq.~\eqref{eq:A29} is of order $O(\gs \wedge (\log \gs)^2)$. Therefore, for a sufficiently large $\gs$, Eq.~\eqref{eq:A29} does not hold. Plugging such $\cU$ into Eq.~\eqref{eq:A30}, we deduce that $\cP(x_T, 1) \geq 1 - O(\gs^{-1} (\log \gs)^2)$ as $\gs \to \infty$.  

%Plugging this upper bound into Eq.~\eqref{eq:two-component-xt-mu} and \eqref{eq:two-component-zt-mu}, we obtain that
%
%\begin{align*}
%    & 2 \langle x_t, \mu \rangle - 2 \langle z_t, \mu \rangle \geq  (e^{-(T - t)} - e^{-T}) M_1, \\
%    & M_1 = 4\gs (1 - \max \{G_1(g_1(\cP(x_0, y))), G_1(w_1)\}) \|\mu\|_2^2. 
%\end{align*}
%
%Substituting the above lower bound into the definition of $\cP(x_t, 1)$, we see that 
%
%\begin{align*}
%    \cP(x_t, 1) = & \frac{w_1}{w_1 + w_2 \exp \left( -2 \langle x_t, \mu \rangle - (\|\mu_2\|_2^2 - \|\mu_1\|_2^2) / 2 \right)} \\
%    \geq & \frac{w_1}{w_1 + w_2 \exp\left( - 2 \langle z_t, \mu \rangle - (\|\mu_2\|_2^2 - \|\mu_1\|_2^2) / 2 - (e^{-(T - t)} - e^{-T}) M_1  \right)} \\
%    = & \frac{\cP(z_t, 1)}{\cP(z_t, 1) + (1 - \cP(z_t, 1)) \cdot \left( - (e^{-(T - t)} - e^{-T}) M_1 \right)}. 
%\end{align*}
%
%This concludes the proof of the lemma. 

\subsection{Proof of Theorem \ref{thm:DDPM-two-clusters}}
\label{sec:proof-thm:DDPM-two-clusters}

\subsubsection*{Proof of the first claim}

Similar to the proof of Theorem \ref{thm:DDIM-two-clusters}, we set $\mu_0 = (\mu_1 + \mu_2) / 2$ and $\mu = \mu_1 - \mu_0$. Following the derivation of Eq.~\eqref{eq:SDE-long} and \eqref{eq:zt-mu-SDE}, we obtain
\begin{align}
\label{eq:SDE-two-component}
\begin{split}
     2 \diff \langle \bar x_t, \mu \rangle = & \left[- 2 \langle \bar x_t, \mu \rangle + 2e^{-(T - t)}\|\mu_1\|_2^2 - 2 e^{-(T - t)} \langle \mu_1, \mu_2 \rangle + 8\gs e^{-(T - t)} (1 - q_{T - t}(\bar x_t, 1)) \|\mu\|_2^2 \right] \diff t \\
     & + 2\sqrt{2} \langle \diff B_t, \mu \rangle, \\
     2 \diff \langle \bar z_t, \mu \rangle = & \left[ - 2 \langle \bar z_t, \mu \rangle + 2 e^{-(T - t)} \|\mu_1\|_2^2 - 2 e^{-(T - t)} \langle \mu_1, \mu_2 \rangle \right] \diff t + 2 \sqrt{2} \langle \diff B_t, \mu \rangle. 
\end{split}
\end{align}
Note that $q_{T - t}(\bar x_t, 1)$ depends on $\bar x_t$ only through $\langle \bar x_t, \mu \rangle$, hence both equations listed above represent an SDE. 
Then, we may leverage the SDE comparison theorem (Lemma \ref{lemma:SDE-comparison-theorem}) to deduce that almost surely, $\langle \bar x_t, \mu \rangle \geq \langle \bar z_t, \mu \rangle$ for all $t \in [0, T]$. This completes the proof of the first claim. 

\subsubsection*{Proof of the second claim}

%We then prove the second claim. 
%We define $f_1(t) = 2e^{-(T - t)}\|\mu_1\|_2^2 - 2 e^{-(T - t)} \langle \mu_1, \mu_2 \rangle + 8\gs e^{-(T - t)}  \|\mu\|_2^2$. 
%Multiplying both sides of Eq.~\eqref{eq:SDE-two-component} by $e^t$, we obtain
%
%\begin{align*}
%    2 \diff \left[ e^t \langle \bar x_t, \mu \rangle \right] \leq e^t f_1(t) \diff t + 2\sqrt{2} e^t \langle \diff B_t, \mu \rangle, 
%\end{align*}
%
%which further implies that
%
%\begin{align*}
%    2 \langle \bar x_t, \mu \rangle \leq 2e^{-t} \langle \bar x_0, \mu \rangle + e^{-t} \int_0^t e^s f_1(s) \diff s + 2e^{-t} \int_0^t \sqrt{2} e^s \langle \diff B_s, \mu \rangle.
%\end{align*}
%
%We define 
%
%\begin{align*}
%    \bar U_1 = (e^T - e^{-T}) (\|\mu_1\|_2^2 - \langle \mu_1, \mu_2 \rangle + 4\gs \|\mu\|_2^2) + \sup_{0 \leq t \leq T} \left| 
%    2e^{-t} \int_0^t \sqrt{2} \langle \diff B_s, \mu\rangle\right|, 
%\end{align*}
%
%then $2e^t \langle \bar x_t, \mu \rangle \leq 2 \langle \bar x_0, \mu \rangle + \bar U_1$. Using this upper bound, we conclude that 
%

Plugging Eq.~\eqref{eq:q-T-t-lower-min} into Eq.~\eqref{eq:SDE-two-component}, we see that 
\begin{align*}
	2 \diff \langle \bar x_t - \bar z_t, \mu \rangle \geq \left[ -2\langle \bar x_t - \bar z_t, \mu \rangle + 8 \gs \|\mu\|_2^2 e^{-T + t - \Delta_1 / 8} \min \{1 - \cP(\bar x_t, 1), w_2\}  \right] \diff t. 
\end{align*}
Multiplying both sides above by $e^t$, we get
\begin{align}
\label{eq:A32-new}
	\diff \left[ 2e^t \langle \bar x_t - \bar z_t, \mu \rangle \right] \geq  8 \gs \|\mu\|_2^2 e^{-T + 2t - \Delta_1 / 8} \min \{1 - \cP(\bar x_t, 1), w_2\} \diff t. 
\end{align}
Since by assumption $\langle \bar x_0 - \bar z_0, \mu \rangle \geq 0$, we then conclude that almost surely we have $\langle \bar x_t - \bar z_t, \mu\rangle \geq 0$ for all $t \in [0, T]$. If we assume $2e^t\langle \bar x_t - \bar z_t, \mu\rangle \in [0, \cU]$ for all $t \in [0, T]$, then it holds that $2 \langle \bar x_t - \bar z_t, \mu \rangle \leq \cU$ for all $t \in [0, T]$. 
Following the derivation of Eq.~\eqref{eq:A18-cP}, we have 
\begin{align}
\label{eq:A32}
	1 - \cP (\bar x_t, 1) \geq \cF \big( \max_{0 \leq t \leq T} \cP(\bar z_t, 1), \cU \big). 
\end{align}
Putting together Eq.~\eqref{eq:A32-new} and \eqref{eq:A32}, we see that for $\cU$ to serve as a valid upper bound, we must have 
\begin{align}
\label{eq:A34}
	\cU \geq 2\langle \bar x_0 - \bar z_0, \mu \rangle + 4\gs \|\mu\|_2^2 e^{-\Delta_1 / 8} (e^T - e^{-T}) \min \Big\{\cF \big( \max_{0 \leq t \leq T} \cP(\bar z_t, 1), \cU \big), w_2\Big\}. 
\end{align} 
If Eq.~\eqref{eq:A34} is not satisfied, then for such $\cU$ we have $2 \langle \bar x_T - \bar z_T, \mu \rangle \geq e^{-T} \cU$, and
\begin{align}
\label{eq:A35}
	\cP(\bar x_T, 1) \geq \frac{\cP(\bar z_T, 1)}{\cP(\bar z_T, 1) + (1 - \cP(\bar z_T, 1)) \cdot \exp(-e^{-T} \cU)}. 
\end{align} 
Setting $\bar \cU = e^{-T} \cU$ completes the proof of the first bound. 

To prove the convergence rate, we simply set $e^{-\cU} = \gs^{-1} (\log \gs)^2$. As $\gs \to \infty$, the left hand side of Eq.~\eqref{eq:A34} is of order $O(\log \gs)$, while the right hand side of Eq.~\eqref{eq:A34} is of order $O(\gs \wedge (\log \gs)^2)$. For a large enough $\gs$, we see that Eq.~\eqref{eq:A34} does not hold. 
Plugging such $\cU$ into Eq.~\eqref{eq:A35}, we conclude that $\cP(\bar x_T, 1) = 1 - O(\gs^{-e^{-T}}(\log \gs)^{2 e^{-T}})$.

\section{Proofs related to diversity reduction}

This section contains proofs related to diversity reduction.
We present in Appendix \ref{sec:FK-derivation} a heuristic derivation of Theorem \ref{thm:DDIM-diversity} based on the Fokker-Planck equation, and leave the establishment of a rigorous procedure to the remaining sections. 
In Appendix \ref{sec:proof-lemma:density-exist}, we demonstrate the existence of probability density functions $Q(\cdot, t)$ and $Q_0(\cdot, t)$ for all $t \in [0, T]$.  

\subsection{Derivation of Theorem \ref{thm:DDIM-diversity} via the Fokker-Planck equation}
\label{sec:FK-derivation}

We provide in this section a non-rigorous derivation of Theorem \ref{thm:DDIM-diversity} via the Fokker-Planck equation. 
This part serves as a motivation of our theorem, 
and a rigorous proof can be found in Appendix \ref{sec:proof-thm:DDIM-diversity} instead.

Leveraging the Fokker–Planck equation (Lemma \ref{lemma:Fokker-Planck}), on $[0, T]$ we have
\begin{align*}
	& \frac{\partial}{\partial t} Q(t, x) = - \sum_{i = 1}^d \frac{\partial}{\partial x_i} \left[ Q(t, x) \cdot \big(x_i + \frac{\partial}{\partial x_i} \log p_{T - t}(x, y) + \gs \frac{\partial}{\partial x_i} \log p_{T - t}(y \mid x)  \big) \right], \\
	& \frac{\partial}{\partial t} Q_0(t, x) = - \sum_{i = 1}^d \frac{\partial}{\partial x_i} \left[ Q_0(t, x) \cdot \big(x_i + \frac{\partial}{\partial x_i} \log p_{T - t}(x, y)  \big) \right]. 
\end{align*}
Therefore, 
\begin{align*}
	\frac{\partial}{\partial t} H(t) = & - \int\frac{\partial}{\partial t} Q(t, x) \log Q(t, x) \diff x - \int \frac{\partial}{\partial t} Q(t, x) \diff x \\
	\overset{(i)}{=} & \sum_{i = 1}^d \int  \frac{\partial}{\partial x_i} \left[ Q(t, x) \cdot \big( x_i + \frac{\partial}{\partial x_i} \log p_{T - t}(x, y) + \gs \frac{\partial}{\partial x_i} \log p_{T - t}(y \mid x) \big) \right] \log Q(t, x) \diff x \\
    = &  \sum_{i = 1}^d \int \log Q(t, x) \frac{\partial}{\partial x_i} Q(t, x) \cdot \left[ x_i + \frac{\partial}{\partial x_i} \log p_{T - t}(x, y) + \gs \frac{\partial}{\partial x_i} \log p_{T - t}(y \mid x) \right] \diff x  \\
    & + \sum_{i = 1}^d \int Q(t, x) \log Q(t, x) \cdot \left[ 1 + \frac{\partial^2}{\partial^2 x_i} \log p_{T - t}(x, y) + \gs \frac{\partial^2}{\partial^2 x_i} \log p_{T - t}(y \mid x) \right] \diff x \\
	\overset{(ii)}{=} & \sum_{i = 1}^d \int \left(1 + \frac{\partial^2}{\partial^2 x_i} \log p_{T - t}(x, y) + \gs \frac{\partial^2}{\partial^2 x_i} \log p_{T - t}(y \mid x) \right) Q(t, x) \diff x, 
\end{align*}
where $(i)$ is because $\int Q(t, x) \diff x \equiv 1$, and $(ii)$ is via integration by parts. 
Applying a similar procedure to the diffusion model without guidance, we obtain
\begin{align*}
	\frac{\partial }{\partial t} H_0(t) = \sum_{i = 1}^d \int \left(1 + \frac{\partial^2}{\partial^2 x_i} \log p_{T - t}(x, y) \right) Q_0(t, x) \diff x. 
\end{align*}
Note that 
\begin{align*}
	& \sum_{i = 1}^d \frac{\partial^2}{\partial^2 x_i} \log p_{T - t}(x, y) = -\tr\left[ \Sigma_{T - t}^{-1}  \right], \\
	& \sum_{i = 1}^d \frac{\partial^2}{\partial ^2 x_i} \log p_{T - t}(y \mid x) = - \tr\left[ \sum_{y' \in \cY} q_{T - t}(y' \mid x) \Sigma_{T - t}^{-1} \mu_{y'} \mu_{y'}^{\top} \Sigma_{T - t}^{-1}  -  vv^{\top}\right], \\
 & v = \sum_{y' \in \cY} q_{T - t}(y' \mid x) \Sigma_{T - t}^{-1} \mu_{y'}. 
\end{align*}
As a consequence, we have $\sum_{i = 1}^d \frac{\partial^2}{\partial^2 x_i} \log p_{T - t}(y \mid x) \leq 0$. Putting together this result and the ODE comparison theorem (Lemma \ref{lemma:comparison-theorem}), we obtain the desired result. However, we emphasize that the above derivation is non-rigorous. 
For example, it is unclear whether the Fokker-Planck equation has a solution, and also the exchange of integration and differentiation is unjustified. 

%The theorem then follows by an application of the ODE comparison theorem (Lemma \ref{lemma:comparison-theorem}). 

\subsection{Existence of probability density functions}
\label{sec:proof-lemma:density-exist}

In this section, we justify the existence of probability density functions. Namely, we establish the following lemma. 

\begin{lemma}
\label{lemma:density-exist}
We assume the conditions of Theorem \ref{thm:DDIM-diversity}. Then $Q(t, \cdot)$ and $Q_0(t, \cdot)$ exist for all $t \in [0, T]$. 
    
\end{lemma}

We prove Lemma \ref{lemma:density-exist} in the remainder of this section. We separately discuss the guided process and the unguided process below. 

\subsubsection*{Proof for $Q_0$}

We first show that $z_t$ has a probability density function for all $t \in [0, T]$. 
Observe that $(z_t)_{0 \leq t \leq T}$ is a solution to the following ODE: 
\begin{align}
\label{eq:zt-ODE}
    \frac{\diff z_t}{\diff t} = A(t) z_t + b(t), 
\end{align}
where the symmetric matrix $A(t) \in \RR^{d \times d}$ and the vector $b(t) \in \RR^d$ depend only on $t$. 
In addition, $A(t) A(s) = A(s) A(t)$ for all $t, s \in [0, T]$.
Solving Eq.~\eqref{eq:zt-ODE}, we conclude that 
\begin{align*}
    z_t = \exp \left( \int_0^t A(s) \diff s \right) z_0 + \int_0^t \exp \left( \int_s^t A(i) \diff i \right) b(s) \diff s.
\end{align*}
The matrix $\exp \left( \int_0^t A(s) \diff s \right)$ is non-degenerate.
By assumption, $z_0$ has a probability density function with respect to the Lebesgue measure.  
Therefore, $z_t$ also has a density. The proof is complete.

\subsubsection*{Proof for $Q$}

We then prove the lemma for the guided process $x_t$. 
Inspecting Eq.~\eqref{eq:dxt-long} and applying the triangle inequality, we obtain
\begin{align*}
    %& \frac{\diff \|x_t\|_2}{\diff t} \leq \|I_d - \Sigma_{T - t}^{-1}\|_{\op} \|x_t\|_2 + \|\Sigma_{T - t}^{-1}\|_{\op} \|\mu_y\|_2 + 2w \|\Sigma_{T - t}^{-1}\|_{\op} \sup_{y' \in \cY} \|\mu_{y'}\|_2, \\
     \frac{\diff \|x_t\|_2}{\diff t} \geq & -\|I_d - \Sigma_{T - t}^{-1}\|_{\op} \|x_t\|_2 - \|\Sigma_{T - t}^{-1}\|_{\op} \|\mu_y\|_2 - 2\gs \|\Sigma_{T - t}^{-1}\|_{\op} \sup_{y' \in \cY} \|\mu_{y'}\|_2 \\
    \geq & - \Big( 1 + [\sigma_{\min}(\Sigma) \wedge 1]^{-1} \Big) \|x_t\|_2  - [\sigma_{\min}(\Sigma) \wedge 1]^{-1} \|\mu_y\|_2 - 2\gs [\sigma_{\min}(\Sigma) \wedge 1]^{-1}  \sup_{y' \in \cY} \|\mu_{y'}\|_2,
\end{align*}
which by Lemma \ref{lemma:comparison-theorem} further implies that 
\begin{align}
\label{eq:CD}
    \|x_t\|_2 \geq e^{-C t} \|x_0\|_2 - \frac{D}{ C }\cdot (1 - e^{-C t} ), 
\end{align}
where $C = ( 1 + [\sigma_{\min}(\Sigma) \wedge 1]^{-1} ) $ and $D = [\sigma_{\min}(\Sigma) \wedge 1]^{-1} \|\mu_y\|_2 + 2\gs [\sigma_{\min}(\Sigma) \wedge 1]^{-1}  \sup_{y' \in \cY} \|\mu_{y'}\|_2$. Now we consider the set of initial values that lead to $x_t$: 
$$\cI_t(x_t) := \{x_0 \in \RR^d: \mbox{ ODE \eqref{eq:dxt-long} with initial value $x_0$ has value $x_t$ at time $t$}\}.$$ 
Examining Eq.~\eqref{eq:CD}, we conclude that $\cI_t(x_t) \subseteq \cB(f_t(\|x_t\|_2))$, where $\cB(r)$ stands for the ball in $\RR^d$ that has radius $r$ and is centered at the origin, and $f_t(r) = e^{Ct} r + C^{-1} D(e^{Ct} - 1)$. 

For the sake of simplicity, we rewrite Eq.~\eqref{eq:dxt-long} as $\diff x_t = F(x_t, t) \diff t$. For any $\delta > 0$, we consider the approximation to the ODE defined in Eq.~\eqref{eq:dxt-long} that has step size $\delta$: 
\begin{align*}
    \hat x_{k\delta}^{\delta} = \hat x_{(k - 1) \delta}^{\delta} + \delta F(\hat x_{(k - 1) \delta}^{\delta}, (k - 1) \delta), \qquad \hat x_{0}^{\delta} = x_0. 
\end{align*}
For $t \in [(k - 1) \delta, k\delta]$, we compute $\hat x_t^{\delta}$ by linearly interpolating $\hat x_{(k - 1) \delta}^{\delta}$ and $\hat x_{k\delta}^{\delta}$. 
To simplify analysis, we consider only $\delta$ that takes the form $T / K$ for $K \in \NN_+$. 
Taking the Jacobian matrix of $F(x, t)$ with respect to $x$, we get
\begin{align*}
    & \nabla_x F(x, t) = I_d - \Sigma_{T - t}^{-1} - \gs e^{-T + t} \Omega_{T - t}(x), \\
    & \Omega_{T - t}(x) = \sum_{y' \in \cY} q_{T - t}(x, y') \Sigma_{T - t}^{-1} \mu_{y'} \mu_{y'}^{\top} \Sigma_{T - t}^{-1} - \Big( \sum_{y' \in \cY} q_{T - t} (x, y') \Sigma_{T - t}^{-1} \mu_{y'} \Big) \Big( \sum_{y' \in \cY} q_{T - t} (x, y') \Sigma_{T - t}^{-1} \mu_{y'} \Big)^{\top}. 
\end{align*}
Therefore, we conclude that $F(\cdot, t)$ is Lipschitz continuous in its first argument, and the Lipschitz constant is uniformly bounded for all $t \in [0, T]$.
In addition, $F$ is continuous in its second argument.  
Leveraging a standard Gronwall type argument, we obtain that for all $r > 0$,  
\begin{align}
\label{eq:uniform-convergence-ODE}
    \limsup_{\delta \to 0^+}\sup_{x_0 \in \cB(r)} \Big\{ \sup_{0 \leq t \leq T} \| \hat x^{\delta}_t - x_t \|_2 \Big\} = 0.  %\leq C_0 \delta, 
\end{align}
%
%where $C_0$ is a constant that is independent of $\delta$ (but might depend on $T$, the Lipschitz constant of $F$, and other parameters of the problem). 
We can compute the Jacobian of the mapping $x_0 \mapsto \hat{x}_t^{\delta}$. 
To simplify presentation, here we let $t = k \delta$. We comment that the treatment for general $t \in [0, T]$ is similar, and we leave the homework to interested readers. We denote the Jacobian of this mapping $x_0 \mapsto \hat x_t^{\delta}$ by $J_{0 \to k \delta}^{\delta} (x_0) \in \RR^{d \times d}$. Observe that
\begin{align*}
    J_{0 \to k \delta}^{\delta} (x_0) = \prod_{i = 0}^{k - 1} \Big( I_d + \delta (I_d - \Sigma_{T - i\delta}^{-1} - \gs e^{-T + i\delta} \Omega_{T - i\delta} (\hat x_{i\delta}^{\delta})) \Big), 
\end{align*}
where $\prod_{i = 0}^{k - 1} A_i := A_{k - 1} A_{k - 2} \cdots A_0$. For a sufficiently small $\delta$ we see that $J_{0 \to k \delta}^{\delta} (x_0)$ is non-degenerate for all $k$ and $x_0$.
In addition, one can verify that for fixed $T, r > 0$ ($T = K \delta$), it holds that
\begin{align}
\label{eq:uniform-J}
\begin{split}
   & \lim_{\delta \to 0^+}\sup_{x_0 \in \cB(r), k \in \{0\} \cup [K]} \big\| J_{0 \to k \delta}^{\delta} (x_0) - J_{0 \to k \delta} (x_0) \big\|_{\op} = 0, \\
   & J_{0 \to k \delta} (x_0) = \exp \Big( \int_0^{k \delta} (I_d - \Sigma_{T - t}^{-1} - \gs e^{-T + t} \Omega_{T - t} (x_t)) \diff t \Big) \in \RR^{d \times d}. 
\end{split}
\end{align}
We write $x_t = G_t(x_0)$ and $\hat x_t^{\delta} = \hat G_t^{\delta}(x_0)$.
By Eq.~\eqref{eq:uniform-convergence-ODE} we have $\lim_{\delta \to 0^+}\sup_{x_0 \in \cB(r), 0 \leq t \leq T}\|G_t(x_0) - \hat G_t^{\delta}(x_0)\|_2 = 0$. 
Next, we prove that $\nabla_{x_0} G_t(x_0) = J_{0 \to t} (x_0)$. 
To this end, it suffice to show 
\begin{align*}
   \big\| G_t(x_0 + x') - G_t(x_0) - J_{0 \to t} (x_0) x' \big\|_2 = o(\|x'\|_2). 
\end{align*}
By triangle inequality, 
\begin{align}
\label{eq:GGJ}
\begin{split}
    & \big\| G_t(x_0 + x') - G_t(x_0) - J_{0 \to t} (x_0) x' \big\|_2 \\
    \leq & \|G_t(x_0 + x') - G_t(x_0) -\hat G_t^{\delta}(x_0 + x') + \hat G_t^{\delta} (x_0) \|_2 + \|\hat G_t^{\delta} (x_0 + x') - \hat G_t^{\delta}(x_0) - J^{\delta}_{0 \to t} (x_0) x'\|_2 \\
    & + \|J^{\delta}_{0 \to t} (x_0) x' - J_{0 \to t}(x_0) x'\|_2. 
\end{split} 
\end{align}
Note that 
\begin{align*}
    \|G_t(x_0 + x') - G_t(x_0) -\hat G_t^{\delta}(x_0 + x') + \hat G_t^{\delta} (x_0) \|_2 = \lim_{\delta' \to 0^+} \|\hat G_t^{\delta'}(x_0 + x') - \hat G_t^{\delta'}(x_0) -\hat G_t^{\delta}(x_0 + x') + \hat G_t^{\delta} (x_0) \|_2. 
\end{align*}
Without loss, we may only consider $x' \in \cB(1)$. 
For all $\delta_1, \delta_2 \in (0, \delta]$, by the mean value theorem
\begin{align*}
    & \|\hat G_t^{\delta_1}(x_0 + x') - \hat G_t^{\delta_1}(x_0) -\hat G_t^{\delta_2}(x_0 + x') + \hat G_t^{\delta_2} (x_0) \|_2 \\
    \leq & \|\nabla \hat G_t^{\delta_1} (x_0 + \alpha x') - \nabla \hat G_t^{\delta_2}(x_0 + \alpha x') \|_{\op} \cdot \|x'\|_2 \\
    \leq  & \limsup_{\delta_1, \delta_2 \in (0, \delta]} \sup_{x \in \cB(x_0, 1)} \|J_{0 \to t}^{\delta_1}(x) - J_{0 \to t}^{\delta_2}(x)\|_{\op} \cdot \|x'\|_2 = c(\delta)  \cdot \|x'\|_2, 
\end{align*}
where by Eq.~\eqref{eq:uniform-J} we have $\lim_{\delta \to 0^+}c(\delta) = 0$. 
Also by Eq.~\eqref{eq:uniform-J}, we see that $\|J_{0 \to t}^{\delta}(x_0)x' - J_{0 \to t} (x_0) x'\|_2 \leq c'(\delta) \cdot \|x'\|_2$, with $c'(\delta) \to 0^+$ as $\delta \to 0^+$. 
Plugging these results back into Eq.~\eqref{eq:GGJ}, we conclude that for any $\varepsilon > 0$, there exists $\delta_{\varepsilon} > 0$ such that for all $\delta \leq \delta_{\varepsilon}$, 
\begin{align*}
    \big\| G_t(x_0 + x') - G_t(x_0) - J_{0 \to t} (x_0) x' \big\|_2 \leq \varepsilon \|x'\|_2 + \|\hat G_t^{\delta} (x_0 + x') - \hat G_t^{\delta}(x_0) - J^{\delta}_{0 \to t} (x_0) x'\|_2. 
\end{align*}
By definition, we have $\lim_{\|x'\|_2 \to 0^+}\|\hat G_t^{\delta} (x_0 + x') - \hat G_t^{\delta}(x_0) - J^{\delta}_{0 \to t} (x_0) x'\|_2 = 0$. 
Since $\varepsilon$ can be arbitrarily small, we then conclude that 
\begin{align}
\label{eq:nabla-G-J}
    \nabla_{x_0} G_t(x_0) = J_{0 \to t}(x_0)
\end{align}
for all $t \in [0, T]$ and $x_0 \in \RR^d$. 

Finally, we are ready to prove the existence of a probability density. 
Recall that $\cI_t(x_t) \subseteq \cB(f_t(\|x_t\|_2))$. Therefore, for any $R > 0$ we have $G_t^{-1} (\cB(R)) \subseteq \cB(f_t(R))$.
We choose $R \in \RR_{> 0}$ large enough such that $\|x_t\|_2 < R$.   
By Eq.~\eqref{eq:nabla-G-J} we know that the mappint $x_0 \mapsto x_t = G_t(x_0)$ has everywhere non-degenerate Jacobian matrix. 
Applying the inverse mapping theorem \citep{rudin1976principles}, we conclude that for all $x \in \cB(f_t(R))$, there exists an open set $\cS(x)$ that contains $x$, such that $G_t$ is injective on $\cS(x)$, and the inverse is continuously differentiable. 
We denote this mapping by $h_{\cS(x)}$ that is defined on $\cS(x)$.
By the Heine–Borel theorem \citep{borel1928leccons}, $\cB(f_t(R))$ is covered by finitely many such $\cS(x)$, and we denote by $\mathfrak{S}_R$ the collection of such $\cS(x)$.
As a consequence, we conclude that for all $x_t \in \RR$, there are finitely many $x \in \RR^d$ that satisfies $G_t(x) = x_t$, i.e., $|\cI_t(x_t)| < \infty$. 
Therefore, $p_t(x_t) = \sum_{x \in \cI_t(x_t)} p_0 (x) \det (J_{0 \to t}(x))^{-1}$.
The proof is complete. 

%For $x \in \cB(R)$, we define
%
%\begin{align*}
%    p(x) = \sum_{\cS \in \mathfrak{S}_R} Q()
%\end{align*}

\subsection{Proof of Theorem \ref{thm:DDIM-diversity}}
\label{sec:proof-thm:DDIM-diversity}

We present in this section a rigorous proof of Theorem \ref{thm:DDIM-diversity}. 
Recall that we have proved in the first part of Appendix \ref{sec:proof-lemma:density-exist} that $z_t = M_t z_0 + \xi_t$, where $M_t \in \RR^{d \times d}$ and $\xi_t \in \RR^d$ are functions of $t$ only. 
We define $x_t' = M_t x_0 + \xi_t$, and denote its differential entropy by $H'(t)$. Through standard computation, we see that $H'(t)$ and $H_0(t)$ exist and satisfy
\begin{align*}
    H'(t) = H(0) + \log \det (M_t) \leq H_0(0) + \log \det(M_t), \qquad H_0(t) = H_0(0) + \log \det(M_t). 
\end{align*}
Therefore, in order to show $H(t) \leq H_0(t)$, it suffices to prove $H(t) \leq H'(t)$. 
One caveat is that we still have to show $H(t)$ exists (in the sense of Lebesgue measure). 

Recall that $\mathfrak{S}_R$ is a collection of covering sets introduced at the end of Appendix \ref{sec:proof-lemma:density-exist}. We can in fact choose the coverings appropriately such that for all $m \in \NN_+$, $\mathfrak{S}_m \subseteq \mathfrak{S}_{m + 1}$. Define $\mathfrak{S}_{\infty} = \cup_{m = 1}^{\infty} \mathfrak{S}_{m} = \{\cS_i: i \in \NN_+\}$. 
Here, recall each $\cS_i$ is an open set. 
For all $i \in \NN_+$, we let $\bar \cS_i = \cS_i \backslash (\cup_{j = 1}^{i - 1} \cS_j)$. Then it holds that $\bar \cS_i \cap \bar \cS_j = \emptyset$ for all $i \neq j$, and $\cup_{i = 1}^{\infty} \bar \cS_i = \RR^d$. 
We denote by $f_X$ the probability density function of a random variable $X$. 
Recall that in Appendix \ref{sec:proof-lemma:density-exist} we have defined $x_t = G_t(x_0)$ and $\nabla_x G_t(x) = J_{0 \to t}(x)$. 
Furthermore, by Eq.~\eqref{eq:uniform-J} it holds that $\det(J_{0 \to t}(x)) \leq \det (M_t)$ for all $x \in \RR^d$. 
Based on the derivations in Appendix \ref{sec:proof-lemma:density-exist}, we see that 
\begin{align*}
    f_{x_t} (x) = \sum_{z \in \cI_t(x)} \sum_{i = 1}^{\infty} \mathbbm{1}\{z \in \bar \cS_i\} f_{x_0}(z) \det(J_{0 \to t}(z))^{-1}. 
\end{align*}
For $i \in \NN_+$, we define $\bar\cZ_i = \{ x \in \RR^d: M_t^{-1} (x - \xi_t) \in \bar \cS_i \}$. Then $\bar \cZ_i \cap \bar\cZ_j = \emptyset$ for $i \neq j$, and $\cup_{i = 1}^{\infty} \bar \cZ_i = \RR^d$. In addition, 
\begin{align}
\label{eq:B7}
\begin{split}
    - \int_{\bar \cZ_i} f_{x_t'}(w) \log f_{x_t'}(w) \diff w = & - \int_{\bar \cS_i} f_{x_0}(z) \log \left[ f_{x_0} (z) \det(M_t)^{-1} \right] \diff z \\
    \geq & - \int_{\bar \cS_i} f_{x_0}(z) \log \left[ f_{x_0} (z) \det(J_{0 \to t}(z))^{-1} \right] \diff z. 
\end{split} 
\end{align}
By Eq.~\eqref{eq:uniform-J} we know that there exist constants $c_1, c_2 > 0$, such that $\det (J_{0 \to t}(z)) \in (c_1, c_2)$ for all $z \in \RR^d$. 
By assumption, the left hand side above has a finite Lebesgue integral, hence the Lebesgue integral in the second line of right hand side above also exists and is finite.  
Adding up the above terms over $i \in \NN_+$ (recall that $h_{\cS}$ is the restriction of $G_t$ on $S$), we get 
\begin{align*}
    & -\int f_{x_t'}(w) \log f_{x_t'}(w) \diff w \\
    \overset{(i)}{\geq} & - \sum_{i = 1}^{\infty} \int_{\bar \cS_i} f_{x_0}(z) \log \left[ f_{x_0}(z) \det(J_{0 \to t}(z))^{-1} \right] \diff z \\
    \overset{(ii)}{=} & - \sum_{i = 1}^{\infty} \int_{h_{\cS_i}(\bar \cS_i)} f_{x_0}(h_{S_i}^{-1}(x)) \det(J_{0 \to t}(h_{S_i}^{-1}(x)))^{-1} \log \left[ f_{x_0}(h_{S_i}^{-1}(x)) \det(J_{0 \to t}(h_{S_i}^{-1}(x)))^{-1} \right] \diff x \\
    \overset{(iii)}{=} & -  \int \sum_{i = 1}^{\infty} \mathbbm{1}_{h_{\cS_i}(\bar \cS_i)} (x) f_{x_0}(h_{S_i}^{-1}(x)) \det(J_{0 \to t}(h_{S_i}^{-1}(x)))^{-1} \log \left[ f_{x_0}(h_{S_i}^{-1}(x)) \det(J_{0 \to t}(h_{S_i}^{-1}(x)))^{-1} \right] \diff x \\
    \overset{(iv)}{\geq} & - \int f_{x_t}(x) \log f_t(x) \diff x. 
\end{align*}
In the above display, the summation on the right hand side of $\emph(i)$ exists due to Lemma \ref{lemma:B2}, 
$(ii)$ is by the change-of-variable technique for probability density functions, 
$(iii)$ is also due to Lemma \ref{lemma:B2}, and $(iv)$ is because
\begin{align*}
	f_{x_t}(x) = \sum_{i = 1}^{\infty} \mathbbm{1}_{h_{\cS_i}(\bar \cS_i)}(x) f_{x_0}(h_{S_i}^{-1}(x)) \det(J_{0 \to t}(h_{S_i}^{-1}(x)))^{-1}.  
\end{align*}  
The proof is complete.

\subsection{Technical lemmas} 

We collect in this section the technical lemmas that support proof in this section. 

\begin{lemma}
\label{lemma:B2}
	We assume the conditions of Theorem \ref{thm:DDIM-diversity}. Then the following sum exists and is finite: 
	\begin{align*}
		 - \sum_{i = 1}^{\infty} \int_{\bar \cS_i} f_{x_0}(z) \log \left[ f_{x_0}(z) \det(J_{0 \to t}(z))^{-1} \right] \diff z. 
	\end{align*}
	Furthermore, we can exchange the order of integration and summation, in the sense that 
	\begin{align*}
		\sum_{i = 1}^{\infty} \int \mathbbm{1}_{\bar \cS_i}(z) f_{x_0}(z) \log \left[ f_{x_0}(z) \det(J_{0 \to t}(z))^{-1} \right] \diff z =  \int \sum_{i = 1}^{\infty} \mathbbm{1}_{\bar \cS_i}(z) f_{x_0}(z) \log \left[ f_{x_0}(z) \det(J_{0 \to t}(z))^{-1} \right] \diff z. 
	\end{align*}
\end{lemma}
\begin{proof}[Proof of Lemma \ref{lemma:B2}]
	By Eq.~\eqref{eq:uniform-J}, we know that there exist constants $c_1, c_2 > 0$, such that $\det (J_{0 \to t}(z)) \in (c_1, c_2)$ for all $z \in \RR^d$. Using this together with the assumption that the differential entropy of $x_0$ exists and is finite, we conclude that the function $z \mapsto f_{x_0}(z) \log \left[ f_{x_0} (z) \det(J_{0 \to t}(z))^{-1} \right]$ has a finite Lebesgue integral. The desired claims then immediately follow from the properties of Lebesgue integral. 
\end{proof}
\section{Proofs related to the discretized process}

\subsection{Proof of Theorem \ref{thm:DDIM-confidence-dis}}
\label{sec:proof-thm:DDIM-confidence-dis}

\subsubsection*{Proof of the first claim}

For all $k = 0, 1, \cdots, K - 1$, by Eq.~\eqref{eq:X-k+1-X-k}
\begin{align}
\label{eq:dis-ODE-diff}
\begin{split}
    & \langle X_{k + 1}, \mu_y - \mu_{y'} \rangle - \langle X_k, \mu_y - \mu_{y'} \rangle \\
    = &\, \delta_k e^{-(T - t_k)} \Big( (1 + \gs)\|\mu_y\|_2^2 - \gs \sum_{y'' \in \cY} q_{T - t_k} (X_k, y'') \langle \mu_{y''}, \mu_y \rangle - (1 + \gs)\langle \mu_y, \mu_{y'} \rangle \\
    & + \gs \sum_{y'' \in \cY} q_{T - t_k}(X_k, y'') \langle \mu_{y''}, \mu_{y'}\rangle \Big)  \\
    = & \delta_k e^{-(T - t_k)} \Big( \|\mu_y\|_2^2 - \langle \mu_y, \mu_{y'} \rangle + \gs \langle \mu_y - \mu_0, \mu_y - \mu_{y'} \rangle - \gs \sum_{y'' \in \cY} q_{T - t_k}(X_k, y'') \langle \mu_{y''} - \mu_0, \mu_y - \mu_{y'} \rangle  \Big)  \Big) \\
    \geq & \delta_k e^{-(T - t_k)} \Big( \|\mu_y\|_2^2 - \langle \mu_y, \mu_{y'} \rangle + \gs(1 - q_{T - t_k}(X_k, y)) (\|\mu_y - \mu_0\|_2^2 - 3 \varepsilon) \Big).
\end{split}
\end{align}
On the other hand, by definition 
\begin{align*}
    \langle Z_{k + 1}, \mu_y - \mu_{y'} \rangle - \langle Z_k, \mu_y - \mu_{y'}\rangle = \delta_k e^{-(T - t_k)} \left( \|\mu_y\|_2^2 - \langle \mu_y, \mu_{y'} \rangle \right)
\end{align*}
recall that we have assumed $\langle X_0, \mu_y - \mu_{y'}\rangle \geq \langle Z_0, \mu_y - \mu_{y'} \rangle$. 
By induction, we are able to conclude that $\langle X_k, \mu_y - \mu_{y'} \rangle \geq \langle Z_k, \mu_y - \mu_{y'} \rangle$ for all $y' \in \cY$ and $k \in \{0\} \cup [K]$. The first claim of the theorem then immediately follows.

\subsubsection*{Proof of the second claim}

%Inspecting Eq.~\eqref{eq:dis-ODE-diff}, we have
%
%\begin{align*}
%    & \langle X_k, \mu_{y} - \mu_{y'} \rangle - \langle X_0, \mu_y - \mu_{y'} \rangle \\
%    \leq & \sum_{i = 0}^{k - 1} \delta_i e^{-(T - t_i)} \left( \|\mu_y\|_2\|\mu_y - \mu_{y'}\|_2  + 3\gs \varepsilon  + \gs\|\mu_y - \mu_0\|_2^2 + \gs\|\mu_{y'} - \mu_0\|_2^2 \right) \\
%    \leq & (e^{-(T - t_k)} - e^{-T}) \left( \|\mu_y\|_2 \|\mu_y - \mu_{y'}\|_2 + 3\gs\varepsilon + \gs\|\mu_y - \mu_0\|_2^2 + \gs\|\mu_{y'} - \mu_0\|_2^2 \right). 
%\end{align*}
%
The proof closely mirrors that of Theorem \ref{thm:DDIM-confidence-quan}.
Similar to the derivation of Eq.~\eqref{eq:last-step-thm-2.5}, we obtain that 
\begin{align*}
	& \langle X_{k + 1} - Z_{k + 1}, \mu_y - \mu_{y'} \rangle - \langle X_k - Z_k, \mu_y - \mu_{y'} \rangle \\
	\geq & \delta_k e^{-T + t_k} \left( \|\mu_y\|_2^2 - \langle \mu_y, \mu_{y'} \rangle + \gs e^{-\Delta / 8}(\|\mu_y - \mu_0\|_2^2 - 3 \varepsilon) \cdot \min \{1 - \cP(X_k, y), \xi_w\} \right), 
\end{align*}
where we recall that $\xi_w = 1 - w_y / (w_y + \min_{y' \neq y} w_{y'})$. 
Now suppose $\langle X_k - Z_k, \mu_y - \mu_{y'} \rangle \in [0, \cU]$ for all $k \in \{0\} \cup [K]$ and $y' \in \cY$, then like the derivation of Eq.~\eqref{eq:p-F-max}, we get
\begin{align*}
	\cP(X_k, y) \leq 1 - \cF(\max_{0 \leq k \leq K} \cP(Z_k, y), \cU) 
\end{align*}
for all $k \in \{0\} \cup [K]$. 
For $\cU$ to be a valid upper bound, we must have 
\begin{align}
\label{eq:C2}
	\cU \geq & \langle X_0 - Z_0, \mu_y - \mu_{y'} \rangle \nonumber \\
	&  + \sum_{k = 0}^{K - 1} \delta_k e^{-T + t_k} \left( \|\mu_y\|_2^2 - \langle \mu_y, \mu_{y'} \rangle + \gs e^{-\Delta / 8}(\|\mu_y - \mu_0\|_2^2 - 3 \varepsilon) \cdot \min \{\cF(\max_{0 \leq k \leq K} \cP(Z_k, y), \cU), \xi_w\} \right) \nonumber \\
	\geq & \langle X_0 - Z_0, \mu_y - \mu_{y'} \rangle  \\
	& + e^{-\Delta_{\max}} (1 - e^{-T})\left( \|\mu_y\|_2^2 - \langle \mu_y, \mu_{y'} \rangle + \gs e^{-\Delta / 8}(\|\mu_y - \mu_0\|_2^2 - 3 \varepsilon) \cdot \min \{\cF(\max_{0 \leq k \leq K} \cP(Z_k, y), \cU), \xi_w\} \right).  \nonumber
\end{align}
For any $\cU$ that does not satisfy Eq.~\eqref{eq:C2}, we know that $\langle X_K - Z_K, \mu_y - \mu_{y'} \rangle \geq \cU$, and 
\begin{align}
\label{eq:C3}
	\cP(X_K, y) \geq \frac{\cP(Z_K, y)}{\cP(Z_K, y) + (1 - \cP(Z_K, y)) \cdot \exp(-\cU)}, 
\end{align}
completing the proof of the first result. 
As for the proof of the convergence rate, we simply set $e^{-\cU} = \gs^{-1} (\log \gs)^2$. For such $\cU$, the left hand side of Eq.~\eqref{eq:C2} is of order $O(\log \gs)$, while the right hand side of Eq.~\eqref{eq:C2} is of order $O(\gs \wedge (\log \gs)^2)$. 
We then conclude that for a large enough $\gs$, Eq.~\eqref{eq:C2} does not hold. Plugging such $\cU$ back into Eq.~\eqref{eq:C3}, we are able to deduce the desired convergence rate.

%Similar to the derivation of Eq.~\eqref{eq:P-xt-y}, we have $\cP(X_k, y) \leq g(\cP(X_0, y))$, where we recall that $g(\cdot)$ is defined in Eq.~\eqref{eq:P-xt-y} as well. 
%Following the derivation of Eq.~\eqref{eq:upper-bound-tilde-q}-\eqref{eq:qT-t-upper2}, we obtain that $q_{T - {t_k}}(X_k, y) \leq \min \{G (g(\cP(X_0, y))), G(w_y / (w_y + \min_{y' \neq y} w_{y'}))\}$. Plugging this upper bound into Eq.~\eqref{eq:dis-ODE-diff}, we see that 
%
%\begin{align*}
%    & \langle X_{k + 1}, \mu_y - \mu_{y'} \rangle - \langle X_k, \mu_y - \mu_{y'} \rangle \\
%    \geq & \delta_k e^{-(T - t_k)} \left( \|\mu_y\|_2^2 - \langle \mu_y, \mu_{y'} \rangle + \gs (1 - \min \{G (g(\cP(X_0, y))), G(w_y / (w_y + \min_{y' \neq y} w_{y'}))\}) (\|\mu_y - \mu_0\|_2^2 - 3 \varepsilon) \right). 
%\end{align*}
%
%On the other hand, note that $\langle Z_{k + 1} - Z_k, \mu_y - \mu_{y'} \rangle = \delta_k e^{-(T - t_k)} (\|\mu_y\|_2^2 - \langle \mu_y, \mu_{y'} \rangle)$.
%Taking the difference and summing over $k$, we conclude that
%
%\begin{align*}
%    & \langle X_k, \mu_y - \mu_{y'} \rangle - \langle Z_k, \mu_y - \mu_{y'} \rangle \\
%    \geq & \sum_{i = 0}^{k - 1} \delta_i e^{-(T - t_i)}   \gs (1 - \min \{G (g(\cP(X_0, y))), G(w_y / (w_y + \min_{y' \neq y} w_{y'}))\}) (\|\mu_y - \mu_0\|_2^2 - 3 \varepsilon)   \\
%    \geq & e^{-\max_{j \in [k]} \delta_{j - 1}}(e^{-T + t_k} - e^{-T}) M. 
%\end{align*}
%
%The proof is complete. 

\subsection{Proof of Theorem \ref{thm:DDIM-entropy-dis}}
\label{sec:proof-thm:DDIM-entropy-dis}

For $k \in \{0\} \cup [K - 1]$, we define
\begin{align*}
    F_k(x, \gs) = x + \delta_k \cdot \Big( x - \Sigma_{T - t_k}^{-1} x + e^{-T + t_k} (1 + \gs) \Sigma_{T - t_k}^{-1} \mu_y - \gs e^{-T + t_k} \sum_{y' \in \cY} q_{T - t_k} (x, y') \Sigma_{T - t_k}^{-1} \mu_{y'} \Big). 
\end{align*}
Observe that $X_{k + 1} = F_k(X_k, \gs)$ and $Z_{k + 1} = F_k(Z_k, 0)$. 
We then take the gradient of $F_k$ with respect to the first argument, which gives
\begin{align*}
    \nabla_x F_k(x, \gs) = I_d + \delta_k \Big( I_d - \Sigma_{T - t_k}^{-1} - \gs e^{-T + t_k} \Omega_{T - t_k}(x)  \Big),
\end{align*}
where
\begin{align*}
    & \Omega_{T - t_k}(x) \\
    = & \sum_{y' \in \cY} q_{T - t_k}(x, y') \Sigma_{T - t_k}^{-1} \mu_{y'} \mu_{y'}^{\top} \Sigma_{T - t_k}^{-1} - \Big( \sum_{y' \in \cY} q_{T - t_k}(x, y') \Sigma_{T - t_k}^{-1} \mu_{y'} \Big) \Big( \sum_{y' \in \cY} q_{T - t_k}(x, y') \Sigma_{T - t_k}^{-1} \mu_{y'} \Big)^{\top}. 
\end{align*}
We then conclude that $\nabla_x F_k(x, \gs)  \preceq \nabla_x F_k(x, 0)$ for all $\gs \geq 0$. 
We denote by $\sigma_{\min}(X)$ the minimum eigenvalue of a matrix $X$. 
Observe that 
\begin{align}
\label{eq:sigma-min-Fk}
\begin{split}
   & \sigma_{\min} (\nabla_x F_k(x, \gs)) \geq 1 + \delta_k - \frac{\delta_k}{\sigma_{\min}(\Sigma) \wedge 1 } - \frac{\delta_k \gs \sup_{y' \in \cY} \|\mu_{y'}\|_2^2}{\sigma_{\min}(\Sigma)^2 \wedge 1}, \\
   & \sigma_{\min} (\nabla_x F_k(x, 0)) \geq 1 + \delta_k - \frac{\delta_k}{\sigma_{\min}(\Sigma) \wedge 1 },  
\end{split}
\end{align}
both are strictly positive under the current set of assumptions. 
Observe that $x \mapsto F_k(x, 0)$ is an affine transformation, then it is also bijective,  
while $F_k(x, \gs)$ is not necessarily one-to-one. 
For $x \in \RR^d$ and $\gs \geq 0$, we let $G_k(x, \gs) := F_k(F_{k - 1}( \cdots F_0(x, \gs) \cdots, \gs), \gs)$. Then $X_{k + 1} = G_k(X_0, \gs)$ and $Z_{k + 1} = G_k(Z_0, 0)$. Observe that there exists $A_k \in \RR^{d \times d}$ and $\beta_k \in \RR^d$, such that $G_k(x, 0) = A_k x + \beta_k$.
In addition, one can verify that $A_k$ is non-degenerate. 

In the sequel, we use $f_X$ to represent the probability density function of a random variable $X$. Utilizing a change-of-variable technique, we see that
$\det(A_k) \cdot f_{Z_{k + 1}}(A_k x + \beta_k) = f_{Z_0}(x)$ for all $x \in \RR^d$. We can then express the differential entropy of $Z_{k + 1}$ based on that of $Z_0$: 
\begin{align*}
    \cH_0(k + 1) = - \int f_{Z_{k + 1}}(x) \log f_{Z_{k + 1}}(x) \diff x = \cH_0(0) + \log \det (A_k). 
\end{align*}
In the next lemma, we demonstrate  the existence of probability density function of $X_{k + 1}$ with respect to the Lebesgue measure. 
\begin{lemma}
    \label{lemma:existence}
    We assume the conditions of Theorem \ref{thm:DDIM-entropy-dis}. Then, for all $k \in \{0\} \cup [K]$, $X_k$ has a probability density function with with respect to the Lebesgue measure. 
\end{lemma}
The proof of Lemma \ref{lemma:existence} is similar to that of Lemma \ref{lemma:density-exist}, and we skip it for the compactness of presentation. 
We define $X_{k + 1}' = A_k X_0 + \beta_k$, and denote the differential entropy of $X_{k + 1}'$ by $\cH'(k + 1)$. Similarly, we have $\cH'(k + 1) = \cH(0) + \log \det (A_k) \leq \cH_0(0) + \log \det (A_k)$. Therefore, in order to prove $\cH(k + 1) \leq \cH_0(k + 1)$, it suffices to show $\cH(k + 1) \leq \cH'(k + 1)$. 

The remainder proof follows analogously as that of Lemma \ref{lemma:density-exist}. 
Taking the Jacobian matrix of $G_k(x, \gs)$ with respect to the first argument, we get 
\begin{align*}
    \nabla G_k(x_0, \gs) = \nabla F_k(x_k, \gs) \cdot \nabla F_{k - 1}(x_{k - 1}, \gs) \cdots \nabla F_{0}(x_0, \gs),
\end{align*}
where $x_i = G_{i - 1}(x_0, \gs)$. By Eq.~\eqref{eq:sigma-min-Fk}, we know that $\nabla G_k(x_0, \gs)$ is everywhere non-degenerate. 
In addition, by induction we conclude that $\|X_{k + 1}\|_2 \geq \alpha_{k + 1} \|X_0\|_2 - \gamma_{k + 1}$, where $\alpha_{k + 1} \in \RR_{> 0}$ and $\gamma_k \in \RR$. 
We define 
\begin{align*}
    \cI_{k + 1}(X_{k + 1}) := \{X_0 \in \RR^d, G_k(X_0, \gs) = X_{k + 1}\}. 
\end{align*}
Then $\cI_{k + 1}(X_{k + 1}) \subseteq \cB(\alpha_{k + 1}^{-1} (\|X_{k + 1}\|_2 + \gamma_{k + 1}))$. 
By the inverse mapping theorem, we obtain that for all $x \in \RR^d$, there exists an open set $\cS(x)$ that contains $x$, such that $G_k(\cdot, \gs)$ is injective on $\cS(x)$. We denote this injection by $h_{\cS}$. By Heine–Borel theorem, $\cB(\alpha_{k + 1}^{-1} (\|X_{k + 1}\|_2 + \gamma_{k + 1})) $ can be covered by finitely many $\cS(x)$.
Therefore, for all $x \in \RR^d$, we have $f_{X_{k + 1}}(x) = \sum_{z \in \cI_{k + 1}(x)} f_{X_0}(z) \det(\nabla G_k(z, \gs))^{-1}$. 
In addition,  $f_{X_{k + 1}'}(A_k z + \beta_k) = f_{X_0}(z) \cdot \det(A_k)^{-1}$ and $\det(A_k)^{-1} \leq  \det(\nabla G_k(z, \gs))^{-1}$.
By the assumption that the differential entropy of $X_0$ exists and is finite, we may also conclude that the differential entropy of $X_{k + 1}'$ exists and is finite.

When $\alpha_{k + 1}^{-1} (\|X_{k + 1}\|_2 + \gamma_{k + 1}) = R$, we denote by $\mathfrak{S}_R$ the collection of such $\cS(x)$. 
We can construct $\mathfrak{S}_R$ for every positive $R$. 
It is not hard to see that we can choose the coverings appropriately such that for all $m \in \NN_+$, we have $\mathfrak{S}_m \subseteq \mathfrak{S}_{m + 1}$. 
Consider the union of these covering sets: $\mathfrak{S}_{\infty} = \cup_{m = 1}^{\infty} \mathfrak{S}_m = \{\cS_i: i \in \NN_+\}$. 
We define $\bar \cS_i = \cS_i \backslash (\cup_{j = 1}^{i - 1} \cS_j)$, and let $\bar\cZ_i = \{x \in \RR^d: A_k^{-1} (x - \beta_k) \in \bar\cS_i\}$.
Note that $\bar\cS_i \cap \bar \cS_j = \emptyset$ for all $i \neq j$ and $\cup_{i = 1}^{\infty} \bar \cS_i = \RR^d$.
Then, it hold that
\begin{align*}
    f_{X_{k + 1}}(x) = \sum_{z \in \cI_{k + 1}(x)} \sum_{i = 1}^{\infty} \mathbbm{1}\{z \in \bar\cS_i\} f_{X_0}(z) \det (\nabla G_k(z, \gs))^{-1}
\end{align*}
Note that for all $i \in \NN_+$, 
\begin{align*}
     - \int_{\bar \cZ_i} f_{X_{k + 1}'} (w) \log f_{X_{k + 1}'} (w) \diff x = & - \int_{\bar \cS_i} f_{X_0}(z)\log \left[ f_{X_0}(z) \cdot \det(A_k)^{-1} \right] \diff z \\
     \geq & - \int_{\bar \cS_i} f_{X_0}(z) \log \left[ f_{X_0}(z) \cdot \det(\nabla G_k(z, \gs))^{-1} \right] \diff z. 
    %\\
    %\geq & - \int_{ h_{\cS_i}(\bar \cS_i)} f_{X_{k + 1}} (x) \log f_{X_{k + 1}}(x) \diff x. 
\end{align*}
Summing both sides of the above equality over $i$, we get
\begin{align*}
    & - \int f_{X_{k + 1}'} (w) \log f_{X_{k + 1}'}(w) \diff w \\
    %\geq & - \sum_{i \in \NN_+}\int_{\bar \cS_i} f_{X_0}(z) \log \left[ f_{X_0}(z) \cdot \det(\nabla G_k(z, w))^{-1} \right] \diff z \\
    \geq & - \sum_{i \in \NN_+} \int_{\bar \cS_i} f_{X_0}(z) \log \left[ f_{X_0}(z) \cdot \det(\nabla G_k(z, \gs))^{-1} \right] \diff z. \\
    = & - \sum_{i \in \NN_+} \int_{h_{\cS_i}(\bar \cS_i)} f_{X_0} (h_{\cS_i}^{-1}(x)) \cdot \det(\nabla G_k(h_{\cS_i}^{-1}(x), \gs))^{-1} \log \left[ f_{X_0} (h_{\cS_i}^{-1}(x)) \cdot \det(\nabla G_k(h_{\cS_i}^{-1}(x), \gs))^{-1} \right] \diff x  \\
    \overset{(i)}{=} & - \int \sum_{i \in \NN_+} \mathbbm{1}\{x \in h_{\cS_i}(\bar \cS_i)\}  f_{X_0} (h_{\cS_i}^{-1}(x)) \cdot \det(\nabla G_k(h_{\cS_i}^{-1}(x), \gs))^{-1} \log \left[ f_{X_0} (h_{\cS_i}^{-1}(x)) \cdot \det(\nabla G_k(h_{\cS_i}^{-1}(x), \gs))^{-1} \right] \diff x \\
    \geq & - \int f_{X_{k + 1}} (x) \log f_{X_{k + 1}}(x) \diff x, 
\end{align*}
where to exchange the order of summation and integration in $(i)$ we make use of the assumption that the differential entropy of $X_0$ exists and is finite. 
The proof is complete.

% Let $\cI_{k + 1} (x) = \{z_1, z_2, \cdots, z_{n_x}\}$, and for $i \in [n_x]$ we let $w_i = A_k z_i + \beta_k$. 
% Then, 
% %
% \begin{align*}
%     f_{X_{k + 1}}(x) \diff x = \sum_{i \in [n_x]} f_{X_0} (z_i) \diff z_i = \sum_{i \in [n_x]} f_{X_{k + 1}'} (w_i) \diff w_i. 
% \end{align*}
% %
% Taking the integral, we get
% %
% \begin{align*}
%     - \int f_{X_{k + 1}} (x) \log f_{X_{k + 1}}(x) \diff x = & - \sum_{i \in [n_x]} \int f_{X_{k + 1}'} (w_i) \log f_{X_{k + 1}} (x) \diff w_i \\
%     = & - \sum_{i \in [n_x]} \int f_{X_{k + 1}'}(w_i) \log \Big[ \sum_{i \in [n_x]} f_{X_0}(z_i) \det (\nabla G_k(z_i, w))^{-1} \Big] \diff w_i \\
%     \leq & - \sum_{i \in [n_x]} \int f_{X_{k + 1}'}(w_i) \log \Big[ \sum_{i \in [n_x]} f_{X_0} (z_i) \det (A_k)^{-1} \Big] \diff w_i \\
%     \leq & - \sum_{i \in [n_x]} \int f_{X_{k + 1}'}(w_i) \log f_{X_{k + 1}'}(w_i) \diff w_i. 
% \end{align*}
% %
% The proof is complete. 

%The proof for the remaining part closely mirrors that of Theorem \ref{thm:DDIM-confidence-quan}, and we skip it for the compactness of presentation. 

\subsection{Proof of Theorem \ref{thm:DDPM-confidence-dis}}
\label{sec:proof-thm:DDPM-confidence-dis}

\subsubsection*{Proof of the first claim}

Observe that for the guided process, 
\begin{align*}
   &  \langle \bar X_{k + 1}, \mu_y - \mu_{y'} \rangle  \\
   = & (1 - \delta_k)\langle \bar X_k, \mu_y - \mu_{y'} \rangle + 2\delta_k e^{-(T - t_k)} \Big( (1 + \gs)\|\mu_y\|_2^2 - \gs \sum_{y'' \in \cY} q_{T - t_k} (\bar X_k, y'') \langle \mu_{y''}, \mu_y \rangle - (1 + \gs)\langle \mu_y, \mu_{y'} \rangle \\
    & + \gs \sum_{y'' \in \cY} q_{T - t_k}(\bar X_k, y'') \langle \mu_{y''}, \mu_{y'}\rangle \Big)  + \sqrt{2 \delta_k} \langle W_k, \mu_y - \mu_{y'} \rangle \\
    \geq & (1 - \delta_k)\langle \bar X_k, \mu_y - \mu_{y'} \rangle + 2\delta_k e^{-(T - t_k)} \Big( \|\mu_y\|_2^2 - \langle \mu_y, \mu_{y'} \rangle + \gs(1 - q_{T - t_k}(\bar X_k, y)) (\|\mu_y - \mu_0\|_2^2 - 3 \varepsilon) \Big) \\
    &  + \sqrt{2 \delta_k} \langle W_k, \mu_y - \mu_{y'} \rangle. 
\end{align*}
As for the unguided process, we have 
\begin{align*}
     \langle \bar Z_{k + 1}, \mu_y - \mu_{y'} \rangle =  (1 - \delta_k) \langle \bar Z_k, \mu_y - \mu_{y'}\rangle + 2\delta_k e^{-(T - t_k)} (\|\mu_y\|_2^2 - \langle \mu_y, \mu_{y'} \rangle) + \sqrt{2 \delta_k} \langle W_k, \mu_y - \mu_{y'} \rangle. 
\end{align*}
By induction, we know that $\langle \bar X_k, \mu_y - \mu_{y'} \rangle \geq \langle \bar Z_k, \mu_y - \mu_{y'} \rangle$ for all $k \in \{0\} \cup [K]$. 
The first claim then immediately follows.

\subsubsection*{Proof of the second claim}

Similar to the derivation of Eq.~\eqref{eq:C-A17}, we obtain that 
\begin{align*}
	& \langle \bar X_{k + 1} - \bar X_k, \mu_y - \mu_{y'} \rangle \\
	\geq & - \delta_k \langle \bar X_k, \mu_y - \mu_{y'} \rangle + 2 \delta_k e^{-T + t_k} \Big( \|\mu_y\|_2^2 -  \langle \mu_y, \mu_{y'} \rangle + \gs e^{ - \Delta / 8} \min \{1 - \cP(\bar X_k, y), \xi_w\} (\|\mu_y - \mu_0\|_2^2 - 3 \varepsilon) \Big) \\
	& + \sqrt{2\delta_k} \langle W_k, \mu_y - \mu_{y'} \rangle.  
\end{align*}
As for the unguided process, we have 
\begin{align*}
		\langle \bar Z_{k + 1} - \bar Z_k, \mu_y - \mu_{y'} \rangle =  - \delta_k \langle \bar Z_k, \mu_y - \mu_{y'} \rangle + 2 \delta_k e^{-T + t_k} \Big( \|\mu_y\|_2^2 -  \langle \mu_y, \mu_{y'} \rangle\Big)+ \sqrt{2\delta_k} \langle W_k, \mu_y - \mu_{y'} \rangle.
\end{align*}
Taking the difference, we see that 
\begin{align}
\label{eq:C5-new}
\begin{split}
	& \langle \bar X_{k + 1} - \bar Z_{k + 1}, \mu_y - \mu_{y'} \rangle \\
	 \geq & (1 - \delta_k ) \langle \bar X_{k} - \bar Z_{k}, \mu_y - \mu_{y'} \rangle + 2 \delta_k \gs e^{-T + t_k  - \Delta / 8}  \min \{1 - \cP(\bar X_k, y), \xi_w\} (\|\mu_y - \mu_0\|_2^2 - 3 \varepsilon).
\end{split} 
\end{align}
From the above equation as well as our initial assumption we know that $\langle \bar X_k - \bar Z_k, \mu_y - \mu_{y'}\rangle \geq 0$ for all $k \in \{0\} \cup [K]$. Now suppose $\langle \bar X_k - \bar Z_k, \mu_y - \mu_{y'}\rangle \in [0, \bar \cU]$ for all $k \in \{0\} \cup [K]$. Then similar to the derivation of Eq.~\eqref{eq:p-F-max}, we know that 
\begin{align}
\label{eq:C5}
	\cP(\bar X_k, y) \leq 1 - \cF\big( \max_{0 \leq k \leq K} \cP(\bar Z_k, y), \bar \cU \big)
\end{align}
for all $k \in \{0\} \cup [K]$. 
By Eq.~\eqref{eq:C5-new} and \eqref{eq:C5} and induction hypothesis, we get the following lower bound: 
\begin{align}
\label{eq:long-long}
\begin{split}
    & \langle \bar X_{K} - \bar Z_{K}, \mu_y -\mu_{y'} \rangle - e^{-T - \Delta_{\max}} \langle \bar X_0 -  \bar Z_{0}, \mu_y - \mu_{y'} \rangle \\
    \geq & \sum_{k = 0}^{K - 1} 2\gs \delta_k e^{-T + t_k - \Delta / 8}  \min \{\cF\big( \max_{0 \leq k \leq K} \cP(\bar Z_k, y), \bar \cU \big), \xi_w\} (\|\mu_y - \mu_0\|_2^2 - 3 \varepsilon) \cdot \prod_{j = k + 1}^{K - 1} (1 - \delta_j) \\
    \geq & \sum_{k = 0}^{K - 1} 2\gs \delta_k e^{-T + 2 t_{k + 1} - \Delta / 8 }  \min \{\cF\big( \max_{0 \leq k \leq K} \cP(\bar Z_k, y), \bar \cU \big), \xi_w\}  (\|\mu_y - \mu_0\|_2^2 - 3 \varepsilon) \cdot e^{-2T} \\
    \geq &\, \gs e^{-\Delta / 8} (e^{-T} - e^{-3T}) \min \{\cF\big( \max_{0 \leq k \leq K} \cP(\bar Z_k, y), \bar \cU \big), \xi_w\}  (\|\mu_y - \mu_0\|_2^2 - 3 \varepsilon). 
    %\geq  & \sum_{i = 0}^k 2w \delta_i e^{-T + 2t_i + 2 t_{i + 1} }  (1 - \bar G_d(q_T(\bar X_0, y))) (\|\mu_y - \mu_0\|_2^2 - 3 \varepsilon) \cdot e^{-3t_{k + 1}} \\
    %\geq & \frac{e^{-T + t_{k + 1}} - e^{-T - 3 t_{k + 1}}}{2}w (1 - \bar G_d(q_T(\bar X_0, y)))(\|\mu_y - \mu_0\|_2^2 - 3 \varepsilon).
\end{split}
\end{align}
Then for $\bar \cU$ to serve as a valid upper bound, we must have
\begin{align}
\label{eq:C7}
	& \bar\cU - e^{-T - \Delta_{\max}} \langle \bar X_0 -  \bar Z_{0}, \mu_y - \mu_{y'} \rangle \\
	 \geq & \gs e^{-\Delta / 8} (e^{-T} - e^{-3T}) \min \{\cF\big( \max_{0 \leq k \leq K} \cP(\bar Z_k, y), \bar \cU \big), \xi_w\}  (\|\mu_y - \mu_0\|_2^2 - 3 \varepsilon). \nonumber 
\end{align}
Hence, if Eq.~\eqref{eq:C7} is not satisfied, then there exists $k \in \{0\} \cup [K]$ such that $\langle \bar X_k - \bar Z_k, \mu_y - \mu_{y'} \rangle \geq \bar \cU$. 
As a consequence, we have $\langle \bar X_{K} - \bar Z_{K}, \mu_y - \mu_{y'} \rangle \geq \prod_{j = k}^{K - 1}(1 - \delta_j) \bar \cU \geq e^{-2T} \bar \cU$, which further implies that
\begin{align*}
	\cP(\bar X_K, y) \geq \frac{\cP(\bar Z_K, y)}{\cP(\bar Z_K, y) + (1 - \cP(\bar Z_K, y)) \cdot \exp(- e^{-2T} \bar \cU)}. 
\end{align*}
The proof of the first result is complete. The proof of the convergence rate follows analogously as that of the second part of Theorems \ref{thm:DDPM-confidence} and \ref{thm:DDPM-two-clusters}. Here, we skip it for the compactness of presentation.

\subsection{Proofs in Section~\ref{sec:example-curious}}\label{pf:negative_result}
\paragraph{Assumption~\ref{assumption:confidence} does not hold for $\mu_{\rm neg}$}
It suffices to argue for the center component in $\mu_{\rm neg}$. Suppose for contradiction that there exists a vector $\mu_0$ and a positive $\varepsilon$ satisfying
\begin{align*}
\left| \inner{-\mu_0}{\mu - \mu_0} \right| \leq \varepsilon, \quad \left| \inner{-\mu_0}{-\mu - \mu_0} \right| \leq \varepsilon, \quad \norm{\mu_0}_2^2 / 3 \geq \varepsilon.
\end{align*}
Rewriting the first two inequalities, we obtain
\begin{align*}
\left| \inner{\mu_0}{\mu} + \norm{\mu_0}_2^2 \right| \leq \varepsilon \quad \text{and} \quad \left| -\inner{\mu_0}{\mu} + \norm{\mu_0}_2^2 \right| \leq \varepsilon.
\end{align*}
Due to the symmetry, we assume without loss of generality that $\inner{\mu_0}{\mu} \geq 0$. By comparing
\begin{align*}
\inner{\mu_0}{\mu} + \norm{\mu_0}_2^2 \leq \varepsilon \quad \text{with} \quad \norm{\mu_0}_2^2/3 \geq \varepsilon,
\end{align*}
we must have $\mu_0 = 0$ and $\varepsilon = 0$. This contradicts the fact that $\varepsilon$ is positive. Therefore, the first item in Assumption~\ref{assumption:confidence} does not hold.

\paragraph{Proof of Proposition~\ref{prop:phase_change}} We focus on generating the center component ${\sf N}(0, I_d)$. Setting the guidance strength parameter $\gs$ and using the discretized DDIM backward process yield
\begin{align}\label{eq:negative_DDIM_score}
X_{k+1} = X_k - \delta_k \gs e^{-T + t_k} \frac{\exp(e^{-T + t_k} X_k^\top \mu) - \exp(- e^{-T + t_k} X_k^\top \mu)}{\exp(e^{-2T+2t_k} \norm{\mu}_2^2/2) + \exp(e^{-T + t_k} X_k^\top \mu) + \exp(- e^{-T + t_k} X_k^\top \mu)}\mu.
\end{align}
Taking inner product with $\mu$ on both sides of Eq.~\eqref{eq:negative_DDIM_score} gives rise to
\begin{align*}
\inner{X_{k+1}}{\mu} - \inner{X_k}{\mu} = - \delta_k \gs e^{-T + t_k} \frac{\exp(e^{-T + t_k} X_k^\top \mu) - \exp(- e^{-T + t_k} X_k^\top \mu)}{\exp(e^{-2T+2t_k} \norm{\mu}_2^2/2) + \exp(e^{-T + t_k} X_k^\top \mu) + \exp(- e^{-T + t_k} X_k^\top \mu)} \norm{\mu}_2^2.
\end{align*}
We denote $v_k = \inner{X_k}{\mu}$ and cast the last display into
\begin{align}\label{eq:vk_update}
v_{k+1} - v_k = - \delta_k \gs e^{-T + t_k} \frac{\exp(e^{-T + t_k} v_k) - \exp(- e^{-T + t_k} v_k)}{\exp(e^{-2T+2t_k} \norm{\mu}_2^2/2) + \exp(e^{-T + t_k} v_k) + \exp(- e^{-T + t_k} v_k)} \norm{\mu}_2^2.
\end{align}
Examining Eq.~\eqref{eq:vk_update} suggests that $v_k (v_{k+1} - v_k) \leq 0$, which implies that the increment of $v_{k}$ has an opposite sign as $v_k$ itself. We show the following stronger version of Proposition~\ref{prop:phase_change}.
\begin{lemma}\label{lemma:phase_change}
Consider the Gaussian mixture model in Eq.~\eqref{eq:gmm_negative}. There exist positive constants $\gs_0 \leq \frac{1}{\norm{\mu}_2^2 \max \delta_k}$ and $\gs_0'$ that depend on discretization step sizes $\{\delta_k\}_{k=0}^{K-1}$, such that for any $k$ verifying $e^{-T + t_k} \geq 1/2$, it holds that
\begin{enumerate}
\item when $\gs \leq \gs_0$, $v_k$ evolves towards $0$, i.e., $|v_{k+1}| < |v_k|$ if $v_k \neq 0$. Furthermore, for a small $\gamma \in (0,1)$ satisfying 
$$\frac{\gamma (\exp(\|\mu\|^2/2) + 2\exp(|v_k|))^2}{2\exp(\frac{\|\mu\|^2}{8})+4} \leq \delta_k \gs e^{-2T + 2t_k}\|\mu\|_2^2 \leq \frac{(4-2\gamma)\left(2+\exp\big(\frac{\|\mu\|_2^2}{8}\big)\right)^2}{8 + \left(2+\exp\big(\frac{\|\mu\|_2^2}{8}\big)\right)^2},$$
we have $|v_{k+1}| \leq (1 - \gamma)|v_k| $. One can easily verify the existence of such $\gamma$. 

%$$\frac{2\exp(\frac{\|\mu\|^2}{4})+4}{(\exp(\|\mu\|^2/2) + 2\exp(|v_k|))^2}\leq \eta \leq \frac{(4-2\gamma)\left(2+\exp\big(\frac{\|\mu\|_2^2}{4}\big)\right)^2}{8 + \left(2+\exp\big(\frac{\|\mu\|_2^2}{4}\big)\right)^2},$$ we have $|v_{k+1}| \leq (1 - \gamma)|v_k| $. One can easily verify the existence of such $\gamma$. 
\item when $\gs \geq \gs_0'$, there exists positive $a$ and $b$ dependent on $\gs$, and it holds that
\begin{align*}
|v_{k+1}| > |v_k| & \quad \text{if}\quad |v_k| \in (0, a]; \\
|v_{k+1}| < |v_k| & \quad \text{if} \quad |v_{k}| > b.
\end{align*}
In particular, thresholds $a$ and $b$ increase as $\gs$ increases. 
\end{enumerate}
\end{lemma}
\begin{proof}
%According to the dynamic \eqref{eq:vk_update}, we 
Observe that the right-hand side of Eq.~\eqref{eq:vk_update} as a function of $v_k$ is symmetric about $0$. Therefore, it is enough to consider $v_k \geq 0$. 
We study the solution of the equation
\begin{align}\label{eq:stable_equation}
2v_k = \delta_k \gs e^{-T + t_k} \frac{\exp(e^{-T + t_k} v_k) - \exp(- e^{-T + t_k} v_k)}{\exp(e^{-2T+2t_k} \norm{\mu}_2^2/2) + \exp(e^{-T + t_k} v_k) + \exp(- e^{-T + t_k} v_k)} \norm{\mu}_2^2.
\end{align}
Intuitively, the solution of Eq.~\eqref{eq:stable_equation} implies that for such $v_k$, after one iteration, it holds that $v_{k+1} = -v_k$. We denote
\begin{align}\label{eq:h_vk_k}
h(v_k, k) = \delta_k \gs e^{-T + t_k} \frac{\exp(e^{-T + t_k} v_k) - \exp(- e^{-T + t_k} v_k)}{\exp(e^{-2T+2t_k} \norm{\mu}_2^2/2) + \exp(e^{-T + t_k} v_k) + \exp(- e^{-T + t_k} v_k)} \norm{\mu}_2^2 - 2v_k.
\end{align}
To prove the lemma, below we will establish the following dichotomy for appropriate $\gs_0$ and $\gs_0'$: 1) when $\gs < \gs_0$, $v_k = 0$ is the only solution to $h(v_k, k) = 0$; 2) when $\gs > \gs_0'$, $h(v_k, k) = 0$ has multiple solutions.

\paragraph{Proof of the first claim} Taking the derivative of $h(v_k, k)$ with respect to $v_k$ gives
\begin{align}\label{eq:h_derivative}
\frac{\partial h}{\partial v_k} (v_k, k) = \delta_k \gs e^{-2T+2t_k} \norm{\mu}_2^2 \frac{\exp(e^{-2T+2t_k} \norm{\mu}_2^2 /2)[\exp(e^{-T + t_k} v_k) + \exp(-e^{-T + t_k} v_k)] + 4}{[\exp(e^{-2T+2t_k} \norm{\mu}_2^2/2) + \exp(e^{-T + t_k} v_k) + \exp(- e^{-T + t_k} v_k)]^2} - 2.
\end{align}
We then choose $\gs_0$ small enough, such that $\delta_k \gs e^{-2T+2t_k}\norm{\mu}_2^2 \leq 1$ for all $\gs \leq \gs_0$, which allows us to set $\gs_0 \leq \frac{1}{\norm{\mu}_2^2 \max \delta_k}$. In this case, Eq.~\eqref{eq:h_derivative} is always negative for any $v_k$ and $k$. To see this, note that
\begin{align*}
\frac{\partial h}{\partial v_k} (v_k, k) & \leq \frac{\exp(e^{-2T+2t_k}\norm{\mu}_2^2/2)[\exp(e^{-T + t_k} v_k) + \exp(-e^{-T + t_k} v_k)] + 4}{[\exp(e^{-2T+2t_k} \norm{\mu}_2^2/2) + \exp(e^{-T + t_k} v_k) + \exp(- e^{-T + t_k} v_k)]^2} - 2 \\
& = - \frac{2\exp(e^{-2T+2t_k} \norm{\mu}_2^2) + 2\exp(2e^{-T + t_k} v_k) + 2\exp(-2e^{-T + t_k} v_k)}{[\exp(e^{-2T+2t_k} \norm{\mu}_2^2/2) + \exp(e^{-T + t_k} v_k) + \exp(- e^{-T + t_k} v_k)]^2} \\
& \quad - \frac{3\exp(e^{-2T+2t_k}\norm{\mu}_2^2/2)[\exp(e^{-T + t_k} v_k) + \exp(-e^{-T + t_k} v_k)]}{[\exp(e^{-2T+2t_k} \norm{\mu}_2^2/2) + \exp(e^{-T + t_k} v_k) + \exp(- e^{-T + t_k} v_k)]^2} \\
& < 0.
\end{align*}
As a consequence, $h(v_k, k)$ is strictly decreasing for $v_k \geq 0$ as demonstrated in the left panel of Figure~\ref{fig:negative_result}. It is straightforward to check that $h(0, k) = 0$. Therefore, $0$ is the only solution to $h(v_k, k) = 0$. 
We define $\tilde h(v_k, k) = h(v_k, k) + 2v_k$. Then $\tilde h(v_k, k) < 2v_k$ for all $v_k > 0$. By symmetry, we have $\tilde h(v_k, k) < 2|v_k|$ for all $v_k \neq 0$.
Observe that $\tilde h(v_k, k) > 0$ for $v_k > 0$, and $\tilde h(v_k, k) < 0$ for all $v_k < 0$. 
Therefore, rewriting Eq.~\eqref{eq:vk_update}, we get
\begin{align*}
|v_{k+1}| = |\tilde h(v_k, k) - v_k| < |v_k| \quad \text{for} \quad v_k \neq 0.
\end{align*}
Setting $\gs_0 = \frac{1}{\norm{\mu}_2^2 \max \delta_k}$ proves the claim that $|v_{k+1}| < |v_k|$ if $v_k \neq 0$. 

We can further show that when $\delta_k \gs e^{-2T + 2t_k}\|\mu\|_2^2$ is sufficiently small, we can guarantee a strict magnitude shrinkage of $|v_k|$, i.e., $|v_{k+1}| \leq (1-\gamma) |v_k|$ for some small $\gamma \in (0, 1)$. To see this, we aim to show a sandwich inequality when $v_k$ is non-negative:
\begin{align*}
\gamma v_k \leq \tilde h(v_k, k) \leq (2-\gamma) v_k.
\end{align*}
%
%A proof for non-positive $v_k$ can be derived symmetrically. 
Accordingly, we denote
\begin{align}\label{eq:h_vk_k_gamma}
\begin{split}
&h_{2- \gamma}(v_k, k) = \delta_k \gs e^{-T + t_k} \frac{\exp(e^{-T + t_k} v_k) - \exp(- e^{-T + t_k} v_k)}{\exp(e^{-2T+2t_k} \norm{\mu}_2^2/2) + \exp(e^{-T + t_k} v_k) + \exp(- e^{-T + t_k} v_k)} \norm{\mu}_2^2 - (2- \gamma)v_k,\\
&h_\gamma(v_k, k) = \delta_k \gs e^{-T + t_k} \frac{\exp(e^{-T + t_k} v_k) - \exp(- e^{-T + t_k} v_k)}{\exp(e^{-2T+2t_k} \norm{\mu}_2^2/2) + \exp(e^{-T + t_k} v_k) + \exp(- e^{-T + t_k} v_k)} \norm{\mu}_2^2 -  \gamma v_k. 
\end{split}
\end{align}
It is obvious that $v_k = 0$ is a zero point of the two functions stated in Eq.~\eqref{eq:h_vk_k_gamma}. To show that $|v_{k+1}| \leq (1-\gamma) |v_k|$, since $(v_{k+1} - v_k)v_k \leq 0$, we only need to show that for a sufficiently small $\delta_k \gs e^{-2T + 2t_k} \|\mu\|_2^2$, $h_{2-\gamma} \leq 0$ and $h_{\gamma}\geq 0$ for all $v_k \geq 0$. % has no other solution on $[0, \infty)$ except for $0$. 
We adopt the notations $a_k = \exp(e^{-2T+2t_k} \norm{\mu}_2^2/2), b_k = \exp(e^{-T + t_k} v_k)$ and $m_k = \delta_k \gs e^{-2T + 2t_k}\|\mu\|_2^2$. To ensure $h_{2-\gamma}(v_k,k)\leq 0$, it suffices to find a sufficiently small $m_k$, such that \begin{align}
\label{eq:C15}
    m_k\cdot\frac{a_k(b_k + 1/b_k)+4}{(a_k + b_k + 1/b_k)^2} \leq 2-\gamma.
\end{align}
Since $xy \leq (x+y)^2 / 2$, we have the following inequality: \begin{align*}
    m_k\cdot\frac{a_k(b_k + 1/b_k)+4}{(a_k + b_k + 1/b_k)^2} &\leq m_k \cdot \frac{(a_k+b_k + 1/b_k)^2 / 2 +4}{(a_k + b_k + 1/b_k)^2}.
\end{align*}
Therefore, to show Eq.~\eqref{eq:C15}, we only need to show 
\begin{align*}
    m_k\bigg(\frac{1}{2} + \frac{4}{(a_k + b_k + 1/b_k)^2} \bigg) \leq (2- \gamma).
\end{align*}
%$$
%m_k\bigg(\frac{1}{2} + \frac{4}{(a_k + b_k + 1/b_k)^2} \bigg) \leq (2- \gamma).
%$$
Since $b_k + 1/b_k \geq 2$ and 
$a_k \geq \exp\big(\frac{\|\mu\|_2^2}{8}\big)$,
%$a_k \geq \exp\big(\frac{\|\mu\|_2^2}{4}\big)$, 
it suffices to ensure $$m_k \leq \frac{(4-2\gamma)\left(2+\exp\big(\frac{\|\mu\|_2^2}{8}\big)\right)^2}{8 + \left(2+\exp\big(\frac{\|\mu\|_2^2}{8}\big)\right)^2}. $$ 
On the other hand, in order to show $h_{\gamma}(v_k,k)\geq 0$ for all $v_k \geq 0$, we only need to prove
\begin{align}
\label{eq:15.5}
    m_k\cdot\frac{a_k(b_k + 1/b_k)+4}{(a_k + b_k + 1/b_k)^2} \geq \gamma.
\end{align}
Note that $a_k(b_k + 1/b_k) \geq 2\exp(\|\mu\|^2/8)$ and $(a_k + b_k + 1/b_k)^2 \leq \{\exp(\|\mu\|^2/2) + 2\exp(|v_k|)\}^2$, then to establish Eq.~\eqref{eq:15.5}, it suffices to have 
$$
m_k \geq \frac{\gamma (\exp(\|\mu\|^2/2) + 2\exp(|v_k|))^2}{2\exp(\frac{\|\mu\|^2}{8})+4}.
$$
The proof of the first claim is complete. 

\paragraph{Proof the second claim} We denote $s(v_k, k) = \exp(e^{-T + t_k} v_k) + \exp(-e^{-T + t_k} v_k)$, which is naturally lower bounded by $2$. Revisiting Eq.~\eqref{eq:h_derivative}, we have
\begin{align}
\label{eq:C16}
\begin{split}
\frac{\partial h}{\partial v_k} (v_k, k) & = \delta_k \gs e^{-2T+2t_k} \norm{\mu}_2^2 \frac{\exp(e^{-2T+2t_k} \norm{\mu}_2^2 /2)s(v_k, k) + 4}{[\exp(e^{-2T+2t_k} \norm{\mu}_2^2/2) + s(v_k, k)]^2} - 2 \\
& \geq \frac{1}{4} \delta_k \gs \norm{\mu}_2^2 \frac{\exp(\norm{\mu}_2^2 / 8) s(v_k, k)}{\exp(\norm{\mu}_2^2) + 2s(v_k, k) \exp(\norm{\mu}_2^2/2) + s(v_k, k)^2} - 2 \\
& \ge \frac{1}{4} (\min_{0 \leq k \leq K - 1} \delta_k) \gs \norm{\mu}_2^2 \frac{\exp(\norm{\mu}_2^2 / 8)}{\frac{\exp(\norm{\mu}_2^2)}{s(v_k, k)} + s(v_k, k) + 2 \exp(\norm{\mu}_2^2/2)} - 2.
\end{split}
\end{align}
When $\delta_k\neq 0$ for all $k \in \{0\} \cup [K - 1]$, the lower bound in the display above first increases then decreases as $s(v_k, k)$ increases from $0$ to $\infty$. 
We take $\gs_0'$ sufficiently large, such that for any $\gs > \gs_0'$, there exists $s_0 \geq 2$ dependent on $(\gs, \|\mu\|, \{\delta_k\}_{k \in \{0\} \cup [K - 1]})$,  such that $\frac{\partial h}{\partial v_k} (v_k, k) > 0$ for all $2 \leq s(v_k, k) \leq s_0$ and all $k \in \{0, 1, \cdots, K - 1\}$. 
In fact, we can choose $\gs_0'$ large enough so that $\frac{\partial h}{\partial v_k}(v_k, k) |_{v_k = 2} >  0$. 
In this case, we may set $s_0$ to be the larger solution to the following quadratic equation (with the variable being $s$):
\begin{align}
\label{eq:quadratic}
\frac{1}{4} (\min_{0 \leq k \leq K - 1} \delta_k) \gs \norm{\mu}_2^2 \frac{\exp(\norm{\mu}_2^2 / 8)}{\frac{\exp(\norm{\mu}_2^2)}{s} + s + 2 \exp(\norm{\mu}_2^2/2)} - 2 = 0.
\end{align}
One can verify that in order to have $\frac{\partial h}{\partial v_k}(v_k, k) |_{v_k = 2} >  0$, it suffices to choose
\begin{align}\label{eq:gs_prime_1}
\gs_0' \geq \frac{16 + 16\exp(\norm{\mu}_2^2 / 2) + 8 \exp(\norm{\mu}_2^2)}{\norm{\mu}_2^2 \exp(\norm{\mu}_2^2 / 8) \min_{0 \leq k \leq K - 1} \delta_k}. 
\end{align}
The larger solution to Eq.~\eqref{eq:quadratic} takes the form: 
\begin{align*}
s_0 & = \frac{1}{16} (\min_{0 \leq k \leq K - 1} \delta_k) \gs \norm{\mu}_2^2 \exp(\norm{\mu}_2^2 / 8) - \exp(\norm{\mu}_2^2 / 2) \\ 
& \quad + \frac{\sqrt{\left[\frac{1}{8} (\min_{0 \leq k \leq K - 1} \delta_k) \gs \norm{\mu}_2^2 \exp(\norm{\mu}_2^2 / 8) - 2\exp(\norm{\mu}_2^2 / 2)\right]^2 - 4\exp(\norm{\mu}_2^2)}}{2}.
\end{align*}
To ensure $s_0 \geq 2$, we may choose $\eta_0'$ satisfying
\begin{align*}
\frac{1}{16} (\min_{0 \leq k \leq K - 1} \delta_k) \gs_0' \norm{\mu}_2^2 \exp(\norm{\mu}_2^2 / 8) - \exp(\norm{\mu}_2^2 / 2) \geq 2,
\end{align*}
which gives rise to
\begin{align}
\label{eq:C20}
\gs_0' \geq \frac{32 + 16\exp(\norm{\mu}_2^2/2)}{\norm{\mu}_2^2 \exp(\norm{\mu}_2^2 / 8) \min_{0 \leq k \leq K - 1} \delta_k}.
\end{align}
Combining  Eq.~\eqref{eq:gs_prime_1} with Eq.~\eqref{eq:C20} leads to
\begin{align*}
\gs_0' \geq \frac{16 + 16\exp(\norm{\mu}_2^2/2) + \max\{16, 8\exp(\norm{\mu}_2^2)\}}{\norm{\mu}_2^2 \exp(\norm{\mu}_2^2 / 8) \min_{0 \leq k \leq K - 1} \delta_k}.
\end{align*}
We observe that $\gs_0' \geq \gs_0$ (recall $\gs_0 = \frac{1}{\norm{\mu}_2^2 \max \delta_k}$), and also $s_0$ increases linearly as the guidance strength $\gs$ increases.

Similar to the derivation of Eq.~\eqref{eq:C16}, we get
\begin{align*}
\frac{\partial h}{\partial v_k} (v_k, k) & \leq (\max_{0 \leq k \leq K - 1} \delta_k) \gs \norm{\mu}_2^2 \frac{\exp(\norm{\mu}_2^2 / 2) s(v_k, k) + 4}{\exp(\norm{\mu}_2^2/4) + 2s(v_k, k) \exp(\norm{\mu}_2^2/8) + s(v_k, k)^2} - 2.
\end{align*}
For a sufficiently large $s_1$, it holds that $\frac{\partial h}{\partial v_k} (v_k, k) < 0$ whenever $s(v_k, k) > s_1$. Indeed, we can solve for $s_1$ explicitly as
\begin{align*}
s_1 & = \frac{1}{4} (\max_{0 \leq k \leq K - 1} \delta_k) \gs \norm{\mu}_2^2 \exp(\norm{\mu}_2^2 / 2) - \exp(\norm{\mu}_2^2 / 8) \\
& \quad + \frac{\sqrt{\left[\frac{1}{2} (\max \delta_k) \gs \norm{\mu}_2^2 \exp(\norm{\mu}_2^2 / 2) - 2\exp(\norm{\mu}_2^2 / 8)\right]^2 - 4\exp(\norm{\mu}_2^2 / 4) + 8(\max \delta_k) \gs \norm{\mu}_2^2}}{2}.
\end{align*}
Again $s_1$ increases as $\gs$ increases, and we can ensure $s_1 \geq 2$ by choosing sufficiently large $\gs_0'$.
In fact, we only require
$$\gs_0' \geq \frac{16 + 16\exp(\norm{\mu}_2^2 / 8)}{\norm{\mu}_2^2 \exp(\norm{\mu}_2^2 / 2) \max \delta_k}.$$

Given $s_0$ and $s_1$, we solve for a constant $a$ so that $s(v_k, k) \leq s_0$ when $v_k \leq a$ for all $k$. This is plausible since  by assumption $e^{-T + t_k}$ takes value inside the interval $[1/2, 1]$. We also solve for $b'$ so that $s(v_k, k) \geq s_0$ when $v_k \geq b'$ for all $k$. 
Checking the definition of $s(v_k, k)$, 
we conclude that we can choose $a, b'$ appropriately, such that both of them increase as $s_0$ and $s_1$ increase, respectively. 
Recall that both $s_0$ and $s_1$ are increasing functions of $\gs$.
Therefore, we deduce that we can find $a$ and $b'$ that satisfy all the above desiderata. Furthermore, both of them get larger as we increase the guidance strength $\gs$.

To summarize, we conclude that for any $\gs \geq \gs_0'$, $h(v_k, k)$ is strictly increasing for all $k$ when $v_k \in [0, a]$, and strictly decreasing for all $k$ when $v_k \geq b'$. Since $h(v_k, k)$ is continuous in $v_k$ and $h(\infty, k) < 0$, there exists $b > 0$ such that $h(v_k, k) < 0$ when $v_k \geq b$ for all $k$. Hence, we have established that for all possible $k$, it holds that $h(v_k, k) > 0$ for $v_k \in (0, a]$ and $h(v_k, k) < 0$ for $v_k > b$. An illustration of the $h(v_k, k)$ curve can be found in the right panel of Figure~\ref{fig:negative_result}. 
Next, we apply the same argument that we used to derive the first claim, and deduce that
\begin{align*}
|v_{k+1}| > |v_k| & \quad \text{if}\quad |v_k| \leq a; \\
|v_{k+1}| < |v_k| & \quad \text{if} \quad |v_{k}| \geq b.
\end{align*}
We skip the proof details for the equations above to avoid redundancy. 
The second claim is verified and thus the proof is complete.
\end{proof}

\begin{figure}[h]
\centering
\includegraphics[width = 0.7\textwidth]{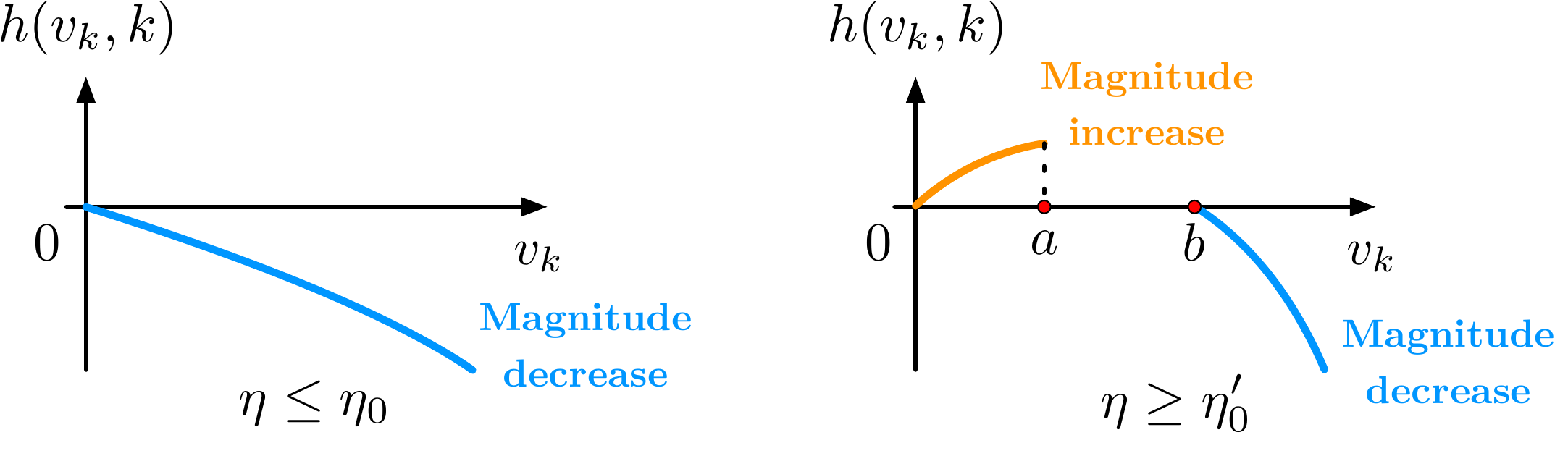}
\caption{Illustration of the behaviors of $h(v_k, k)$ when constrained to the positive real line, under different ranges of guidance strength $\gs$. The left panel corresponds to a small strength $\gs < \gs_0$. In this case, $h(v_k, k)$ is negative and decreasing for all $v_k \geq 0$. In contrast, the right panel corresponds to a large strength $\gs > \gs_0'$, where $h(v_k, k)$ is increasing on $[0, a]$ and decreasing on $[b, \infty)$.}
\label{fig:negative_result}
\end{figure}

\section{Additional numerical experiments}\label{sec:additional_exp}
We collect in this section outcomes from additional numerical experiments.
We first consider discretized samplers, and verify the theoretical results in Section \ref{sec:example-curious}.
Specifically, Figures \ref{fig:3d-align-ddim} and \ref{fig:2d-align-ddim}  demonstrate the behavior of DDIM in 2D/ 3D symmetric 3-component GMMs, Figures \ref{fig:3d-align-ddpm} and \ref{fig:2d-align-ddpm} display the corresponding behavior of DDPM.   
As our theory (Proposition \ref{prop:phase_change}) suggests, when the guidance strength is enormously large, the middle component splits into two clusters. Such phenomenon is not limited to symmetric GMMs: as shown by Figures \ref{fig:2d-nonalign-ddpm} and \ref{fig:3d-nonalign-ddim}, for a large enough guidance strength, the middle component becomes distorted under diffusion guidance even in the context of a non-symmetric GMM. 

We then switch to continuous-time samplers, and the goal is in turn to justify the theoretical implications listed in Section \ref{sec:pre-confidence}.
More precisely, in Figure \ref{fig:3d-ortho-ddim} 
we visualize the effect of diffusion guidance  on a 3-component symmetric GMM in $\RR^3$. 
This can be regarded as an analogue of Figure \ref{fig:three-components} in the 3D setting. 
We further confirm our theoretical results by Figure \ref{fig:prob-DP-2}, which demonstrates how classification confidence and differential entropy evolve as guidance strength $\gs$ increases. 
\begin{figure}[h]
\centering
\includegraphics[width = 1\textwidth]{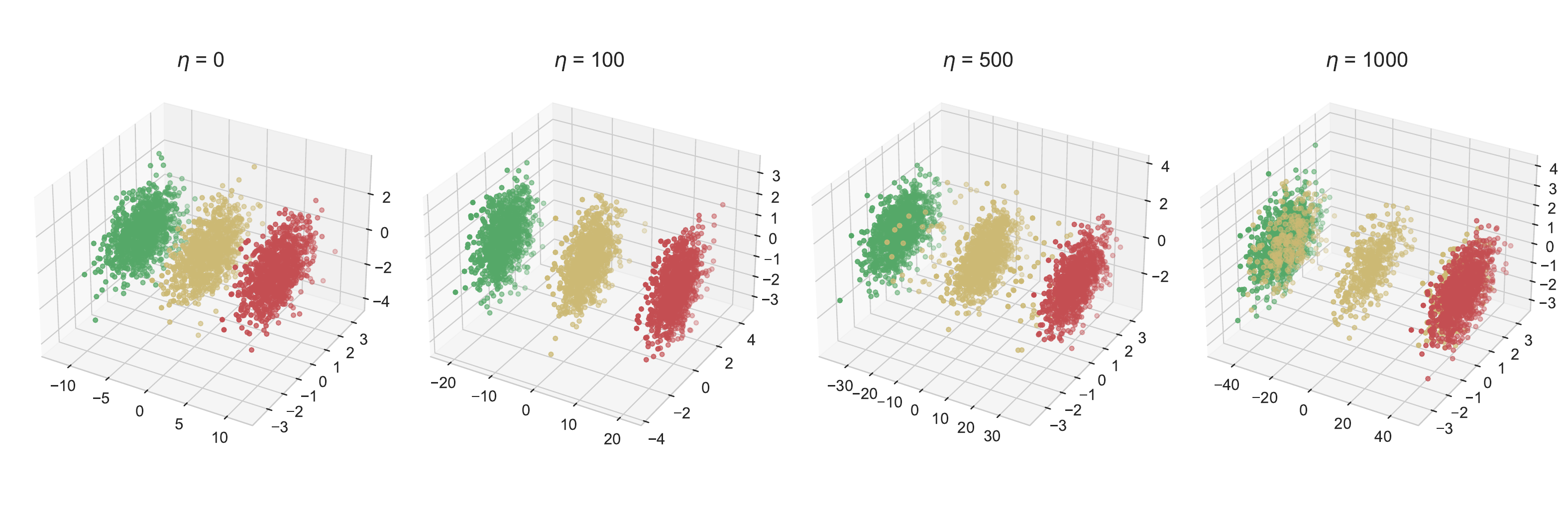}
\caption{
Illustration of the effect of guidance on a discretized DDIM sampler. 
For this experiment, we set $p_{\ast} = \frac{1}{3} \normal((0,0,0),I_3) + \frac{1}{3} \normal\big((0,\sqrt{3},0), I_3\big)+ \frac{1}{3} \normal\big((0,-\sqrt{3},0), I_3\big)$, $T = 10$, and $\delta_k = 0.01$ for all possible $k$.
From the plot, we see that the middle component splits with a sufficiently large $\gs$. 
To summarize, the numerical observations corroborate our theory. 
}
\label{fig:3d-align-ddim}
\end{figure}
\begin{figure}
\centering
\includegraphics[width = 1\textwidth]{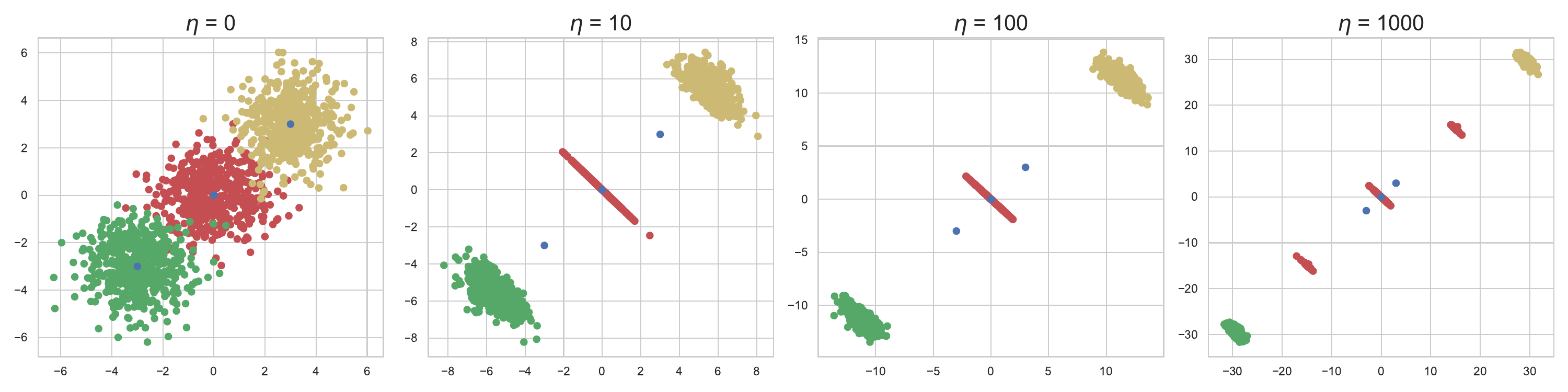}
\caption{
Illustration of the effect of guidance on a discretized DDIM sampler. 
For this experiment, we set $p_{\ast} = \frac{1}{3} \normal((0,0),I_2) + \frac{1}{3} \normal\big((3,3), I_2\big)+ \frac{1}{3} \normal\big((-3,-3), I_2\big)$, $T = 10$, and $\delta_k = 0.01$ for all possible $k$.
The same splitting phenomenon is observed with a sufficiently large $\gs$.}

%Behaviors of symmetric three components symmetric GMM as guidance increases, with DDIM as the backward process. Here the blue points represent the mean of the components. We set stepsize $\delta_k=0.01$ and $T =10$.  Here $p_{\ast} = \frac{1}{3} \normal((0,0),I_2) + \frac{1}{3} \normal\big((3,3), I_2\big)+ \frac{1}{3} \normal\big((-3,-3), I_2\big)$.}
\label{fig:2d-align-ddim}
\end{figure}
\begin{figure}[h]
\centering
\includegraphics[width = 1\textwidth]{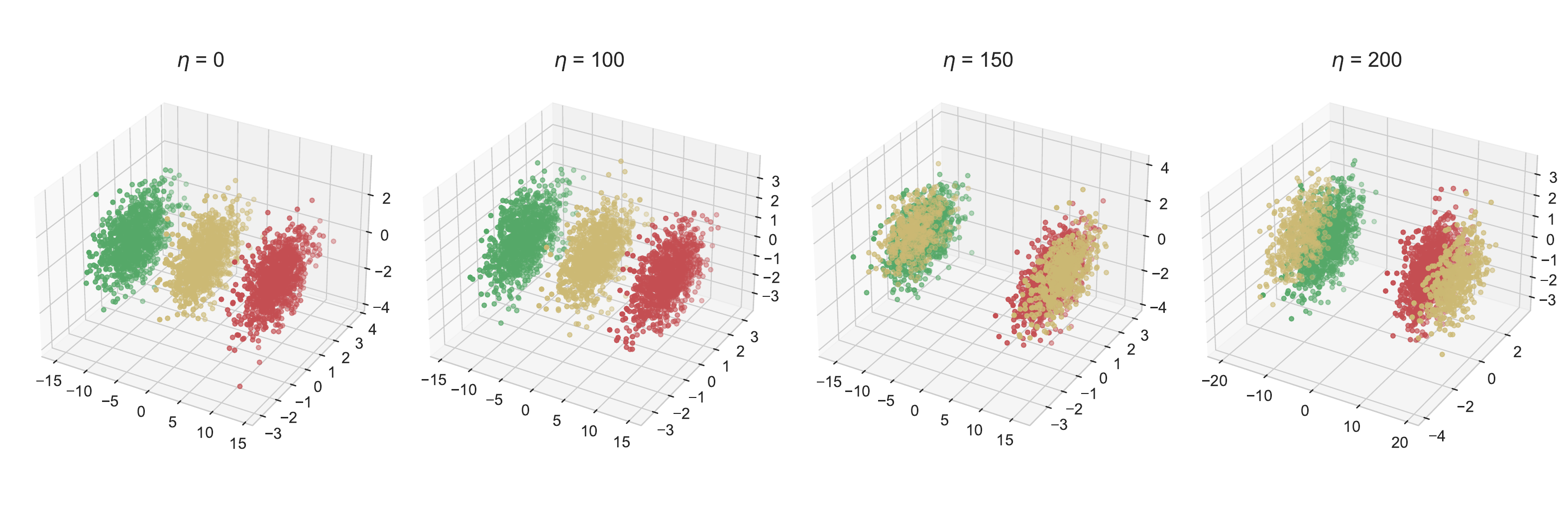}
\caption{Illustration of the effect of guidance  on a discretized DDPM sampler. The experiment setup is the same as that of Figure \ref{fig:3d-align-ddim}. }
\label{fig:3d-align-ddpm}
\end{figure}

\begin{figure}[h]
\centering
\includegraphics[width = 1\textwidth]{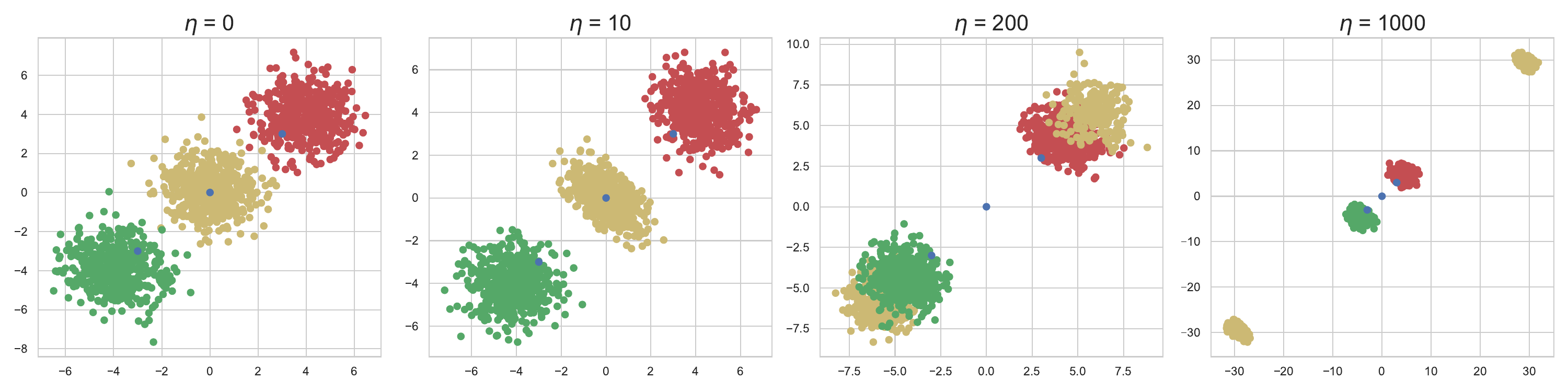}
\caption{Illustration of the effect of guidance  on a discretized DDPM sampler. The experiment setup is the same as that of Figure \ref{fig:2d-align-ddim}. }
\label{fig:2d-align-ddpm}
\end{figure}

\begin{figure}[h]
\centering
\includegraphics[width = 1\textwidth]{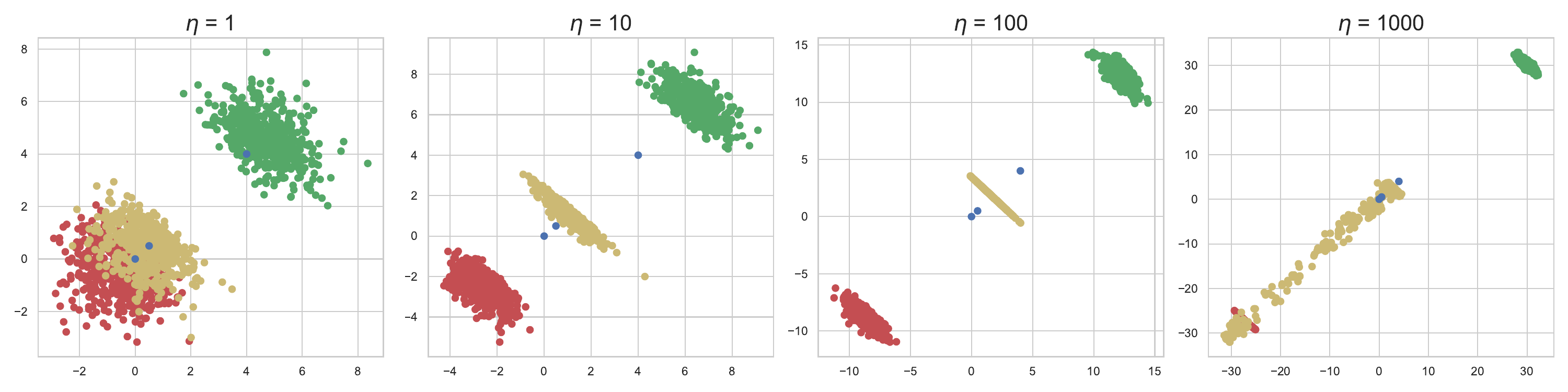}
\caption{Illustration of the effect of guidance  on a discretized DDPM sampler. Here, $p_{\ast} = \frac{1}{3} \normal((0,0),I_2) + \frac{1}{3} \normal\big((0.5,0,5), I_2\big)+ \frac{1}{3} \normal\big((4,4), I_2\big)$, $T = 10$, and $\delta_k = 0.01$ for all possible $k$. In this asymmetric GMM, the center component penetrates the left side component (the left component is colored in red, and without any  guidance is close to the center component) under large enough guidance.}
\label{fig:2d-nonalign-ddpm}
\end{figure}
\begin{figure}[h]
\centering
\includegraphics[width = 1.3\textwidth]{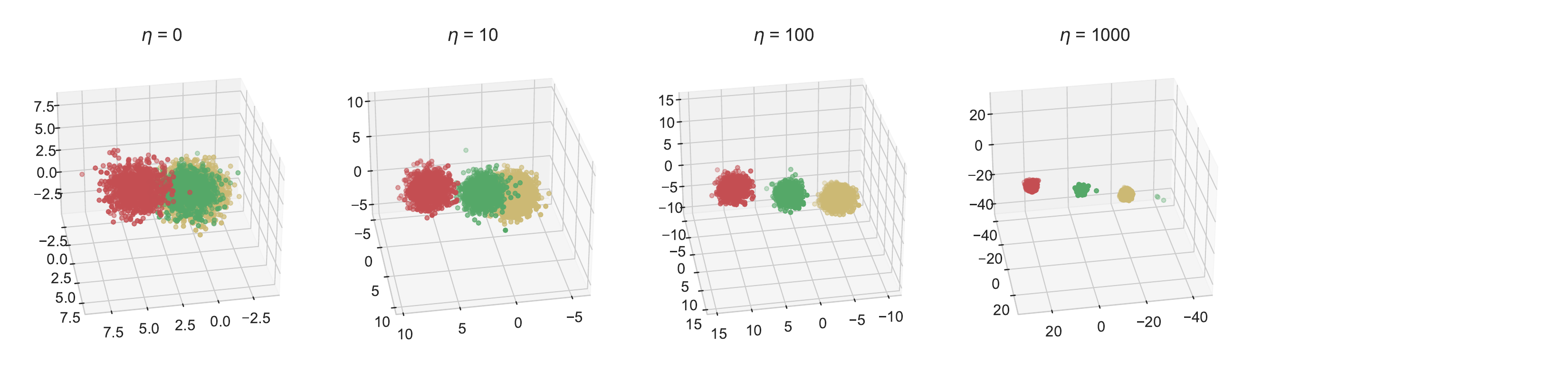}
\caption{Illustration of the effect of guidance  on a discretized DDIM sampler. Here, $p_{\ast} = \frac{1}{3} \normal((0,0,0),I_3) + \frac{1}{3} \normal\big((0.5,0.5,0), I_3\big)+ \frac{1}{3} \normal\big((5,5,0), I_3\big)$, $T = 10$, and $\delta_k = 0.01$ for all possible $k$. The observation is similar to Figure~\ref{fig:2d-align-ddpm} for the 2D case. The center component splits into two components under sufficiently large guidance.}
\label{fig:3d-nonalign-ddim}
\end{figure}
% \begin{figure}[h]
%     \centering
%     \includegraphics{figures/ddpm/symmetric_3_components/3_com_3_symmetric_ddpm.pdf}[width = 0.9\textwidth]
%     \caption{Behaviors of symmetric three components symmetric GMM as guidance increases, with DDPM as the backward process. Here we adopt the same setting for $T$ and $\eta$ in Fig.\ref{fig:2d-align-ddim}}
%     \label{fig:enter-label}
% \end{figure}

\begin{figure}[h]
\centering
\includegraphics[width = 1\textwidth]{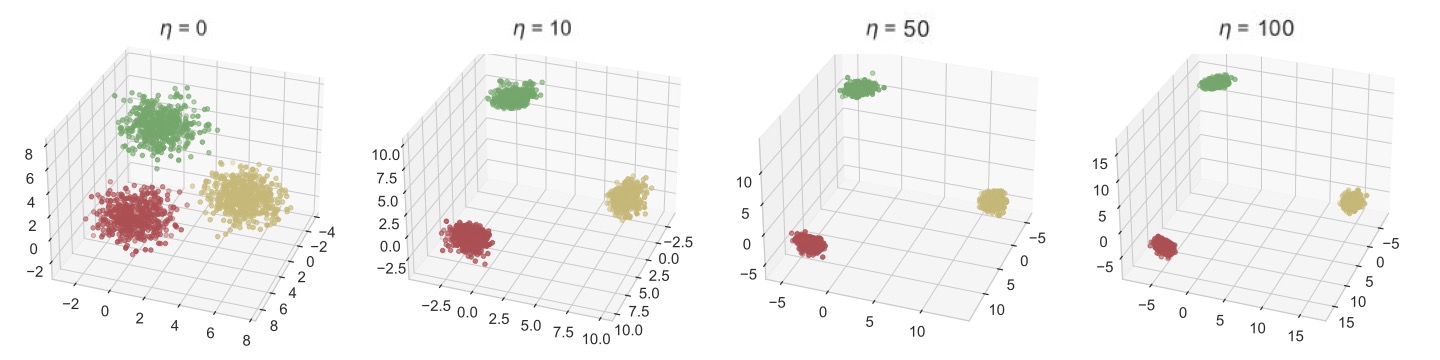}
%\vspace{-2.55cm}
\caption{Illustration of the effect of guidance  on a continuous-time DDIM sampler. Here, $p_{\ast} = \frac{1}{3} \normal((1,0,0),I_3) + \frac{1}{3} \normal\big((0,1,0), I_3\big)+ \frac{1}{3} \normal\big((0,0,1), I_3\big)$. This setting satisfies Assumption~\ref{assumption:confidence}, and we observe that the components become more separated as we increase the guidance strength. }
\label{fig:3d-ortho-ddim}
\end{figure}

\begin{figure}%
    \centering
    \includegraphics[width=0.49\textwidth]{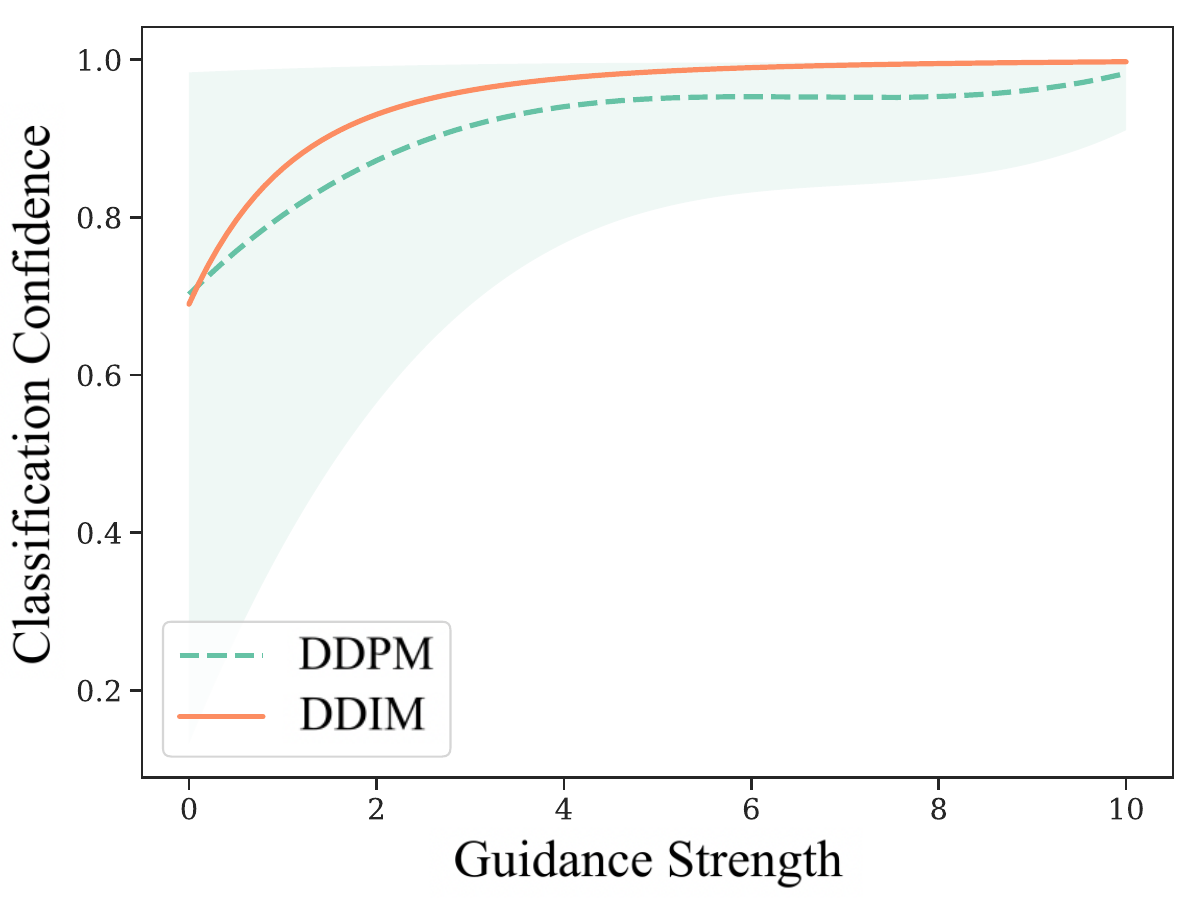} %
    \hfill
    \includegraphics[width=0.49\textwidth]{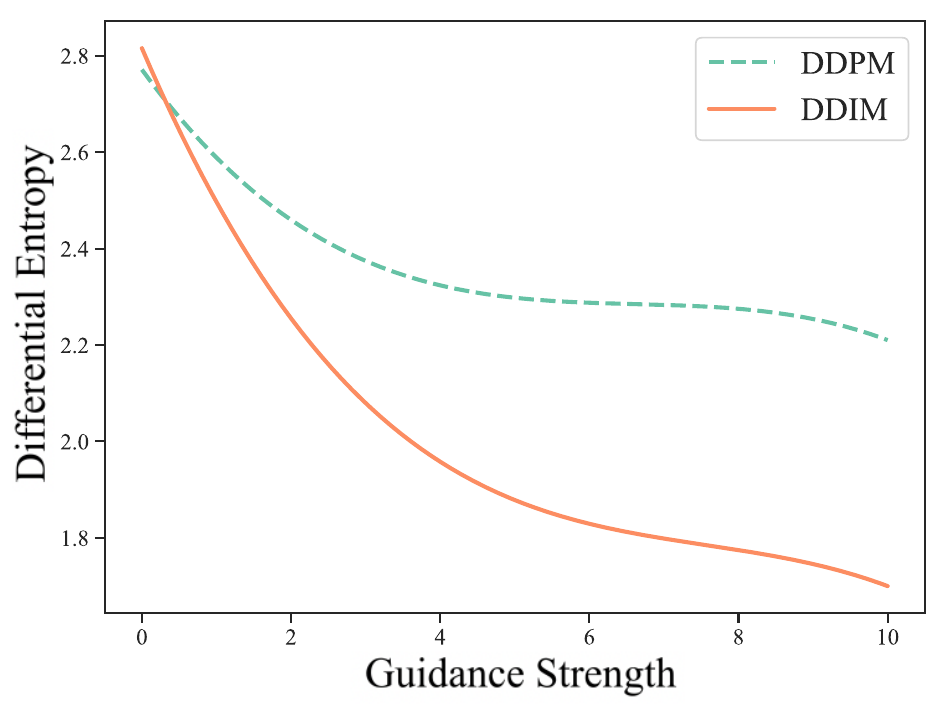} %
    \caption{The effect of diffusion guidance on continuous-time samplers. We consider a three-component equidistant GMM: $p_{\ast} = \frac{1}{3} \normal((0,1),I_2) + \frac{1}{3} \normal\big((\frac{\sqrt{3}}{2}, -\frac{1}{2}), I_2\big)+ \frac{1}{3} \normal\big((-\frac{\sqrt{3}}{2}, -\frac{1}{2}), I_2\big)$. (a) In the left panel, we initiate the reverse processes at the origin and record the classification confidence (measured by the posterior probability of class label) under different levels of guidance. For the DDPM sampler, the output sample is random. We generate $10^3$ samples for each guidance strength and plot the $97.5\%$ and $2.5\%$ quantiles. (b) In the right panel, we initiate the processes following a standard Gaussian distribution and plot the differential entropy of the output distributions. For each guidance strength, we generate $10^4$ samples. We adopt the function $\mathtt{scipy.neighbors.KernelDensity}$ from the $\mathtt{scipy}$ module in Python to estimate the density function of the generated distribution using one half of the generated samples, and use the other half for a Monte Carlo algorithm to estimate the differential entropy.  }%
    \label{fig:prob-DP-2}%
\end{figure}

\end{document}